\documentclass[preprint,12pt]{elsarticle}

\usepackage{fullpage}
\usepackage{amssymb}
\usepackage{xspace}
\usepackage{subfigure}
\usepackage{algorithm}
\usepackage{algorithmic}
\usepackage{amsmath}
\usepackage{tikz}
\usepackage{pgfkeys}
\usepackage{pgfplots}

\newtheorem{theorem}{Theorem}

\newtheorem{proposition}{Proposition}

\newtheorem{definition}{Definition}

\newtheorem{example}{Example}

\newtheorem{lemma}{Lemma}

\newtheorem{observation}{Observation}

\newtheorem{corollary}{Corollary}

\newproof{proof}{Proof}
\newproof{proofofthm-min-con-criterion}{Proof of Theorem \ref{thm:min-con-criterion}}
\newproof{proofofthm-bounded-rel-con}{Proof of Theorem \ref{thm:bounded-rel-con}}
\newproof{proofofthm-disjuncts}{Proof of Theorem \ref{thm:disjuncts}}

\newcommand{\cseq}{\Phi}
\newcommand{\form}[1]{F_{#1}}
\newcommand{\D}{\mbox{{\it C }}}
\newcommand{\Q}{\mbox{{\it Q }}}
\newcommand{\U}{\mathcal{U}}
\newcommand{\A}{\mathcal{A}}
\newcommand{\Dom}{\mathcal{D}}

\newcommand{\pref}{\,\succ}

\newcommand{\contracting}{base contractor\xspace}
\newcommand{\contractingrel}{base contractor\xspace}
\newcommand{\contractingrels}{base contractors\xspace}
\newcommand{\contraction}{full contractor\xspace}
\newcommand{\meetcontraction}{full meet contractor\xspace}
\newcommand{\meetcontractions}{full meet contractors\xspace}
\newcommand{\meetprotcontraction}[1]{full #1-protecting meet contractor\xspace}

\newcommand{\contractions}{full contractors\xspace}
\newcommand{\contractedrel}{contracted relation\xspace}
\newcommand{\lang}{\mathcal{L}}

\newcommand{\setrev}{\ ^*\ }
\newcommand{\setcon}{\div }

\newcommand{\normsubscr}[1]{\mbox{{\scriptsize #1}}}
\newcommand{\divtolayers}{stratifiable\xspace}
\newcommand{\layer}{stratum\xspace}
\newcommand{\layerindex}{\layer index\xspace}
\newcommand{\layers}{strata\xspace}
\newcommand{\boundedlayer}{finitely stratifiable\xspace}
\newcommand{\boundedlayerprop}{finite stratifiability property\xspace}
\newcommand{\ero}{ERO}
\newcommand{\tname}[1]{\mbox{$#1$}}

\newcommand{\leftsideproof}{\noindent \begin{tikzpicture} \node[draw, inner sep=3pt, rounded corners=2pt] at (0,0) {{\scriptsize $\Leftarrow$}};;\end{tikzpicture}\ \ }
\newcommand{\rightsideproof}{\noindent \begin{tikzpicture} \node[draw, inner sep=3pt, rounded corners=2pt] at (0,0) {{\scriptsize $\Rightarrow$}};;\end{tikzpicture}\ \ }

\def\join#1{\mathrel{\mathop{
       \hbox{$\triangleright\!\triangleleft$}}
       \limits_{#1}}}

\journal{ARTIFICIAL INTELLIGENCE Journal}

\begin{document}

\begin{frontmatter}
\title{Contracting preference relations for database applications\tnoteref{t1}}
\tnotetext[t1]{Research partially supported by NSF grant IIS-0307434. This paper is an 
extended version of \cite{DBLP:conf/aaai/MindolinC08}}

\author{Denis Mindolin}
\ead{mindolin@cse.buffalo.edu}

\author{Jan Chomicki}
\ead{chomicki@cse.buffalo.edu}

\address{Department of Computer Science and Engineering\\
201 Bell Hall, University at Buffalo\\
Buffalo, NY 14260-2000, USA}

\begin{abstract}
The binary relation framework has been shown to be applicable to
many real-life preference handling scenarios.
Here we study preference contraction: the problem of discarding selected preferences.
We argue that the property of minimality and the preservation of
strict partial orders are crucial for contractions.
Contractions can be further constrained by specifying which
preferences should be protected. 
We consider two classes of preference relations: finite and finitely representable.
We present algorithms for
computing minimal and preference-protecting minimal contractions for
finite as well as finitely representable preference relations. 
We study relationships between preference
change in the binary relation framework and belief change in
the belief revision theory. We also introduce some preference
query optimization techniques which can be
used in the presence of contraction. We evaluate the proposed algorithms
experimentally and present the results.
\end{abstract}


\end{frontmatter}

\section{Introduction}

A large number of preference handling frameworks have been 
developed \cite{Fishburn1970,boutilier03cpnets,sep-preferences}. 
In this paper, we work with the \emph{binary relation} preference 
framework \cite{Chomicki2003, kiesling2002}.
Preferences are represented as binary relations over tuples. 
They are required to be \emph{strict 
partial orders (SPO)}: transitive and irreflexive binary relations.
The SPO properties are known to capture the rationality of preferences
\cite{Fishburn1970}.
This framework can deal with finite as well as infinite 
preference relations, the latter represented using finite \emph{preference
formulas}. 

Working with preferences in any framework, 
it is naive to expect that
they never change. Preferences can change
over time: if one likes something now, it does not mean
one will still like it in the future. Preference change
is an active topic of current research \cite{Chomicki2007,Freund2004}. 
It was argued \cite{Doyle2004pref} that along with the discovery
of sources of preference change and elicitation of the change
itself, it is important to preserve the correctness
of preference model in the presence of change. 
In the binary relation framework, a natural correctness
criterion is the preservation of SPO properties of preference relations.

An operation of preference change -- preference revision -- has been proposed in 
\cite{Chomicki2007}. 
We note that when a preference relation is changed using a revision operator,
new preferences are ``semantically combined'' with the original preference relation. 
However, combining new preferences with the existing ones is not the only way people 
change their preferences in real life. 
Another very common operation of preference change is ``semantic subtraction'' from 
a set of preferences another set of preferences 
one used to hold, if the reasons for holding
the \emph{contracted} preferences are no longer valid.
That is, we are given an initial 
preference relation $\,\succ$ and a subset $CON$ 
of $\pref$ (called here a \emph{\contractingrel}) which should not hold. 
We want to change $\pref$ in such a way that $CON$ does not hold in it. 
This is exactly opposite to the way the preference revision operators change preference relations.
Hence, such a change cannot be captured by the existing preference revision operators.

In addition to the fact that discarding preferences is common, there is another 
practical reason why preference contraction is important. In many database applications,
preference relations are used to compute sets of the best (i.e. the most preferred) objects 
according to user's preferences. Such objects may be cars, books, cameras etc. 
The operator which is used in the binary relation 
framework to compute such sets is called \emph{winnow} \cite{Chomicki2003} (or \emph{BMO} in \cite{kiesling2002}).
The winnow operator is denoted as 
$w_{\succ}(r)$, where $r$ is the original set of objects, and $\succ$ is a preference relation. 
If the preference relation $\pref$ is large (i.e. the user has many preferences), the result of
$w_{\succ}(r)$ may be too narrow. One way to widen the result is by discarding some preferences
in $\pref$. Those may be the preferences which do not hold any more or are not longer important.

In this paper, we address the problem of contraction of preference relations. We consider it for 
finitely representable infinite preference relations (Example \ref{ex:motiv}) and finite preference
relations (Example \ref{ex:motiv-finite}).

\begin{example}\label{ex:motiv}
  Assume that Mary wants to buy a car. She prefers newer cars, and given 
	two cars made in the same year, a cheaper one is preferred. 
  $$o \succ o' \equiv o.year > o'.year \vee o.year = o'.year \wedge o.price < o'.price$$
  where $>, <$ denote the standard orderings of rational numbers, 
  the attribute $year$ defines the year when cars
  are made, and the attribute $price$ -- their price.
  The information about all cars which are
  in stock now is shown in the table below:
  \smallskip
\begin{center}
  {\small
  \begin{tabular}{l|l|l|l}
    \hline
      $id$ & $make$ & $year$ & $price$ \\
   \hline
      $t_1$ & VW & 2007 & 15000 \\
      $t_2$ & VW & 2007 & 20000 \\
      $t_3$ & Kia & 2006 & 15000 \\
      $t_4$ & Kia & 2007 & 12000 \\
   \hline    
  \end{tabular}
  }
\end{center}
  \smallskip

  Then the set of the most preferred cars according 
  to $\pref$ is $S_1 = \{t_4\}$. Assume that having observed the set $S_1$, Mary understands that
  it is too narrow. She decides that the car $\,t_1$ is not really
  worse than $t_4$. 
  She generalizes that by stating that the cars made in $2007$ 
  which cost $12000$ are not better
  than the cars made in $2007$ costing $15000$. So
  $t_4$ is not preferred to $t_1$ any more, and thus
  the set of the best cars according to the new preference
  relation should be $S_3 = \{t_1,t_4\}.$

  The problem which we face here is 
  how to represent the change to the preference relation
  $\pref$.
  Namely, we want to find a preference
  relation obtained from $\pref$, in which 
  certain preferences
  do not hold. A naive solution is to 
  represent the new preference as 
  \mbox{$\pref_1 \ \equiv \ (\pref -\ CON)$}, where
  $CON(o,o') \equiv o.year = o'.year = 2007 \wedge 
  o.price = 12000 \wedge o'.price = 15000$, i.e.,
  $CON$ is the preference we want to discard. 
  So 
	\begin{align*}
  o\succ_1 o' \equiv & (o.year > o'.year \vee o.year = o'.year \wedge
  o.price < o'.price) \wedge \\
	& \neg (o.year = o'.year = 2007 \wedge o.price = 12000 \wedge o'.price = 15000).
	\end{align*}

  However, $\pref_1$ is not transitive
  since if we take $t_5 = (VW, 2007, 12000)$,
  $t_6 = (VW, 2007, $ $14000)$, and $t_7 = (VW, 2007, 15000)$,
  then $t_5 \succ_1 t_6$ and $t_6 \succ_1 t_7$ but
  \mbox{$t_5 \not \succ_1 t_7$}. Hence,
  this change does not preserve SPO. To 
  make the changed preference relation transitive,
  some other preferences have to be discarded
  in addition to $CON$. At the same time,
  discarding too many preferences is not
  a good solution since they may be important. 
  Therefore, we need to discard 
  a minimal part of $\pref_1$ which contains $CON$
  and preserves SPO in the modified preference relation.
  An  SPO preference relation 
  which is minimally different from $\pref_1$ and
  does not contain $CON$ is shown below:
	\begin{align*}
    o \succ_2 o' \equiv & (o.y > o'.y \vee o.y = o'.y \wedge o.p < o'.p) \wedge \\
                      	    & \neg (o.y = o'.y = 2007 \wedge o.p = 12000 \wedge
                              o'.p > 12000 \wedge o'.p \leq 15000)
  	\end{align*}
  The set of the best cars according to $\,\succ_2$ is $S_2' = \{t_1,t_4\}$.
  As we can see, the relation $\pref_2$ is different
  from the naive solution $\pref_1$ in the sense
  that $\pref_2$ implies that a car made in $2007$ costing
  $12000$ is not better than a car made in $2007$ costing
  \emph{from $12000$ to $15000$}. We note that $\succ_2$ is not the only relation
  minimally different from $\pref_1$ and not containing $CON$. 
\end{example}

\begin{example}\label{ex:motiv-finite}
Let Mary have the following preferences over cars. She prefers \emph{VW to Kia}.
Given two VW, her preference over color is \emph{red is better than green, which is better than 
blue}. Given two Kias, her preference over color is \emph{green is better red, which is better than blue}.
 In this example, we use the \emph{ceteris paribus} semantics \cite{boutilier03cpnets}:
 the preference statements above are used to compare only the tuples different in a single attribute. 
 The SPO preference relation $\succ_1$ representing Mary's preferences is shown in Figure
 \ref{pic:cpnet-contraction-orig}. An edge from a tuple to another one denotes the preference of the first tuple
 to the second one.

Assume that after some time, Mary decides to change her preferences:
 \emph{a red $VW$ is not better than a red $Kia$}, and \emph{a green $VW$ is not better than a blue $Kia$}.
 That means that the set of preferences we need to drop from the current preference relations is 
 $CON = \{($ \mbox{$(VW, red)$} , \mbox{$(Kia, red)$} $); ($ \mbox{$(VW, green)$}, \mbox{$(Kia, blue)$} $)\}$.
 The corresponding edges are dashed in Figure \ref{pic:cpnet-contraction-orig}. Clearly, if these edges are removed from the graph, it will not be transitive any more.
 Hence, additional edges need to be removed to preserve the transitivity of the preference relation. One minimal set of edges
 whose removal along with $CON$ results in a transitive preference relation is shown 
 in Figure \ref{pic:cpnet-contraction-contr} as the dashed edges.
\end{example}

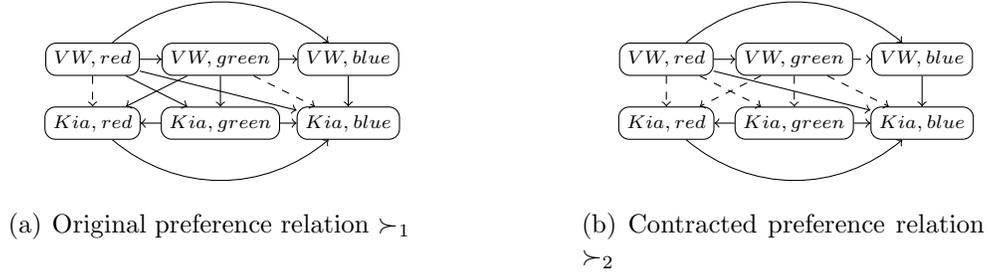
\begin{figure}[ht]
	\centering
	\subfigure[Original preference relation $\succ_1$]{
		\begin{tikzpicture}[scale=0.85]
			\tikzstyle{cir} = [draw=black,rounded corners,inner sep=3pt, font=\fontsize{7}{7}]
			 \tikzstyle{upperedge} = [bend left=40, ->];
			 \tikzstyle{loweredge} = [bend right=40, ->];

			\node[cir] (VR) at (0,1) {$VW, red$};
			\node[cir] (VG) at (2,1) {$VW, green$};
			\node[cir] (VB) at (4,1) {$VW, blue$};
			\node[cir] (KR) at (0,0) {$Kia, red$};
			\node[cir] (KG) at (2,0) {$Kia, green$};
			\node[cir] (KB) at (4,0) {$Kia, blue$};

			\draw[upperedge] (VR) to (VB);
			\draw[->] (VR.east) to (VG.west);
			\draw[->] (VG.east) to (VB.west);

			\draw[->] (KG.west) to (KR.east);
			\draw[->] (KG.east) to (KB.west);
			\draw[loweredge] (KR) to (KB);

			\draw[->, dashed] (VR) to (KR);
			\draw[->] (VG) to (KG);
			\draw[->] (VB) to (KB);

 			\draw[->, dashed] (VG) to (KB);
			\draw[->] (VR) to (KB);
			\draw[->] (VG) to (KR);
			\draw[->] (VR) to (KG);

			\node at (-1,0){};
			\node at (4,0){};

		\end{tikzpicture}
		\label{pic:cpnet-contraction-orig}
	}
	\hspace{2cm}
	\subfigure[Contracted preference relation $\succ_2$]{
		\begin{tikzpicture}[scale=0.85]
			\tikzstyle{cir} = [draw=black,rounded corners,inner sep=3pt, font=\fontsize{7}{7}]
			 \tikzstyle{upperedge} = [bend left=40, ->];
			 \tikzstyle{loweredge} = [bend right=40, ->];

			\node[cir] (VR) at (0,1) {$VW, red$};
			\node[cir] (VG) at (2,1) {$VW, green$};
			\node[cir] (VB) at (4,1) {$VW, blue$};
			\node[cir] (KR) at (0,0) {$Kia, red$};
			\node[cir] (KG) at (2,0) {$Kia, green$};
			\node[cir] (KB) at (4,0) {$Kia, blue$};

			\draw[upperedge] (VR) to (VB);
			\draw[->] (VR.east) to (VG.west);
			\draw[->, dashed] (VG.east) to (VB.west);

			\draw[->] (KG.west) to (KR.east);
			\draw[->] (KG.east) to (KB.west);
			\draw[loweredge] (KR) to (KB);

			\draw[->, dashed] (VR) to (KR);
 			\draw[->, dashed] (VG) to (KG);
			\draw[->] (VB) to (KB);

 			\draw[->, dashed] (VG) to (KB);
 			\draw[->] (VR) to (KB);
 			\draw[->, dashed] (VG) to (KR);
 			\draw[->, dashed] (VR) to (KG);

			\node at (-1,0){};
			\node at (4,0){};
		\end{tikzpicture}
		\label{pic:cpnet-contraction-contr}
	}

 \caption{Contraction of a finite preference relation}
 \label{pic:cpnet-contraction}
\end{figure}

The examples above show that to discard a subset $CON$ of a preference relation $\pref$, 
some preferences additional to $CON$ may be discarded to make the resulting preference
relation an SPO. A subset $P^-$ of $\pref$ which containts $CON$ and
whose removal from $\pref$ preserves the SPO axioms of the modified preference relation is called
a \emph{\contraction of $\pref$ by $CON$}. Such a set $P^-$ may be viewed as a union of the preferences $CON$ to discard
and \emph{a set of reasons} of discarding $CON$. Ideally, if a user decides to discard
preferences, she also provides all the reasons for such a change. In this case, the relation 
$(\pref - \,CON)$ is already a preference relation (i.e., SPO). However, in real life scenarios, it is hard to 
expect that users always provide complete information about the change they want to make. 
At the same time, the number of alternative \contractions $P^-$ for a given $\pref$ and $CON$ may be
large or even infinite for infinite preference relations.
As a result, there is often a need to learn from the user the reasons for discarding preferences. 
That may be done in a step-wise manner by exploring possible alternatives and using user feedback to select the correct ones.

We envision the following scenario here.
To find a complete set of preferences she wants to discard,  
the user iteratively expresses the most obvious
preferences $CON$ that should be dropped from her preference relation $\pref$. 
After that, a possible set of reasons $P^-$ for such a change is computed.
To check if she is satisfied with the computed $P^-$, an impact
of the performed change may be demonstrated to her (e.g., the result of the winnow operator
over a certain data set). If the \contraction $P^-$ does not 
represent the change she actually wanted to make, the user may undo the change and select another alternative 
or tune the contraction by elaborating it.
One type of such elaboration is specifying a set of additional preferences to discard. 
Another type of elaboration which we propose in this paper is \emph{preference protection}. 
Because the exact reasons for contracting $CON$ are not known beforehand, some
preferences which are important for the user may be contracted in an intermediate $P^-$. 
To avoid that, a user can impose a requirement of \emph{protecting a set of preferences
from removal}. The corresponding contraction operator is called here
\emph{preference-protecting contraction}. This iterative process stops when the user
is satisfied with the computed \contraction.

An important property of the scenario above is that the set of reasons $P^-$
which is computed as a result of the iterative process above has to be as small as possible.
That is, preferences that the user does not want to discard and that are not needed to be removed to 
preserve the SPO properties of the modified preference relation should remain. To preserve the minimality of preference change, two approaches
are possible.

In the first one, the \contraction computed in every step is \emph{minimal}. The corresponding 
contraction operator here is called \emph{minimal preference contraction}. It is guaranteed
that the \contraction $P^-$ computed in the last iteration (i.e., when $P^-$ is satisfactory for the user)
is minimal. Note that since there could be many possible minimal 
\contractions of a preference relation by a \contracting, \emph{any} of them may be picked
assuming that if the user is not completely satisfied with it, she will tune the contraction in the next
iteration. 
In belief revision theory, the contraction operator 
with a similar semantics is called \emph{maxichoice contraction} \cite{hansson-book-chapter}. 

In the other variant, the \contraction $P^-$ computed in every step is not necessary minimal.
However, the user can make $P^-$ smaller by specifying the preferences which should be protected
from removal. The \contraction computed in the last step may be not minimal, 
but sufficiently small to meet the user expectations. We propose to construct $P^-$
as the \emph{union} of all minimal \contractions of the preference relation by $CON$.
This contraction operator is called \emph{meet contraction} if no preferences need to be preserved, 
and \emph{preference-protecting meet contraction} if preference preservation is required.
Similar operators in belief revision are  
\emph{full meet contraction}, and \emph{partial meet contraction} \cite{hansson-book-chapter}.

We note that the operations of preference contraction we propose in this paper should 
be understood in the context of the scenario discussed above. However, the details
of the scenario are beyond the scope of this work. The main results of the paper are as follows.
First, we present necessary and sufficient conditions for successful construction of the minimal preference
contraction. Second, we propose two algorithms for minimal preference contraction: the first for 
finitely representable preference relations and 
the second for finite preference relations. Third, we show necessary and sufficient conditions
for successful evaluation of 
the operator of preference-protecting contraction and propose an algorithm of computing the operator of preference-protecting
minimal contraction. Fourth, we show how meet and meet preference-protecting contraction operators can be computed in the preference relation framework.
Fifth, we show how to optimize preference query evaluation in the presence of contraction.
Finally, we perform experimental evaluation of the proposed framework and present the results 
of the experiments. 
In the related work section, we show relationships of the current work with belief revision and 
other approaches of preferences change.

\section{Basic Notions}\label{sec:notations}

The preference relation framework we use in the 
paper is a variation of the one proposed in \cite{Chomicki2003}.
Let $\A = \{A_1,\ldots,A_m\}$ be a fixed set of attributes. 
Every attribute $A_i$ is associated with a \emph{domain}
$\Dom_{A_i}$. We consider here two kinds of infinite domains:
\D (uninterpreted constants) and \Q (rational numbers).
Then a universe $\U$ of tuples is defined as
$$\U = \prod_{A \in \A}\Dom_{A_i}$$
We assume that two tuples $o$ and $o'$ are equal if and only if
the values of their corresponding attributes are equal.

\begin{definition}
  A binary relation $\,\succ\:\subset \U \times \U$ is a
  \emph{preference relation}, if it is a strict
    partial order (SPO) relation, i.e., transitive and irreflexive.
\end{definition}

Binary relations $R \subseteq \U \times \U$ considered
in the paper are \emph{finite}
or \emph{infinite}. Finite binary relations are represented
as sets of pairs of tuples. The infinite binary relations
we consider here are \emph{finitely representable} as \emph{formulas}. 
Given a binary relation $R$, its formula
representation is denoted $\form{R}$. That is, $R(o, o')$ iff $\form{R}(o,o')$. 
A formula representation $\form{\succ}$ of a preference relation 
$\,\succ$ is called a \emph{preference formula}.

We consider two kinds of atomic formulas here:
\begin{itemize}
   \item \emph{equality constraints}: $o.A_i = o'.A_i$, $o.A_i \neq o'.A_i$, 
     $o.A_i = c$, or $o.A_i \neq c$,
     where $o, o'$ are tuple variables, $A_i$ is
     a \D-attribute, and $c$ is an uninterpreted constant;
   \item \emph{rational-order constraints}: $o.A_i\theta o'.A_i$ or 
     $o.A_i\theta c$, where
     \mbox{$\theta \in \{=,\neq, <, >, \leq, \geq\}$}, $o, o'$ are tuple variables,
     $A_i$ is a \Q-attribute, and $c$ is a rational number. 
\end{itemize}

A preference formula whose all atomic formulas are equality (resp. rational-order) 
constraints will be called an \emph{equality} (resp. \emph{rational order}) preference formula.
If both equality and rational order constraints are used in a formula, the formula will be called
an \emph{equality/rational order} formula or simply \emph{\ero}-formula.
Without loss of generality, we assume that all preference formulas are quantifier-free 
because \ero-formulas admit quantifier elimination.
Examples of relations represented using \ero-formulas are $\pref, \pref_1$ and $\pref_2$ 
in Example \ref{ex:motiv}.
\medskip

We also use the representation of binary relations as \emph{directed graphs}, both 
in the finite and the infinite case.

\begin{definition}\label{def:edge}
  Given a binary relation $R \subseteq \U \times \U$ and two tuples $x$ and $y$
  such that $x R y$ ($xy \in R$), \emph{$xy$ is an $R$-edge from $x$ to $y$}.
  A \emph{path in $R$} (or an \emph{$R$-path}) from $x$ to $y$ is a sequence
  of $R$-edges such that the start node of the first edge is $x$, the end node
  of the last edge is $y$, and the end node of every edge (except the last one) is the start node of the next edge
  in the sequence. The sequence of nodes participating in an $R$-path is an \emph{$R$-sequence}.
  The \emph{length of an $R$-path}
  is the number of $R$-edges in the path. The \emph{length of an $R$-sequence} is the number of 
  nodes in it.
\end{definition}

An element of a preference relation is called \emph{a preference}.
We use the symbol $\,\succ$ with subscripts to refer to preference relations. We write
$x \succeq y$ as a shorthand for \mbox{$(x \succ y \vee x = y)$}. 
We also say that \emph{$x$ is preferred to $y$} and \emph{$y$ is dominated by $x$} according to $\pref$
if $x \succ y$.

\medskip

In this paper, we present several algorithms for finite relations. 
Such algorithms are implemented using the \emph{relational algebra} operators:
{selection} $\sigma$, {projection} $\pi$, {join}
$\bowtie$, set difference $-$, and union $\cup$ \cite{ramakrishnan-db-book}. Set difference and 
union in relational algebra have the same semantics as in the set theory.
The semantics of the other operators are as follows:
\begin{itemize}
 \item Selection $\sigma_{F}(R)$ picks from the relation $R$ all the tuples for which the condition
$F$ holds. The condition $F$ is a boolean expression involving comparisons 
between attribute names and constants.
 \item Projection $\pi_{L}(R)$ returns a relation which is obtained from 
the relation $R$ by leaving in it only the columns listed in $L$ and dropping the others.
  \item Join of two relations $R$ and $S$
	$$R \join{R.X_1 = S.Y_1, \ldots, R.X_n = S.Y_n} S$$
computes a product of $R$ and $S$, leaves only the tuples in which
$R.X_1 = S.Y_1, \ldots, R.X_n = S.Y_n$, and drops the columns $S.Y_1, \ldots, S.Y_n$
from the resulting relation. 
\end{itemize}

When we need more than one copy of a relation $R$ 
in a relational algebra expression, we add subscripts to the relation name (e.g. $R_1, R_2$ etc).

\section{Preference contraction}

Preference contraction is an operation of discarding preferences. 
We assume that when a user intends to discard some preferences, he or she
expresses the preferences to be discarded as a binary relation called a \emph{\contractingrel{}}. 
The interpretation of each pair in a \contractingrel is that the first tuple should not be preferred to the second tuple. 
We require \contracting{}
relations to be subsets of the preference relation to be contracted. 
Hence, a \contractingrel{} is irreflexive but not necessary transitive.
Apart from that,
we do not impose any other restrictions on the \contractingrels{}
(e.g., they can be finite of infinite), unless stated otherwise. 
Throughout the paper, \contractingrels{} are typically referred to as $CON$.

\begin{definition}\label{def:contr-rel}
  A binary relation $P^-$ is \emph{a \contraction{} of a preference relation $\,\succ$ by $CON$} if
\mbox{ $CON \subseteq P^- \subseteq\,\succ$}, and 
 $(\,\succ -\ P^-)$ is a preference relation (i.e., an SPO).
  The relation $(\,\succ -\ P^-)$ is called the \emph{\contractedrel{}}.

  A relation $P^-$ is \emph{a minimal \contraction{} of $\,\succ$ by $CON$} if
  $P^-$ is a \contraction{} of $\,\succ$ by $CON$, and there is no other \contraction{} 
  $P'$ of $\,\succ$ by $CON$ s.t. $P' \subset P^-$.
\end{definition}

\begin{definition}
   A preference relation is \emph{minimally contracted} if it is contracted by
  a \emph{minimal \contraction{}}. \emph{Contraction} is an operation of constructing a \contraction{}. \emph{Minimal contraction} 
  is an operation of constructing a minimal \contraction{}.
\end{definition}

The notion of a minimal \contraction{} narrows the set of \contractions{}.
However, as we illustrate in Example \ref{ex:min-con-inf}, 
a minimal \contraction{} is generally not unique for the given preference
and \contracting{} relations. 
Moreover, the number of minimal \contractions{} for infinite preference relations
can be infinite. Thus, minimal contraction
differs from minimal preference revision \cite{Chomicki2007} which is uniquely
defined for given preference and revising relations.

               \begin{figure}[ht]
		 \centering                              
 
			\begin{tikzpicture}
				\tikzstyle{cir} = [draw=black,rounded corners,inner sep=2pt]
				\node[cir] (u) at (0, 0) {{\small $x_1$}};
				\node[cir] (x) at (1, 0) {{\small $x_2$}};
				\node[cir] (y) at (2, 0) {{\small $x_3$}};
				\node[cir] (v) at (3, 0) {{\small $x_4$}};
				
				\draw[->, dashed, bend left=60] (u) to  (v);
				\draw[->, bend left=60] (u) to (x);
				\draw[->, bend left=60] (u) to (y);
				
				\draw[->, bend left=60] (x) to (y);
				\draw[->, bend left=60] (x) to (v);
		
				\draw[->, bend left=60] (y) to (v);
			\end{tikzpicture}
		   \caption{$\pref$ and $CON$}
		   \label{pic:contr-example1} 
		\vspace{-4mm}
	       \end{figure}

\begin{example}\label{ex:min-con-inf}
  Take the preference relation $\,\succ$ which is a total order of $\{x_1, \ldots, x_4\}$ 
  (Figure \ref{pic:contr-example1}). 
  Let the \contracting{} relation $CON$ be $\{x_1x_4\}$.
  Then the following sets are 
minimal \contractions{} of $\,\succ$ by $CON$:
$P^-_1 = \{x_1x_2,x_1x_3,x_1x_4\}$, $P^-_2 = \{x_3x_4, x_2x_4, x_1x_4\}$, $P_3 = \{x_1x_2, x_3x_4, x_1x_4\}$, 
and $P^-_4 = \{x_1x_3, x_2x_4, x_2x_3, x_1x_4\}$.
\end{example}

An important observation here is that that the contracted preference relation is defined as a 
subset of the original preference relation. We want to preserve the SPO properties -- transitivity and irreflexivity -- 
of preference relations. Since any subset of an irreflexive relation
is also an irreflexive relation, no additional actions are needed to preserve
irreflexivity during contraction. However, not every
subset of a transitive relation is a transitive relation. 
We need to consider paths in the original preference relation which by transitivity may produce
$CON$-edges which need to be discarded. We call such paths \emph{$CON$-detours}.

\begin{definition}
  Let $\ \succ$ be a preference relation, and $P \subseteq \ \succ$. 
  Then a $\,\succ$-path from $x$ to $y$ is a \emph{$P$-detour} 
  if $xy \in P$.
\end{definition}

First, let us consider the problem of finding any
\contraction{}, not necessary minimal. 
As we showed above, a contracted preference relation cannot have any $CON$-detours. 
To achieve that, some additional edges of the preference relation have to be discarded. 
However, when we discard these edges, we have to make sure that there are no 
paths in the contracted preference relation which produce the removed edges.
Hence, a necessary and sufficient
condition for a subset of a preference relation \emph{to be its \contraction{}} can be formulated
in an intuitive way. 

\begin{lemma}\label{lemma:contr-necc}
  Given a preference relation (i.e., an SPO) $\,\succ$ and a \contraction $CON$, 
  a relation $P^- \subseteq\,\succ$ is a \contraction of $\pref$ by $CON$ if and only if
  $CON \subseteq P^-$, and for every \mbox{$xy \in P^-$}, \mbox{$(\,\succ -\ P^-)$} contains no paths
  from $x$ to $y$.
\end{lemma}

\begin{proof}$ $

    \leftsideproof Prove that if for all $xy \in P^-$, $(\pref -\,P^-)$ contains no 
     paths from $x$ to $y$, then $(\pref -\,P^-)$ is an SPO. The irreflexivity of $(\pref - \,P^-)$
	follows from the irreflexivity of $\pref$. Assume $(\pref - \,P^-)$ is not transitive, i.e.,
	there are $xz, zy \in (\pref -\,P^-)$ but $xy \not \in (\pref -\,P^-)$. If 
	$xy \in P^-$ then the path $xz, zy$ is not disconnected which contradicts the initial assumption. 
	If $xy \not \in P^-$, then the assumption of transitivity of $\succ$ is violated. 
	
    \rightsideproof First, $CON \not \subseteq P^-$ implies that $P^-$ is not a \contraction of $\pref$ by $CON$ by definition. 
	Second, the existence of a path from $x$ to $y$ in $(\succ - \,P^-)$ 
	for $xy \in P^-$ implies that $(\pref -\,P^-)$ is not transitive, which violates
	the SPO properties. \qed
\end{proof}

Now let us consider the property of minimality of \contractions{}. Let $P^-$ be any minimal 
\contraction{} of a preference relation $\pref$ by a \contractingrel{} $CON$. Pick 
any edge $xy$ of $P^-$. An important question which arises here is 
\emph{why is $xy$ a member of $P^-$}? The answer is obvious if $xy$ is also a member of $CON$:
every $CON$-edge has to be removed from the preference relation. However, what if $xy$ is not 
a member of $CON$? To answer this question, let us introduce the notion of 
the \emph{outer edge set} of an edge belonging to a \contraction{} relation. 

\begin{definition}\label{def:cseq}
  Let $CON$ be a \contractingrel{} of a preference relation $\succ$,
  and $P^-$ be a \contraction{} of $\succ$ by $CON$. Let 
  $xy \in P^- - CON$, and 
      \begin{tabbing}
	$\cseq_0(xy)= $ \= $\{xy\},$ and \\
        $\cseq_i(xy) =$ \> $\{u_iv_i \in P^- | \exists u_{i-1}v_{i-1} \in $ \= $\cseq_{i-1}(xy)\ .\ $ \= 
             $u_{i} = u_{i-1} \wedge v_{i-1}v_i \in (\pref - \,P^-) \vee $ \\
             \>\> $v_{i-1} =  v_{i} \wedge u_{i} u_{i-1} \in (\pref - \,P^-) \},$ for $i > 0$.
      \end{tabbing}
  
  Then the \emph{outer edge set $\cseq(xy)$ for $xy$} is defined as
  $$\cseq(xy) = \bigcup_{i=0}^{\infty}\cseq_i(xy).$$
\end{definition}

               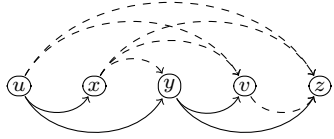
\begin{figure}[ht]
			\vspace{-7mm}
		 \centering                              
			\begin{tikzpicture}
			 \tikzstyle{cir} = [draw=black,rounded corners,inner sep=2pt];
			 \tikzstyle{upperedge} = [bend left=60, ->];
			 \tikzstyle{loweredge} = [bend right=60, ->];
			 \node[cir] (u) at (1, 0) {{\scriptsize $u$}};
			 \node[cir] (x) at (2, 0) {{\scriptsize $x$}};
			 \node[cir] (y) at (3, 0) {{\scriptsize $y$}};
			 \node[cir] (v) at (4, 0) {{\scriptsize $v$}};
			 \node[cir] (z) at (5, 0) {{\scriptsize $z$}};

			\draw[loweredge] (u) to (x);
			\draw[upperedge, dashed] (x) to (y);
			\draw[loweredge] (y) to (v);
			\draw[loweredge, dashed] (v) to (z);
			\draw[upperedge, dashed] (x) to (z);

			\draw[upperedge, dashed] (x) to (v);
			\draw[upperedge, dashed] (u) to (v);
			\draw[upperedge, dashed] (u) to (z);

			\draw[loweredge] (u) to (y);
			\draw[loweredge] (y) to (z);
			\end{tikzpicture}

			\vspace{-3mm}
		   \caption{$\cseq(xy)$ for Example \ref{ex:cseq}.}
		   \label{pic:cseq-constr}
	       \end{figure}

Intuitively, the outer edge set $\cseq(xy)$ of an edge $xy \in (P^- - \,CON)$ contains all the edges
of a \contraction{} $P^-$ which should be removed from $P^-$ (i.e., added back to the preference relation $\pref$) to preserve the \contraction{} property of the result, should $xy$ be removed from $P^-$ (i.e., added back to the preference relation). 
The reasoning here is as follows. When for some $i$, $\cseq_{i}(xy)$ is removed from $P^-$, 
then $\cseq_{i+1}(xy)$ has to be also removed from $P^-$. Otherwise, for every edge in $\cseq_{i+1}(xy)$, 
there is a two-edge path in $\pref$ one of whose edges is in $\cseq_i(xy)$ while the other is not contracted.
Hence, if the SPO properties of $(\pref - \,P^-)$ need to be preserved, 
removing $xy$ from $P^-$ requires recursively removing the entire $\cseq(xy)$ from $P^-$.

The next example illustrates the inductive construction of an outer edge set. 
Some properties of outer edge sets are shown in Lemma \ref{lemma:cool-seq}.
 
\begin{example}\label{ex:cseq}
  Let a preference relation $\succ$ be the set of all edges 
  in Figure \ref{pic:cseq-constr},
  and $P^-$ be defined by the dashed edges. 
  Let us construct $\cseq(xy)$ (assuming that $xy$ is not an 
  edge of the \contractingrel{} $CON$).

  \begin{itemize}
    \item $\cseq_0(xy) = \{xy\}$;
    \item $\cseq_1(xy) = \{xv, xz\}$;
    \item $\cseq_2(xy) = \{uv, uz\}$;
  \end{itemize}

  Thus,  $\cseq(xy) = \{xy, xv, xz, uv, uz\}$.
\end{example}

\begin{lemma}\label{lemma:cool-seq}
  Let $P^-$ be a \contraction{} of a preference relation $\pref$ by a \contractingrel{}
  $CON$. Then for every $xy \in (P^- - CON)$, $\cseq(xy)$ has the following properties:

  \begin{enumerate}
    \item  for all $uv \in \cseq(xy)$, $u \succeq x$ and $y \succeq v$;
    \item  for all $uv \in \cseq(xy)$, $ux, yv \not \in P^-$;
    \item  if $(\cseq(xy) \cap CON) = \emptyset$, then $P' = (P^- - \cseq(xy))$ is a \contraction{}
      of $\succ$ by $CON$.
  \end{enumerate}
\end{lemma}

\begin{proof}
  First, we prove that Properties 1 and 2 hold. We do it by induction on the index of $\cseq_i(xy)$
  used to construct $\cseq(xy)$.
    For every $uv \in \cseq_0(xy)$, Properties 1 and 2 hold by the construction of $\cseq_0$.
    Now let Properties 1 and 2 hold for $\cseq_n(xy)$, i.e.,
      \begin{equation*}\label{eq:cool-seq1}\tag{1}
      \forall u_nv_n \in \cseq_n(xy) \rightarrow u_n \succeq x \wedge y \succeq v_n
      \wedge u_nx,yv_n \not \in P^-
      \end{equation*}

      Pick any $u_{n+1}v_{n+1} \in \cseq_{n+1}(xy)$. By construction of $\cseq_{n+1}(xy)$, we have
      \begin{align*}\label{eq:cool-seq2}
         \exists u_{n}v_{n} \in \cseq_{i}(xy)\ .\ 
             & u_{n+1} = u_{n} \wedge v_{n} \succ v_{n+1} \wedge v_{n}v_{n+1} \not \in P^- \vee\\
             & u_{n+1} \succ u_{n} \wedge v_{n} =  v_{n+1} \wedge u_{n+1}u_{n} \not \in P^-\tag{2}
      \end{align*}
	Note that $u_{n+1} \succeq x$ and $y \succeq v_{n+1}$ follows from \eqref{eq:cool-seq1}, 
	\eqref{eq:cool-seq2}, and transitivity of $\succeq$. 
	Similarly, $u_{n+1}x, yv_{n+1} \not \in P^-$ is implied by 
	\eqref{eq:cool-seq1}, \eqref{eq:cool-seq2}, and transitivity of $(\pref -\, P^-)$. Hence, 
	Properties 1 and 2 hold for $\cup_{i=0}^n\cseq_i(xy)$ for any $n$.
      
      Now we prove Property 3: $(\pref -\, P')$ is an SPO 
	and $CON \subseteq P'$. The latter follows from $CON \subseteq P^-$ and
	$\cseq(xy) \cap CON = \emptyset$. Irreflexivity of $(\pref - \,P')$ follows
	from irreflexivity of $\pref$. Assume $(\pref - \,P')$ is not transitive, i.e.,
	there are $uv \not \in (\pref -\,P')$ and $uz,zv \in (\pref - \,P')$. Transitivity
	of $(\succ -\, P^-)$ implies that at least one of $uz, zv$ is in $\cseq(xy)$.
	However, Property 1 implies that \emph{exactly one} of $uz, zv$ is in
	$\cseq(xy)$ and the other one is not in $\cseq(xy)$ and thus in $(\pref -\,P^-)$. However, 
	$uz \in \cseq(xy)$ and $zv  \in (\pref - \,P^-)$ imply $uv \in \cseq(xy)$, and 
	thus $uv \in (\pref - \,(P^- - \cseq(xy))) = (\pref - \,P')$, 
	i.e., we derive a contradiction. A similar contradiction is derived in the case
	$uz \in (\pref - \,P^-)$ and $zv \in \cseq(xy)$. Therefore, $(\pref -\, P')$ is an SPO and $P'$ is a full contractor
	of $\pref$ by $CON$. \qed	
\end{proof}

Out of the three properties shown in Lemma \ref{lemma:cool-seq}, the last one is the most important.
It says that if an edge $xy$ of a \contraction{} is not needed to disconnect any $CON$-detours, then 
that edge may be dropped from the \contraction{} along with its entire outer edge set. 
A more general result which follows from Lemma \ref{lemma:cool-seq} is formulated 
in the next theorem. It represents a necessary and sufficient condition 
for a \contraction{} to be \emph{minimal}.

\begin{theorem}\label{thm:min-con-criterion}{\bf (Full-contractor minimality test).\ }
  Let $P^-$ be a \contraction{} of $\,\succ$ by $CON$. Then $P^-$ is a \emph{minimal}
  \contraction{} of $\,\succ$ by $CON$ if and only if for every $xy \in P^-$,   
there is a $CON$-detour in $\,\succ$ in which $xy$ is the only $P^-$-edge. 
\end{theorem}

\begin{proof}$ $

  \leftsideproof The proof in this direction is straightforward. Assume that for every edge of the 
  \contraction{} $P^-$ there exists at least one $CON$-detour
  in which only that edge is in $P^-$. If $P^-$ loses any its subset $P$ containing that edge, then
  there will be a $CON$-detour in $\pref$ having no edges in $(P^- - \,P)$, and thus
  $(P^- - \,P)$ is not a \contraction{} of $\pref$ by $CON$ by Lemma \ref{lemma:contr-necc}. 
  Hence, $P^-$ is a minimal \contraction{}.

  \rightsideproof Let $P^-$ be a minimal \contraction{}. For the sake of contradiction, assume
	for some $xy \in P^-$, 1) there is no $CON$-detour which $xy$ belongs to, or 2)
  any $CON$-detour $xy$ belongs to has at least one more $P^-$-edge. If 1) holds, 
  then $\cseq(xy)$ has no edges in $CON$ by construction. Thus, Lemma \ref{lemma:cool-seq}
  implies that $(P^- - \cseq(xy))$ is a \contraction{} of $\succ$ by $CON$. Since $\cseq(xy)$ is not empty,
  we get that $P^-$ is not a minimal \contraction{} which is a contradiction. If 2) holds, then we use the same argument as
  above and show that $\cseq(xy) \cap CON = \emptyset$. If $\cseq(xy) \cap CON$ is not empty (i.e.,
  some $uv \in \cseq(xy) \cap CON$), then by Lemma \ref{lemma:cool-seq}, 
	$$u \succeq x \wedge x \succ y \wedge y \succeq v \wedge ux, yv \not \in P^-,$$
  and thus there is a $CON$-detour going from $u$ to $v$ in which $xy$ is the only $P^-$-edge. 
  This contradicts the initial assumption. \qed
\end{proof}

Note that using the definition of minimal \contraction{} to check the minimality of a \contraction{} $P^-$
requires checking the \contraction{} properties of \emph{all subsets} of $P^-$. 
In contrast, the minimality checking method shown in Theorem \ref{thm:min-con-criterion} 
requires checking properties of \emph{distinct elements} of $P^-$ with respect to its other members.

Sometimes a direct application of the minimality test from Theorem \ref{thm:min-con-criterion} 
is hard because it does not give any bound on the length of $CON$-detours. Hence, it is not clear
how it can be represented as a finite formula. 
Fortunately, the transitivity of preference relations implies that 
the minimality condition from Theorem \ref{thm:min-con-criterion} 
can be stated in terms of paths of length at most three.

\begin{corollary}\label{col:simple-cond}
  A \contraction{} $P^-$ of $\pref$ by $CON$ is \emph{minimal} if and only if for every edge $xy \in P^-$, 
  there is a $CON$-detour consisting of at most three edges among which only $xy$ is in $P^-$. 
\end{corollary}

\begin{proof}$ $\\
 \leftsideproof Trivial.

 \rightsideproof For every $xy \in P^-$, pick any $CON$-detour $T$ in which the only $P^-$-edge is $xy$. If its length 
 is less or equal to three, then the corollary holds. Otherwise, $x$ is not the start node of $T$, or
 $y$ is not the end node of $T$, or both. Let the start node $u$ of $T$ be different from $x$. 
 Since the only common edge of $T$ and $P^-$ is $xy$,
 every edge in the path from $u$ to $x$ is an element of $(\pref - \,P^-)$. Transitivity of $(\pref - \,P^-)$
 implies $ux \in (\pref - \,P^-)$. Similarly, $yv \in (\pref - \,P^-)$ for the end node of $T$ if $y$ is different
 from $v$. Hence, there is a $CON$-detour of length at most three in which $xy$ is the only element of $P^-$. \qed 
\end{proof}

As a result, the following tests can be used to check the minimality of a \contraction $P^-$. 
In the finite case, $P^-$ is minimal if the following relational algebra expression
results in an empty set

	\begin{tabbing}
	 P - \=$[\pi_{\normsubscr{P}_2.\normsubscr{X}, \normsubscr{P}_2.\normsubscr{Y}}$((R$_1$ - P$_1$) $\join{\normsubscr{R}_1.\normsubscr{Y} = \normsubscr{P}_2.\normsubscr{X}}$ P$_2$ 
			$\join{\normsubscr{P}_2.\normsubscr{Y} = \normsubscr{R}_3.\normsubscr{X}}$ (R$_3$ - P$_3$) 
			$\join{\normsubscr{R}_1.\normsubscr{X} = C.X, \ \normsubscr{R}_3.\normsubscr{Y} = \normsubscr{C}.\normsubscr{Y}}$ C) \ $\cup$ \\
	     \> $\pi_{\normsubscr{P}_2.\normsubscr{X}, \normsubscr{P}_2.\normsubscr{Y}}$(P$_2$ 
			$\join{\normsubscr{P}_2.\normsubscr{Y} = \normsubscr{R}_3.\normsubscr{X}}$ (R$_3$ - P$_3$) 
			$\join{\normsubscr{P}_2.\normsubscr{X} = C.X, \ \normsubscr{R}_3.\normsubscr{Y} = \normsubscr{C}.\normsubscr{Y}}$ C) \ $\cup$ \\
	     \> $\pi_{\normsubscr{P}_2.\normsubscr{X}, \normsubscr{P}_2.\normsubscr{Y}}$((R$_1$ - P$_1$) $\join{\normsubscr{R}_1.\normsubscr{Y} = \normsubscr{P}_2.\normsubscr{X}}$ P$_2$ 
			$\join{\normsubscr{R}_1.\normsubscr{X} = C.X, \ \normsubscr{P}_2.\normsubscr{Y} = \normsubscr{C}.\normsubscr{Y}}$ C) \ $\cup$ 
	      C $]$,
	\end{tabbing}

	for the tables R, C and P with columns X and Y, storing $\pref$, $CON$, and $P^-$ correspondingly.
In the finitely representable case, $P^-$ is minimal if the following
formula is valid
   \begin{align*}
  \forall x,y \ (\form{P^-}(x,y) \Rightarrow \form{\succ}(x,y) \wedge 
	         \exists u,v  \ . \ & \form{CON}(u,v) \wedge 
			      ( \form{\succ}(u,x) \vee u = x) \wedge \\ 
			      & (\form{\succ}(y,v) \vee y = v) \wedge 
  			      \neg \form{P^-}(u,x) \wedge \neg \form{P^-}(y,v)).
   \end{align*}

Below we show examples of checking minimality of \contractions using Corollary 
\ref{col:simple-cond}. We note that when the relations are definable using \ero-formulas,
checking minimality of a \contraction{}
can be done by performing quantifier elimination on the above formula.

\begin{figure}
\centering
 \subfigure[Infinite case, Example \ref{ex:min-crit-inf}]{
	\begin{tikzpicture}[scale=0.8]
	 \tikzstyle{point} = [draw=black,rounded corners,inner sep=2pt];
	 \tikzstyle{upperedge} = [bend left=60, ->];
	  	\node (s) at (-0.5, 0) {};
	 	\node [label=below:{\scriptsize 0}, fill=black, inner sep=1pt] (n0) at (0, 0) {};
		\node [label=below:{\scriptsize 1}, fill=black, inner sep=1pt] (n1) at (1, 0) {};
		\node [label=below:{\scriptsize 2}, fill=black, inner sep=1pt] (n2) at (2, 0) {};
		\node [label=below:{\scriptsize 3}, fill=black, inner sep=1pt] (n3) at (3, 0) {};
		\node [label=below:{\scriptsize 4}, fill=black, inner sep=1pt] (n4) at (4, 0) {};
	  	\node (e) at (4.5, 0) {};
 		\draw (s) -- (n0) -- (n1) -- (n2) -- (n3) -- (n4) -- (e);
	 
		\draw[top color=black!30, bottom color = black!10, dashed] (1,0) arc (180:0:1.5 and 1) -- (4,0) arc (0:180:1 and 0.5) -- (n1);
		\draw[upperedge, dashed] (n0) to (n3);
		\draw (4.1,0.1) -- (4,0) -- (3.9, 0.1);
		\node(L) at (1.5, 1.5) {{\scriptsize infinitely many edges}};
		\node(P) at (3, 0.7){};
		\draw[-] (L) to (P);
		\draw (s) -- (n0) -- (n1) -- (n2) -- (n3) -- (n4) -- (e);
	\end{tikzpicture}
	\label{pic:min-crit-1}
  }
 \subfigure[Finite case, Example \ref{ex:min-crit-fin}]{
	\begin{tikzpicture}
	 \tikzstyle{point} = [draw=black,rounded corners,inner sep=2pt];
	 \tikzstyle{upperedge} = [bend left=60, ->];
		\node at (0,5) {
			{\small
			\begin{tabular}{c||c|c}
			$R$ & X & Y\\
			\hline
			\vspace{-1mm}
			& u & x\\
			\vspace{-1mm}
			& x & y\\
			\vspace{-1mm}
			& y & v\\
			\vspace{-1mm}
			& u & y\\
			\vspace{-1mm}
			& x & v\\
			\vspace{-1mm}
			& u & v\\
			\end{tabular}
			}
		};
		\node at (2.3,5.96) {
			{\small
			\begin{tabular}{c||c|c}
			$C$ & X & Y\\
			\hline
			\vspace{-1mm}
			& u & v\\
			\end{tabular}
			}
		};
		\node at (4.6,5.4) {
			{\small
			\begin{tabular}{c||c|c}
			$P$ & X & Y\\
			\hline
			\vspace{-1mm}
			& u & x\\
			\vspace{-1mm}
			& y & v\\
			\vspace{-1mm}
			& x & v\\
			\vspace{-1mm}
			& u & v\\
			\end{tabular}
			}
		};
		\node at (6.95,5.78) {
			{\small
			\begin{tabular}{c||c|c}
			$D$ & X & Y\\
			\hline
			\vspace{-1mm}
			& u & x\\
			\vspace{-1mm}
			& x & v\\
			\end{tabular}
			}
		};
	\end{tikzpicture}
	\label{pic:min-crit-2}
  }
\caption{Checking minimality of a \contraction{}}
\label{pic:min-crit}
\vspace{-4mm}
\end{figure}
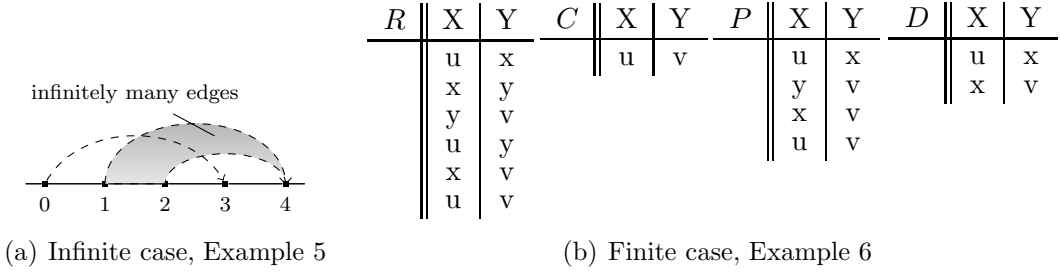

\begin{example}\label{ex:min-crit-inf}
 Let a preference relation $\pref$ be defined by the formula $\form{\succ}(o, o') \equiv o.d < o'.d$, where $d$ is a $\Q$-attribute. Let a \contractingrel{} $CON$
 of $\pref$ be defined by the formula
 $$\form{CON}(o, o') \equiv (1 \leq o.d \leq 2 \wedge o'.d = 4) \vee (o.d = 0 \wedge o'.d = 3)$$
 (Figure \ref{pic:min-crit-1}). Denote the relation represented by the first and second disjuncts of $\form{CON}$ as 
 $CON_1$ and $CON_2$ correspondingly. 
 The relation $P^-$ defined by $\form{P^-}$ is a \contraction{} of $\pref$ by $CON$
 $$\form{P^-}(o, o') \equiv (1 \leq o.d \leq 2 \wedge 2 < o'.d \leq 4 ) \vee (o.d = 0 \wedge 0 < o.d' \leq 3).$$ 
 Similarly, denote the relations represented by the first and the second disjuncts of $\form{P^-}$ as $P^-_1$ and 
 $P^-_2$ correspondingly. We use Corollary \ref{col:simple-cond} to check the minimality of $P^-$. By the corollary, we need to consider $CON$-detours of 
 length at most three. Note that every $P^-_1$-edge starts a one- or two-edge $CON$-detour with the corresponding 
 $CON_1$-edge. Moreover, the second edge of all such two-edge detours is not contracted by $P^-$. 
 Hence, the minimal \contraction{} test is satisfied for $P^-_1$-edges. Now we consider $P^-_2$-edges. 
 All $CON$-detours which these edges belong to 1) correspond to $CON_2$-edges, and 2)
 are started by $P^-_2$-edges. Hence, we need to consider only $CON_2$-detours of length at most two. 
 When a $P^-_2$-edge ends in $o'$ with the value of $d$ in $(0, 1)$ and $(2,3]$, the second edge in the 
 corresponding two-edge $CON_2$-detour is not contracted by $P^-$. However, when $d$ is in $[1, 2]$, the second edge is already
 in $P^-$. Hence, $P^-$ is not minimal by Corollary \ref{col:simple-cond}. To minimize it,
 we construct $P^*$ by removing the edges from $P^-$ which end in $o'$ with $d$ in $[1,2]$

\begin{tabbing}
 \hspace{2cm}$\form{P^*}(o, o') \equiv $ \= $(1 \leq o.d \leq 2 \wedge 2 < o'.d \leq 4 ) \vee $\\
			     \> $(o.d = 0 \wedge (0 < o.d' < 1 \vee 2 < o'.d \leq 3))$
\end{tabbing}
\end{example}

\begin{example}\label{ex:min-crit-fin}
 Take a preference represented by the table $R$, and a \contracting{} $R$
 represented by the table $C$ (Figure \ref{pic:min-crit-2}). 
 Consider the table $P$ representing a \contracting{} of $R$ by $C$.
 Then the result of the relational algebra expression above 
 evaluated for these tables is shown in the
 table $D$. Since it is not empty, the \contraction{} represented by $P$ is not minimal.
 The minimality of $P$ can be achieved by removing from it any (but only one) tuple in $D$.
\end{example}

\section{Construction of a minimal \contraction}\label{sec:min-con-alg}

In this section, we propose a method of computing a minimal \contraction{}. 
We use the idea shown in 
Example \ref{ex:min-con-inf}. Pick for instance the set $P^-_1$. That
set was constructed as follows: we took the $CON$-edge
$x_1x_4$ and put in $P^-_1$ all the edges which start some
path from $x_1$ to $x_4$. For the preference relation $\,\succ$ 
from Example \ref{ex:min-con-inf}, $P^-_1$ turned out to be
a minimal \contraction{}.
As is it shown in the next lemma, the set consisting
of all edges starting $CON$-detours is a \contraction{} by $CON$.

\begin{lemma}\label{lemma:smaller-con}
  Let $\,\succ$ be a preference relation 
  and $CON$ be a \contracting{} relation of $\,\succ$.
  Then
  $$P^- := \{\ xy\ |\ \exists x'v \in CON \ . \ x' = x \wedge x' \succ y \wedge y \succeq v \} $$
  is a \contraction{} of $\,\succ$ by $CON$.
\end{lemma}

\begin{proof}
  By construction of $P^-$, $CON \subseteq P^-$. Lemma \ref{lemma:contr-necc} implies
$(\pref - \,P^-)$ is an SPO. Hence, $(\pref - \,P^-)$ is a \contraction{} of $\pref$ by $CON$. \qed
\end{proof}

However, in the next example we show that such a \contraction{} 
is not always minimal. Recall that by Theorem \ref{thm:min-con-criterion}, for every edge of a 
\contraction{} there should be a $CON$-detour which only shares that edge with the contractor. However, 
it may be the case that an edge starting a $CON$-detour does not have to be discarded because 
the $CON$-detour is already disconnected.
 
  \begin{figure}[ht]
	      \begin{center}
		\subfigure[Preference relation $\pref$]{
			\begin{tikzpicture}[scale=0.55]
			 \tikzstyle{cir} = [draw=black,rounded corners,inner sep=2pt];
			 \tikzstyle{upperedge} = [bend left=60, ->];
			 \tikzstyle{loweredge} = [bend right=60, ->];

			 \node           at (-1, 0) {};
			 \node[cir] (x1) at (1, 0) {{\scriptsize $x_1$}};
			 \node[cir] (x2) at (2, 0) {{\scriptsize $x_2$}};
			 \node[cir] (x3) at (3, 0) {{\scriptsize $x_3$}};
			 \node[cir] (x4) at (4, 0) {{\scriptsize $x_4$}};
			 \node[cir] (x5) at (5, 0) {{\scriptsize $x_5$}};
			 \node           at (6, 0) {};
			  
			 \draw[upperedge] (x1) to (x5);
			 \draw[upperedge, dashed] (x1) to (x4);
			 \draw[upperedge] (x1) to (x3);			
			 \draw[upperedge] (x1) to (x2);

			 \draw[loweredge, dashed] (x2) to (x5);
			 \draw[loweredge] (x2) to (x4);
			 \draw[loweredge] (x2) to (x3);

			 \draw[loweredge] (x3) to (x4);
			 \draw[loweredge] (x3) to (x5);

			 \draw[loweredge] (x4) to (x5);
			\end{tikzpicture}		
		   \label{pic:not-min-con-2-1}
		}
		\subfigure[Preference relation $(\pref -\,P^-)$]{
			\begin{tikzpicture}[scale=0.55]
			 \tikzstyle{cir} = [draw=black,rounded corners,inner sep=2pt];
			 \tikzstyle{upperedge} = [bend left=60, ->];
			 \tikzstyle{loweredge} = [bend right=60, ->];

			 \node           at (-1, 0) { $ $};
			 \node[cir] (x1) at (1, 0) {{\scriptsize $x_1$}};
			 \node[cir] (x2) at (2, 0) {{\scriptsize $x_2$}};
			 \node[cir] (x3) at (3, 0) {{\scriptsize $x_3$}};
			 \node[cir] (x4) at (4, 0) {{\scriptsize $x_4$}};
			 \node[cir] (x5) at (5, 0) {{\scriptsize $x_5$}};
			 \node           at (7, 0) {$ $};
			  
			 \draw[upperedge] (x1) to (x5);
			 \draw[upperedge, dashed] (x1) to (x4);
			 \draw[upperedge, dashed] (x1) to (x3);			
			 \draw[upperedge, dashed] (x1) to (x2);

			 \draw[loweredge, dashed] (x2) to (x5);
			 \draw[loweredge, dashed] (x2) to (x4);
			 \draw[loweredge, dashed] (x2) to (x3);

			 \draw[loweredge] (x3) to (x4);
			 \draw[loweredge] (x3) to (x5);

			 \draw[loweredge] (x4) to (x5);
			\end{tikzpicture}

		   \label{pic:not-min-con-2-2}
		}
		\subfigure[Minimally contracted $\,\succ$]{
			\begin{tikzpicture}[scale=0.55]
			 \tikzstyle{cir} = [draw=black,rounded corners,inner sep=2pt];
			 \tikzstyle{upperedge} = [bend left=60, ->];
			 \tikzstyle{loweredge} = [bend right=60, ->];

			 \node              (-1, 0) {$ $};
			 \node[cir] (x1) at (1, 0) {{\scriptsize $x_1$}};
			 \node[cir] (x2) at (2, 0) {{\scriptsize $x_2$}};
			 \node[cir] (x3) at (3, 0) {{\scriptsize $x_3$}};
			 \node[cir] (x4) at (4, 0) {{\scriptsize $x_4$}};
			 \node[cir] (x5) at (5, 0) {{\scriptsize $x_5$}};
			  
			 \draw[upperedge] (x1) to (x5);
			 \draw[upperedge, dashed] (x1) to (x4);
			 \draw[upperedge, dashed] (x1) to (x3);			
			 \draw[upperedge] (x1) to (x2);

			 \draw[loweredge, dashed] (x2) to (x5);
			 \draw[loweredge, dashed] (x2) to (x4);
			 \draw[loweredge, dashed] (x2) to (x3);

			 \draw[loweredge] (x3) to (x4);
			 \draw[loweredge] (x3) to (x5);

			 \draw[loweredge] (x4) to (x5);
			\end{tikzpicture}

		   \label{pic:not-min-con-2-3}
		}	        
	      \end{center}    
     \vspace{-0.7cm}
     \caption{Preference contraction}
     \label{pic:non-min-con2}	
	\vspace{-5mm}
  \end{figure}
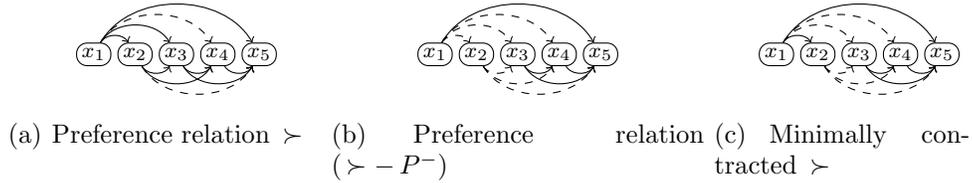

\begin{example}\label{ex:smaller-con-issues}
  Let a preference relation $\pref$ be a total order of $\{x_1, \ldots, x_5\}$
  (Figure \ref{pic:not-min-con-2-1}). 
  Let a \contractingrel{} 
  $CON$ be $\{x_1x_4, x_2x_5\}$.
  Let $P^-$ be defined as in Lemma \ref{lemma:smaller-con}. That is
  $P^- = \{x_1x_2, x_1x_3, x_1x_4, x_2x_3, x_2x_4, x_2x_5\}$.
  Then
  $(\,\succ -\: P^-)$ is shown in Figure \ref{pic:not-min-con-2-2} as the set of solid edges.
  $P^-$ is not minimal because 
  $(P^- - \{x_1x_2\})$ (Figure \ref{pic:not-min-con-2-3}) is also a \contraction{} of 
  $\,\succ$ by $CON$. In fact, $(P^- - \{x_1x_2\})$ is a \emph{minimal} \contraction{}
  of $\,\succ$ by $CON$.
  As we can see, having the edge $x_1x_2$ in $P^-$ is not necessary. 
  First, it is not a $CON$-edge. Second, 
  the edge $x_2x_4$ of the $CON$-detour $x_1 \succ x_2 \succ x_4$ is already in $P^-$.
\end{example}

As we have shown in Example \ref{ex:smaller-con-issues}, 
a minimal \contraction{} can be constructed by including in it only the edges which
start some $CON$-detour, if the detour is not already disconnected. 
Thus, before adding such an edge to a \contraction{}, we need to know if 
an edge \emph{not starting} that detour is already in the \contraction.
Here we propose the following idea of computing a minimal \contraction{}. 
Instead of contracting $\pref$ by $CON$ at once, split $CON$ into \emph{\layers}, and 
contract $\pref$ incrementally by the \layers of $CON$. A \layer of $CON$ consists
of only those edges whose detours can be disconnected simultaneously in a minimal way. 
The method of splitting a \contraction into \layers we propose to use is as follows.

\begin{definition}\label{def:k-dest}
  The \emph{\layerindex} of an edge $xy \in CON$ is
  \emph{the maximum length of a $\,\succ$-path started by $y$ and 
    consisting of the end nodes of $CON$-edges}. A \emph{\layer} is
    the set of all $CON$-edges with the same \emph{\layerindex}.
\end{definition}

This method of stratification has the following useful property. If a preference relation 
is contracted minimally by the \layers with indices of up to $n$, then contracting
that relation minimally by the \layer with the index $n+1$ minimally guarantees the minimality
of the entire contraction. 

Clearly, if a preference relation is infinite, a tuple can start $\succ$-paths of arbitrarily large lengths. 
Therefore, the \layerindex of $CON$-edge may be undefined. We exclude such cases 
here, so we can assume that for each edge of $CON$ relations, the \layerindex is defined.

\begin{definition}\label{def:bounded-rel}
 Let $CON$ be a \contractingrel{} of a preference relation $\pref$. Let $K_{CON} = \{y \ | \ \exists x\ .$ $\ xy \in CON\}$, 
 and $\succ_{CON} \ = \ \succ \cap \ K_{CON} \times K_{CON}$. 
 Then $CON$ is \emph{\divtolayers} iff for every 
	$y \in K_{CON}$ there is an integer $k$ such that all the paths started by $y$ in 
  $\succ_{CON}$ are of length at most $k$.
 $CON$ is \emph{\boundedlayer} iff there is a constant $k$ such that
	all paths in $\succ_{CON}$ are of length at most $k$.
\end{definition}

Definition \ref{def:bounded-rel} implies that for every edge 
of \divtolayers $CON$, the \layerindex is defined. 
 Since the shortest path 
in $\succ_{CON}$ is of length 0, the least \layerindex for \divtolayers relations is 0.
Instances of \boundedlayer \contractingrels are shown in 
Example \ref{ex:motiv} ($k=0$), Example \ref{ex:motiv-finite} ($k=1$), and Example \ref{ex:min-crit-inf} 
($k = 1$). 
Below we present an approach of constructing a minimal \contraction{} for a \emph{\divtolayers} relation $CON$.

\begin{theorem}\label{thm:bounded-rel-con}{\bf (Minimal \contraction{} construction).\ }
 Let $\pref$ be a preference relation, and $CON$ be a \divtolayers \contractingrel{} of $\pref$.
 Let $L_i$ be the set of the end nodes of all $CON$-edges of \layer $i$. Then $P^-$, defined as follows,
 is a \emph{minimal \contraction{}} of $\succ$ by $CON$

 $$P^- = \bigcup_{i \in 0}^{\infty} E_i,$$
 where 
	$$
	E_i = \{xy \ | \ \exists v \in L_i\ .\ xv \in CON \wedge x \succ y \wedge y \succeq v \wedge	
	yv \not \in ( P^-_{i-1} \cup CON)\}
	$$
	$$
	P^-_{-1} = \emptyset, 
	$$
	$$
	P^-_i = \bigcup_{j = 0}^i E_i
	$$
	
\end{theorem}

Intuitively, the set $E_i$ contains all the $CON$ edges of \layer $i$ along with the edges of $\pref$ which 
need to be discarded to contract the preference relation by that \layer. $P^-_i$ is the union of all such 
sets up to \layer $i$. 

\begin{proofofthm-bounded-rel-con}
 Every $E_i$ containts the $CON$-edges of \layer $i$. Thus, $P^-$ contains $CON$. Now we prove that
 $(\pref -\,P^-)$ is an SPO. Its irreflexivity follows from the irreflexivity of $\pref$. Transitivity
 is proved by induction on \layerindex.

 It is given that $\pref$ is transitive. Now assume $(\pref - \,P^-_{n})$ is transitive. Prove that 
 $(\pref - \,P^-_{n+1}) = (\pref - \,P^-_n - E_{n+1})$ is transitive. For the sake of contradiction, assume
	\begin{align*}\label{eq:bounded-rel-con-1}
	\exists x, y, z\ .\ xy \not \in (\pref - \,P^-_{n+1}) \wedge xz, zy \in (\pref - \,P^-_{n+1}) \tag{1}
	\end{align*}
 which implies 
	\begin{align*}\label{eq:bounded-rel-con-3}
	xz, zy \not \in E_{n+1} \cup P^-_n \tag{2}
	\end{align*}
 Transitivity of $(\pref - \,P^-_{n})$ and \eqref{eq:bounded-rel-con-1} imply $xy \in (\pref - \,P^-_{n})$ and thus
 $xy \in E_{n+1}$. Hence, 
	\begin{align*}\label{eq:bounded-rel-con-2}
 		\exists v \in L_n\ .\ xv \in CON \wedge x \succ y \wedge y \succeq v \wedge	
	yv \not \in ( P^-_{n} \cup CON) 	\tag{3}
	\end{align*}
 According to \eqref{eq:bounded-rel-con-2}, $y \succeq v$. If $y = v$, then \eqref{eq:bounded-rel-con-3} and \eqref{eq:bounded-rel-con-2} imply $xz \in E_{n+1}$ which is a contradiction.
 If $y \succ v$, then $xz \not \in E_{n+1}$ implies $zv \in P^-_{n} \cup CON$ by the construction of $E_{n+1}$. 
	Note that $zv \in CON$ implies $zv$ is a $CON$-edge of \layerindex $n+1$ and thus either $zy \in E_{n+1}$
	or $yv \in P^-_{n} \cup CON$, which contradicts \eqref{eq:bounded-rel-con-3} and \eqref{eq:bounded-rel-con-2}.
	If $zv \in P^-_n$, then $zy, yv \not \in P^-_n$ implies intransitivity of $(\pref - \,P^-_n)$, which contradicts 
	the inductive assumption. Thus, $P^-_{n+1}$ is a \contraction{} of $\succ$ by $CON$ by induction. 
	Now assume that $(\pref -\,P^-)$ is not transitive. Violation of transitivity means that there is
	an edge $xy \in P^-$ such that there exists a path from $x$ to $y$ none of whose edges is $P^-$
	(Lemma \ref{lemma:contr-necc}). Since $xy$ must be 
	in $P^-_n$ for some $n$, that implies intransitivity of $(\pref - \,P^-_n)$, which is a contradiction.
	Thus $P^-$ is a \contraction{} of $\pref$ by $CON$.

	Now we prove that
 	$P^-$ is a \emph{minimal} \contraction{}. If it is not, then by Theorem \ref{thm:min-con-criterion}, 
	there is $xy \in P^-$ for which there is no $CON$-detour which shares with $P^-$ only the edge $xy$.
        Note that $xy \in P^-$ implies $xy \in E_n$ for some $n$. 
	By definition of $E_{n}$, there is a $CON$-detour $x \succ y \succeq v$ which shares with $P^-_n$ only $xy$.
	Since all $CON$-detours which $xy$ belongs to have other $P^-$-edges,
	$yv \in P^-$. Since $yv \not \in P^-_{n}$, there must exist $k > n$ such that $yv \in E_k$. 
	However, that is impossible by construction: every $CON$-detour which may be started by $yv$ 
	must have the \layerindex not greater than $n$. \qed
\end{proofofthm-bounded-rel-con}

  \begin{figure}[ht]
\vspace{-0.7cm}
	      \begin{center}
		\subfigure[\small{$\pref$ and $CON$}]{

			\begin{tikzpicture}[scale=0.7]
			 \tikzstyle{cir} = [draw=black,rounded corners,inner sep=2pt];
			 \tikzstyle{upperedge} = [bend left=60, ->];
			 \tikzstyle{loweredge} = [bend right=60, ->];
			 \node[cir] (x1) at (1, 0) {{\scriptsize $x_1$}};
			 \node[cir] (x2) at (2, 0) {{\scriptsize $x_2$}};
			 \node[cir] (x3) at (3, 0) {{\scriptsize $x_3$}};
			 \node[cir] (x4) at (4, 0) {{\scriptsize $x_4$}};
			 \node[cir] (x5) at (5, 0) {{\scriptsize $x_5$}};
			  
			 \draw[upperedge] (x1) to (x2);
			 \draw[upperedge] (x2) to (x3);
			 \draw[upperedge] (x3) to (x4);
			 \draw[upperedge] (x4) to (x5);

			 \draw[upperedge, dashed] (x1) to (x4);
			 \draw[loweredge, dashed] (x2) to (x5);
			\end{tikzpicture}
			\label{pic:min-con-1}
		}
		\subfigure[\small{$P^-_0$}]{
			\begin{tikzpicture}[scale=0.7]
			 \tikzstyle{cir} = [draw=black,rounded corners,inner sep=2pt];
			 \tikzstyle{upperedge} = [bend left=60, ->];
			 \tikzstyle{loweredge} = [bend right=60, ->];
			 \node[cir] (x1) at (1, 0) {{\scriptsize $x_1$}};
			 \node[cir] (x2) at (2, 0) {{\scriptsize $x_2$}};
			 \node[cir] (x3) at (3, 0) {{\scriptsize $x_3$}};
			 \node[cir] (x4) at (4, 0) {{\scriptsize $x_4$}};
			 \node[cir] (x5) at (5, 0) {{\scriptsize $x_5$}};
			  
			 \draw[upperedge] (x1) to (x2);
			 \draw[upperedge] (x3) to (x4);
			 \draw[upperedge] (x4) to (x5);

			 \draw[loweredge, dashed] (x2) to (x3);
			 \draw[loweredge, dashed] (x2) to (x4);
			 \draw[loweredge, dashed] (x2) to (x5);
			\end{tikzpicture}

		   \label{pic:min-con-2}
		}	        
		\subfigure[\small{$P^-_1$}]{
			\begin{tikzpicture}[scale=0.7]
			 \tikzstyle{cir} = [draw=black,rounded corners,inner sep=2pt];
			 \tikzstyle{upperedge} = [bend left=60, ->];
			 \tikzstyle{loweredge} = [bend right=60, ->];
			 \node[cir] (x1) at (1, 0) {{\scriptsize $x_1$}};
			 \node[cir] (x2) at (2, 0) {{\scriptsize $x_2$}};
			 \node[cir] (x3) at (3, 0) {{\scriptsize $x_3$}};
			 \node[cir] (x4) at (4, 0) {{\scriptsize $x_4$}};
			 \node[cir] (x5) at (5, 0) {{\scriptsize $x_5$}};
			  
			 \draw[upperedge] (x1) to (x2);
			 \draw[upperedge] (x3) to (x4);
			 \draw[upperedge] (x4) to (x5);

			 \draw[loweredge, dashed] (x2) to (x3);
			 \draw[loweredge, dashed] (x2) to (x4);
			 \draw[loweredge, dashed] (x2) to (x5);

			 \draw[upperedge, dashed] (x1) to (x3);
			 \draw[upperedge, dashed] (x1) to (x4);
			\end{tikzpicture}
		   \label{pic:min-con-3}
		}	        
	      \end{center}
                   \vspace{-0.7cm}
     \caption{Using Theorem \ref{thm:bounded-rel-con} to compute a minimal \contraction}
     \label{pic:min-contr}
\vspace{-0.7cm}
  \end{figure}
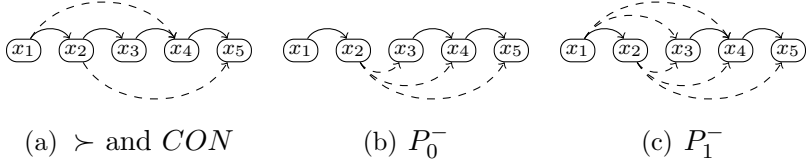

\begin{example}\label{ex:min-pref-con}
  Let a preference relation $\pref$ be a total order of $\{x_1, \ldots, x_5\}$ 
(Figure \ref{pic:min-con-1}, the transitive 
edges are omitted for clarity). Let a \contracting{} $CON$ be $\{x_1x_4, x_2x_5\}$. 
We use Theorem \ref{thm:bounded-rel-con} to construct a minimal \contraction of $\pref$ by $CON$.
The relation $CON$ has two \layers: $L_0 = \{x_2x_5\}$, $L_1 = \{x_1x_4\}$. 
Then $E_0 = \{x_2x_3, x_2x_4, x_2x_5\}$, $P^-_0 = E_0$, $E_1 = \{x_1x_3, x_1x_4\}$, 
$P^-_1 = E_0 \cup E_1$, and a minimal \contraction of $\pref$ by $CON$ is $P^- = P^-_1$. 
\end{example}

It is easy to observe that the \contraction{} $P^-$ constructed in 
Theorem \ref{thm:bounded-rel-con} has the property that its every 
edge \emph{starts} at least one $CON$-detour in which $xy$ is the only $P^-$-edge.
Full contractors which have this property are
called \emph{prefix}. Prefix \contractions{} are minimal by Theorem \ref{thm:min-con-criterion}. 
It turns out that a prefix \contraction{} is unique for a given preference relation
and a given \contractingrel{}.

\begin{proposition}\label{prop:prefix-unique}
 Given a preference relation $\succ$ and a \contractingrel{} $CON$ \divtolayers, there exists 
a unique prefix \contraction{} $P^-$ of $\pref$ by $CON$. 
\end{proposition}

\begin{proof}
 The existence of a prefix \contraction{} follows from Theorem \ref{thm:bounded-rel-con}. The fact that
 every prefix \contraction{} is equal to $P^-$ constructed by Theorem \ref{thm:bounded-rel-con} can be 
 proved by induction in $CON$ \layerindex. Namely, we show that for every $n$, $P^-_n$ is contained
 in any prefix \contraction{} of $\succ$ by $CON$. 
 Clearly, the set $E_{0}$ contracting $\pref$ by the $0^{th}$ \layer
 of $CON$ has to be in any prefix \contraction{}. Assume every edge in $P^-_{n}$ is in any 
 prefix \contraction{} of $\succ$ by $CON$. If an edge $xy \in E_{n+1} - CON$, then there is a $CON$-detour
 $x \succ y \succ v$ in which $xy$ is the only $P^-$-edge (i.e., $yv \not \in P^-$). Hence if $xy$ is not 
 in some prefix \contraction{} $P'$, then $yv$ has to be in $P'$ by Lemma \ref{lemma:contr-necc}. 
 However, $P^-_n \subset P'$ is enough to disconnect every $CON$-detour with index up to $n$, and 
 $yv$ can only start a $CON$-detour with the \layerindex up to $n$. Hence $P'$ is not a minimal 
 \contraction{} and $P^-$ is a unique prefix \contraction{}. \qed
\end{proof}

\section{Contraction by \boundedlayer relations}

In this section, we consider practical issues of computing minimal \contractions{}. In particular,
we show how the method of constructing a prefix \contraction{} we have proposed in Theorem \ref{thm:bounded-rel-con}
can be adopted to various classes of preference and \contracting{} relations. 
Note that the definition of the minimal \contraction{} in Theorem \ref{thm:bounded-rel-con}
is recursive. Namely, to find the edges we need to discard for contracting the preference
relation by the \layer $n+1$ of $CON$, we need to know which edges to discard for contracting it by 
all the previous \layers. It means that for \contracting{} relations which are not \boundedlayer
(i.e., $CON$ has infinite number of \layers), the corresponding computation will never terminate. 

Now assume that $CON$ is a \boundedlayer relation. First we note that any \contractingrel{} of 
a finite
preference relation is \boundedlayer: all paths in such preference relations are not longer than 
the size of the relation, and \contractingrels{} are required to be subsets of the preference 
relations. At the same time, if $CON$ is a \contractingrel{} of an infinite preference relation, then 
the \boundedlayerprop of $CON$ does
not imply the finiteness of $CON$. In particular, it may be the case that the \emph{length} of 
all paths in $\succ_{CON}$ is bounded, but the \emph{number} of paths is infinite. This 
fact is illustrated in the next example.

\begin{example}\label{ex:k-layer-rel}
  Let a preference relation $\,\succ$ be defined as
\mbox{$o \succ o'\equiv o.price < o'.price$}. 
Let every tuple have two $\Q$-attributes: $price$ and $year$. 
Let also the \contracting{} relations $CON_1$ and $CON_2$ be 
  defined as 
  \begin{align*}
  CON_1(o,o') \equiv & \ o.price < 1 \wedge (o'.price = 2 \vee o'.price = 3),\\
  CON_2(o,o') \equiv & \ o.price < 1 \wedge o'.price \geq 2.
  \end{align*}
	then 
  \begin{align*}
   K_{CON_1} \equiv & \ \{ o \ | \ o.price = 2 \vee o.price = 3\}\\
   K_{CON_2} \equiv & \ \{ o \ | \ o.price \geq 2\}
  \end{align*}
	and
  \begin{align*}
   o \succ_{CON_1} o' \equiv & \ o.price = 2 \wedge o'.price = 3\\
   o \succ_{CON_2} o' \equiv & \ o.price \geq 2 \wedge o'.price > 2 \wedge o.price < o'.price
  \end{align*}

  Then $CON_1$ is \emph{\boundedlayer} since despite the fact that the number of 
  edges in $\succ_{CON_1}$ and in $CON_1$ is infinite (due to the infiniteness of the domain of $year$), the length of the longest
  path in $\succ_{CON_1}$ equals to $1$. Such paths are started by tuples
  with the value of $price$ equal to $2$ and ended by tuples with $price$ equal to $3$. At the same time,
  $CON_2$ is not \boundedlayer since $price$ is a $\Q$-attribute and thus
  there is no constant bounding the length of all paths in $\pref_{CON_2}$.
\end{example}

Below we consider the cases of finite and finitely representable \boundedlayer \contractingrels{} separately.
 
\subsection{Computing prefix \contraction{}: finitely representable relations}

Here we assume that the relations $CON$ and $\pref$ are represented by finite \ero-formulas
$F_{CON}$ and $F_{\pref}$. We aim to construct a finite \ero-formula $F_{P^-}$ which 
represents a prefix \contraction{} of $\succ$ by $CON$. The 
function {\tt minContr($F_{\pref}$, $F_{CON}$)} shown below exploits the method of constructing
prefix \contractions{} from Theorem \ref{thm:bounded-rel-con} adopted to formula 
representations of relations. All the intermediate variables
used in the algorithm store formulas. Hence, for example, any expression in the form $''F(x,y):= \ldots''$ 
means that the formula-variable $F$ is assigned the formula written in the right-hand side, which
has two free tuple variables $x$ and $y$.
The operator $QE$ used in the algorithm computes a quantifier-free formula equivalent to its argument formula. 
For \ero-formulas, the operator $QE$ runs in time polynomial in the size of its argument formula (if the number 
of attributes in $\A$ is fixed), and exponential in the number of attributes in $\A$.

To compute formulas representing different \layers of $CON$, {\tt getStratum} (Algorithm \ref{alg:min-contr-get-layer})
is used. 
It takes three parameters: the formula $F_{\succ_{CON}}$ representing the relation $\succ_{CON}$, 
the formula $F_{K_{CON}}$ representing the set of the end nodes of $CON$-edges, and 
the \layerindex $i$. It returns a formula which represents the set of the end nodes of $CON$-edges of \layer $i$ , 
or {\tt undefined} if the corresponding set is empty. That formula is computed
according to the definition of a \layer.

\begin{algorithm}[H]
  \caption{{\tt minContr}($F_{\succ}$, $F_{CON}$)}\label{alg:min-contr}
	\label{alg:minContr}
  \begin{algorithmic}[1]
\STATE $i = 0$
\STATE $F_{P^-_{-1}}(x, y) := false$
\STATE $F_{K_{CON}}(y) := QE( \exists x\ .\ F_{CON}(x, y))$
\STATE $F_{\succ_{CON}}(x,y) := F_{CON}(x, y) \wedge F_{K_{CON}}(x) \wedge F_{K_{CON}}(y)$
\STATE $F_{L_i}(y) := \mbox{{\tt getStratum}}(F_{\succ_{CON}}, F_{K_{CON}}, i)$ 
\WHILE{$F_{L_i}$ is defined}
   \STATE 
\begin{tabbing}
$F_{E_i}(x, y)$ $ := QE(\exists v \, .\, $ \= $F_{L_i}(v) \wedge F_{CON}(x, v) \wedge F_{\succ}(x, y) \wedge$\\
		             	  \> $(y = v \vee F_{\succ}(y, v) \wedge \neg (F_{P^-_{i-1}}(y, v) \vee F_{CON}(y,v))))$
              \end{tabbing} 

   \STATE $F_{P^-_i}(x,y) := F_{P^-_{i-1}}(x,y) \vee F_{E_i}(x,y)$
   \STATE $i$ := $i$ + 1;
   \STATE $F_{L_i}(y) := \mbox{{\tt getStratum}}(F_{\succ_{CON}}, F_{K_{CON}}, i)$ 
   \ENDWHILE

   \RETURN $P^-_i$
   \end{algorithmic}
\end{algorithm}

\begin{algorithm}[H]
  \caption{{\tt getStratum}($F_{\succ_{CON}}$, $F_{K_{CON}}$, $i$)}\label{alg:min-contr-get-layer}
  \begin{algorithmic}[1]
	\REQUIRE {$i \geq 0$}
	\IF{i\ =\ 0}
	\STATE \begin{tabbing}
	$F_{L_i}(y) := QE($ \= $F_{K_{CON}}(y) \wedge \neg \exists x_1 ( F_{\succ_{CON}}(y, x_1) ))$
       \end{tabbing}
	\ELSE
	\STATE \begin{tabbing}
	$F_{L_i}(y) := QE($ \= $\ \ \ \exists x_1, \ldots, x_i \ \ \ .\  F_{\succ_{CON}}(y, x_1) \wedge F_{\succ_{CON}}(x_1, x_2) \wedge \ldots $ \\ \>$ \wedge \ F_{\succ_{CON}}(x_{i-1}, x_i)) \wedge$ 
		       $\neg \exists x_1, \ldots, x_{i+1} \ .\  F_{\succ_{CON}}(y, x_1) $ \\ 
			\> $\wedge \ F_{\succ_{CON}}(x_1, x_2) \wedge \ldots \wedge
						F_{\succ_{CON}}(x_{i}, x_{i+1})))$
       \end{tabbing}
	\ENDIF
	\IF{$\exists y\ .\ F_{L_i}(y)$}
		\RETURN $F_{L_i}$
	\ELSE
		\RETURN {\tt undefined}
	\ENDIF
   \end{algorithmic}
\end{algorithm}

\vspace{-3mm}
\begin{proposition}\label{prop:min-contr-correctness}
 Let $CON$ be a \boundedlayer \contractingrel of a preference relation $\succ$. Then  
Algorithm \ref{alg:min-contr} terminates and computes a prefix \contraction of $\pref$ by $CON$. 
\end{proposition}

Proposition \ref{prop:min-contr-correctness} holds because Algorithm \ref{alg:min-contr} uses the 
construction from Theorem \ref{thm:bounded-rel-con}. 
Below we show an example of computing a prefix \contraction{} for a finitely representable preference relation.

\vspace{-3mm}
\begin{example}
 Let a preference relation $\succ$ be defined by the following formula
 $$\form{\succ}(o, o') \equiv o.m = BMW \wedge o'.m = VW \vee o.m = o'.m \wedge o.price < o'.price$$
and a \contracting{} $CON$ be defined by 
\begin{align*}
\form{CON}(o, o') \equiv o.m = o'.m \wedge ( & (11000 \leq o.price \leq 13000 \wedge o'.price = 15000) \vee \\
						   & (10000 \leq o.price \leq 12000 \wedge o'.price = 14000))
\end{align*}
where $m$ is a $\D$-attribute and $price$ is a $\Q$-attribute. Then $\form{K_{CON}}(o) \equiv o.price = 14000 \vee
o.price = 15000$ and $F_{\succ_{CON}}(o, o') \equiv \form{\succ}(o, o') \wedge \form{K_{CON}}(o) \wedge \form{K_{CON}}(o')$.
The end nodes of the $CON$ \layers are defined by the following formulas:
\begin{align*}
 \form{L_0}(o) \equiv & o.price = 15000 \wedge o.m \neq BMW\\
 \form{L_1}(o) \equiv & o.price = 15000 \wedge o.m = BMW \vee o.price = 14000 \wedge o.m \neq BMW\\
 \form{L_2}(o) \equiv & o.price = 14000 \wedge o.m = BMW.
\end{align*}
The relations contracting all $CON$ \layers are defined by the following formulas
\begin{align*}
 \form{E_0}(o, o') \equiv & o.m = o'.m \neq BMW \wedge 11000 \leq o.price \leq 13000 \wedge 13000 < o'.price \leq 15000\\
 \form{E_1}(o, o') \equiv & o.m = o'.m = BMW \wedge 11000 \leq o.price \leq 13000 \wedge 13000 < o'.price \leq 15000 \vee \\
			  & o.m = o'.m \neq BMW \wedge 10000 \leq o.price < 11000 \wedge 13000 < o'.price \leq 14000 \\
 \form{E_2}(o, o') \equiv & o.m = o'.m = BMW \wedge 10000 \leq o.price \leq 11000 \wedge 13000 < o'.price \leq 14000
\end{align*}
Finally, a \contraction{} $P^-$ of $\pref$ by $CON$ is defined by 
\begin{align*}
 \form{P^-}(o, o') \equiv  o.m = o'.m \wedge (& 11000 \leq o.price \leq 13000 \wedge 13000 < o'.price \leq 15000 \vee \\
			  		       & 10000 \leq o.price < 11000 \wedge 13000 < o'.price \leq 14000)
\end{align*}

\end{example}

The \boundedlayerprop of $CON$ is crucial for the termination of the algorithm: the algorithm 
does not terminate for relations not \boundedlayer. Hence, given a \contracting{} relation, it is 
useful to know if it is \boundedlayer or not. Let us consider the formula $F_{\succ_{CON}}$. 
Without loss of generality, we assume it is represented in DNF.
By definition, $CON$ is a \boundedlayer relation if and only if there is a constant $k$ such that
all $\succ_{CON}$ paths are of length at most $k$. In the next theorem, we show that
this property can be checked by a single evaluation of the quantifier elimination operator. 

\begin{theorem}\label{thm:disjuncts}{\bf (Checking \boundedlayerprop).\ }
 Let $\form{R}$ be an \ero-formula, representing an SPO relation $R$, in the following form
 $$F_R(o, o') = F_{R_1}(o, o') \vee \ldots \vee F_{R_l}(o, o'),$$
 where $\form{R_i}$ is a conjunction of atomic formulas. 
 Then checking if there is a constant $k$ such that the length of all $R$-paths is at most $k$ 
 can be done by a single evaluation of $QE$ over a formula of size linear in $|F_R|$. 
\end{theorem}

In Theorem \ref{thm:disjuncts}, we assume that each atomic formula using the operators $\leq, \geq$
is transformed to disjunction of two formulas: one which uses the strict comparison operator and the other
using the equality operator. The proof of Theorem \ref{thm:disjuncts} and the details of the 
corresponding \boundedlayerprop test are provided in Appendix A.

\subsection{Computing prefix \contraction{}: finite relations}

In this section, we consider finite relations $\pref$ and $CON$. We assume that the relations
are stored in separate tables: a preference relation table \tname{R} and a \contractingrel{}
table \tname{C}, each having two columns \tname{X} and \tname{Y}. Every tuple in a table
corresponds to an element of the corresponding binary relation. Hence, \tname{R} has to be an SPO 
and \tname{C} $\subseteq$ \tname{R}.
Here we present an algorithm of computing a prefix \contraction{} of a preference relation $\pref$
by $CON$ represented by such tables. Essentially, the algorithm is an adaptation of 
Theorem \ref{thm:bounded-rel-con}. 

The function {\tt minContrFinite} takes two arguments:  \tname{R} and \tname{C}. 
The function is implemented in terms of relational algebra operators. First, it constructs two tables: 
\tname{EC} storing the end nodes of all \tname{C}-edges, and \tname{RC} storing a restriction of the original preference relation 
\tname{R} to \tname{EC}. These two tables are needed for obtaining the \layers of \tname{C}. After that, the function 
picks all \layers of \tname{C} one by one and contracts the original preference relation by each \layer in turn, as
shown in Theorem \ref{thm:bounded-rel-con}. 

The extraction of the \layers of $CON$ in the order of the \layerindex is performed as follows.
It is clear that the nodes ending $CON$-edges of \layer $0$ do 
not start any edge in \tname{RC}. 
The set \tname{E} computed in line 8 is a difference of the set \tname{EC} of the nodes ending 
\tname{C}-edges and the nodes starting some edges in \tname{RC}. Hence, \tname{E} stores all the nodes ending
\tname{C}-edges of \layer $0$.
To get the end nodes of the next \layer of \tname{C}, we need remove all the edges from \tname{RC} which end in 
members of \tname{E}, 
and remove \tname{E} from \tname{EC}. After the \layer with the highest index is obtained, the relation 
\tname{EC} becomes empty. 

\begin{proposition}\label{prop:minContrFinite-corr}
 Algorithm \ref{alg:minContrFinite} computes a prefix \contraction of $R$ by $C$. Its running time is
 $\mathcal{O}(|C|^2 \cdot |R| \cdot log|R|)$. 
\end{proposition}

Proposition \ref{prop:minContrFinite-corr} holds because Algorithm \ref{alg:minContrFinite} uses the construction from Theorem \ref{thm:bounded-rel-con}. 
The stated running time may be obtained by applying some simple optimizations: 
(i) sorting \tname{EC} after constructing
it (line 3), (ii) sorting on \tname{X}, \tname{Y} the table \tname{R} and the table \tname{RC} right after its construction 
(line 5), (iii) keeping these relations sorted after every change. 
In addition to that, we store the relation \tname{P} containing
the intermediate \contraction{} edges as a copy of \tname{R}, in which the edges which 
belong to the prefix \contraction{} are marked. By doing so, \tname{P} is maintained in the sorted state throughout
the algorithm.

\begin{algorithm}[H]
  \caption{{\tt minContrFinite}(\tname{R}, \tname{C})}\label{alg:minContrFinite}
  \begin{algorithmic}[1]
    \REQUIRE{ \tname{R} is transitive, \tname{C} $\subseteq$ \tname{R}}
	\STATE \tname{P} $\gets$ \tname{C}
	\STATE /* Get the end nodes of all \tname{C}-edges */
	\STATE \tname{EC}  $\gets$ $\pi_{\mbox{{\scriptsize \tname{Y}}}}(\mbox{\tname{C}})$
  	\STATE{/*\ \ \tname{RC} is related to \tname{R} as $\succ_{CON}$ to $\succ$ in Definition \ref{def:bounded-rel} */}
	\STATE \tname{RC}  $\gets$ $\pi_{\mbox{{\tname{R.X}, \tname{R.Y}}}}$ (\tname{EC}$_1$ $\join{{\tname{EC}_1.\tname{Y} = \tname{R}.\tname{X}}}$ \tname{R} $\join{{\tname{EC}_2.\tname{Y} = 
	\tname{R}.\tname{Y}}}$ \tname{EC}$_2$)
	\WHILE{\tname{EC} not empty}
		\STATE /* Get the end nodes of the next \layer \tname{C}-edges */
		\STATE \tname{E} $\gets$ \tname{EC} $-$ $\pi_{X}$(\tname{RC})
		\STATE /* Prepare \tname{EC} and \tname{RC} for the next iteration */
		\STATE \tname{EC} $\gets$ \tname{EC} $-$ \tname{E}
		\STATE \tname{RC} $\gets$ \tname{RC} $-$ \tname{RC} $\join{{\ \tname{RC}.\tname{Y} = \tname{E}.\tname{Y}}}$ \tname{E}
		\STATE /* Add to \tname{P} the \tname{R}-edges contracting the current \layer of \tname{C}*/
		\STATE \tname{P} $\gets$ \tname{P} $\cup$ $\pi_{\mbox{{\scriptsize \tname{R}$_1$.\tname{X}, \tname{R}$_1$.\tname{Y}}}}$ 
		(\tname{R}$_1$ $\join{\mbox{{\scriptsize \tname{R}$_1$.\tname{Y} = \tname{R}$_2$.\tname{X}}}}$ 
			(\tname{R}$_2$ $-$ \tname{P}) $\join{\ \mbox{{\scriptsize \tname{R}$_1$.\tname{X} = \tname{C}.\tname{X}, \tname{R}$_2$.\tname{Y} = \tname{C}.\tname{Y}}}}$ 
		(\tname{C} $\join{\ \mbox{{\scriptsize \tname{C}.\tname{Y} = \tname{E}.\tname{Y}}}}$ \tname{E}))
	\ENDWHILE
    \RETURN \tname{P}
  \end{algorithmic}
\end{algorithm}

\section{Preference-protecting contraction}

Consider the operation of minimal preference contraction described above. 
In order to contract a preference relation, a user has to specify a \contractingrel{} $CON$. 
The main criteria we use to define a contracted preference relation is \emph{minimality 
of preference change}. However,
a minimal \contraction $P^-$ may contain additional preferences which are not in $CON$. So far, we have not paid
attention to the contents of $P^-$, assuming that any minimal \contraction{} is equally good
for a user. However, this may not be the case in real life. Assume that an original preference
relation $\pref$ is combined from two preference relations $\pref \ = \ \pref_{old} \cup \pref_{recent}$, 
where $\pref_{old}$ describes user preferences introduced by the user a long time ago, and
$\pref_{recent}$ describes more recent preferences. Now assume that the user wants to contract
$\pref$ by $CON$, at least two minimal \contractions{} are possible: $P^-_1$ which consists of
$CON$ and some preferences of $\pref_{old}$, and $P^-_2$ consisting of $CON$ and some preferences of $\pref_{recent}$.
Since $\pref_{recent}$ has been introduced recently, discarding members of $\pref_{old}$ 
may be more reasonable then members of $\pref_{recent}$. 
Hence, sometimes there is a need to compute \contractions{} which \emph{protect} 
some existing preferences from removal.

Here we propose an operator of \emph{preference-protecting} contraction. 
In addition to a \contracting{} $CON$, a subset $P^+$ of the original preference relation
to be \emph{protected from removal} in the contracted preference relation may also be specified. 
Such a relation is complementary with respect to the \contracting{}:
the relation $CON$ defines the preferences to discard, whereas
the relation $P^+$ defines the preferences to protect.

\begin{definition}\label{def:min-cons-contr}
  Let $\pref$ be a preference relation and $CON$ be a \contractingrel{} of $\pref$. 
  Let a relation $P^+$ be such that $P^+ \subseteq\, \succ$.  
  A \contraction{} $P^-$ of $\pref$ by $CON$ such that $P^+ \cap P^- = \emptyset$ is called a 
  \emph{$P^+$-protecting \contraction{} of $\pref$ by $CON$}. 
  A minimal \contraction{} $P^-$ of $\pref$ by $CON$ such that $P^+ \cap P^- = \emptyset$ is called a 
  \emph{$P^+$-protecting minimal \contraction{} of $\pref$ by $CON$}. 
\end{definition}

Given any \contraction{} $P^-$ of $\,\succ$ by $CON$, 
by Lemma \ref{lemma:contr-necc}, $P^-$ must contain at least
one edge from every $CON$-detour. Thus, if $P^+$ contains
an entire $CON$-detour, protecting $P^+$ while contracting 
$\,\succ$ by $CON$ is not possible. 

\begin{theorem}\label{thm:min-con-cons-criterion}
  Let $CON$ be a \divtolayers \contracting{} relation of a preference relation $\succ$
  such that 
  $P^+\subset \,\succ$. There exists a minimal \contraction{} 
  of $\,\succ$ by $CON$ that protects $P^+$ if and only if
  \mbox{$\:P^+_{TC} \cap CON = \emptyset$},
  where $P^+_{TC}$ is the transitive closure of $\:P^+$.
\end{theorem}

As we noted, the necessary condition of the theorem above follows from Lemma \ref{lemma:contr-necc}.
The sufficient condition follows from Theorem \ref{thm:pref-prot-con} we prove further.

A naive way of computing a preference-protecting minimal \contraction{} is
by finding a minimal \contraction{} $P^-$ of \mbox{$(\,\succ -\ P^+)$} and then 
adding $P^+$ to $P^-$. However, \mbox{$(\,\succ -\ P^+)$} is not an SPO in general, 
thus obtaining SPO of $\pref - \ (P^- \cup P^+)$ becomes problematic.

The solution we propose here uses the following idea. 
First, we find 
a \contracting{} $CON'$ such that minimal contraction of $\,\succ$ by $CON'$ 
is equivalent to minimal contraction of $\,\succ$ by $CON$ with 
protected $P^+$. After that, we compute a minimal \contraction{} of $\pref$ by $CON'$ 
using Theorem \ref{thm:bounded-rel-con}. 

Recall that minimal \contractions{} constructed in Theorem \ref{thm:bounded-rel-con} 
are \emph{prefix}, i.e., every edge $xy$ in such a \contraction{} starts some 
$CON$-detour in which $xy$ is the only edge of the contractor. Thus, if no member of $P^+$ 
starts a $CON$-detour in $\pref$, then the minimal \contraction{} and $P^+$ have no common edges.
Otherwise assume that an edge $xy \in P^+$ starts 
a $CON$-detour in $\pref$. By Lemma \ref{lemma:contr-necc}, any
$P^+$-protecting \contraction $P^-$ has to contain an edge different from $xy$
which belongs to $CON$-detours started by $xy$. Moreover, for $CON$-detours
of length two started by $xy$, $P^-$ has to contain the \emph{edges ending those $CON$-detours}.
Such a set of edges is defined as follows:

$$Q = \{xy\ |\ \exists u: u \succ x \succ y \wedge uy \in CON \wedge ux \in P^+\}.$$

It turns out that the set $Q$ is not only contained in any $P^+$-protecting \contraction{}, but it 
can also be used to construct a $P^+$-protecting minimal \contraction{} as shown in the 
next theorem.

\begin{theorem}\label{thm:pref-prot-con}
 Let $\pref$ be a preference relation, and $CON$ be a \divtolayers \contracting{} of $\pref$. 
 Let also $P^+$ be a transitive relation such that $P^+ \subseteq \pref$ and $P^+ \cap CON = \emptyset$.
 Then the prefix \contraction{} of $\pref$ by $CON \cup Q$ is a 
 $P^+$-protecting minimal \contraction{} of $\pref$ by $CON$.
\end{theorem}

\begin{proof}
Let $P^-$ be a prefix \contraction{} of $\pref$ by $CON' = CON \cup Q$. We prove that 
$P^- \cap P^+ = \emptyset$, i.e., $P^-$ protects $P^+$. For the sake of contradiction, assume there is $xy \in P^+ \cap P^-$. 
We show that this contradicts the prefix property of $P^-$.
Since $P^-$ is a prefix \contraction{}, there is a $CON'$-detour from $x$ to some $v$ in $\pref$, 
started by $xy$ and having only the edge $xy$ in $P^-$.
We have two choices: either it is a $CON$-detour or a $Q$-detour. Consider the first case. 
Clearly, $y \neq v$, otherwise $P^+ \cap CON \neq \emptyset$. Thus, $xv \in CON$ and 
$x \succ y \succ v$ (Figure \ref{pic:thm-pref-prot-con-1}). 
$yv \in Q$ follows from $xy \in P^+$, $xv \in CON$ and the construction of $Q$. Note that every path from 
$y$ to $v$ in $\pref$ contains a $P^-$-edge because $P^-$ is a \contraction{} of $\pref$ by $CON \cup Q$. 
That implies that no $CON$-detour from $x$ to $v$ started by $xy$ has only $xy$ in $P^-$ which 
contradicts the initial assumption.

Consider the second case, i.e.,
there is a $Q$-detour from $x$ to some $v$ started by $xy$ and having only the edge $xy$ in $P^-$. 
Since $xv \in Q$, there
is $uv \in CON$ such that $ux \in P^+$ (Figure \ref{pic:thm-pref-prot-con-2}). 
$ux, xy \in P^+$ imply $uy \in P^+$ by transitivity of $P^+$. $uy \in P^+$ and $uv \in CON$ imply
$yv \in Q$. That along with the fact that $P^-$ is a \contraction{} of $\pref$ by $CON \cup Q$ implies
 that every path in $\pref$ from $y$ to $v$ contains a $P^-$-edge. Hence, there is no $Q$-detour
from $x$ to $v$ started by $xy$ and having only $xy$ in $P^-$. That contradicts the initial assumption
about $xy$.

Now we prove that $P^-$ is a minimal \contraction{} of $\pref$ by $CON$. The fact that it is a \contraction{}
of $\pref$ by $CON$ follows from the fact that it is a \contraction{} of $\pref$ by a superset $CON'$ of $CON$.
We prove now its minimality. Since $P^-$ is a prefix \contraction{} of $\pref$ by $CON'$, for every $xy \in P^-$, 
there is $xv \in CON'$ such that there is a corresponding detour $T$ in which $xy$ is the only $P^-$-edge. 
If it is a $CON$-detour, then $xy$ satisfies
the minimality condition from Theorem \ref{thm:min-con-criterion}. If it is a $Q$-detour, 
then there is a $CON$-edge $uv$ such that $ux \in P^+$. We showed above that $P^-$ protects $P^+$. Hence,
the $CON$-detour obtained by joining the edge $ux$ and $T$ has only $xy$ in $P^-$. Therefore, 
$P^-$ is a minimal \contraction{} of $\pref$ by $CON$. \qed

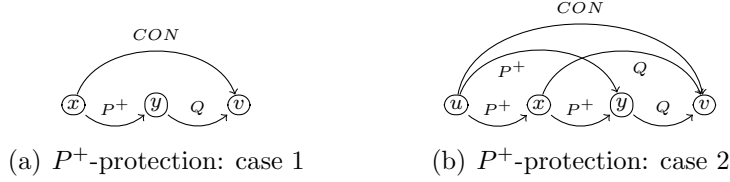
\begin{figure}
\begin{center}
	\subfigure[$P^+$-protection: case 1]{
		\label{pic:thm-pref-prot-con-1}
		\begin{tikzpicture}[scale=1.1]
			 \tikzstyle{cir} = [draw=black,rounded corners,inner sep=2pt];
			 \tikzstyle{lupperedge} = [bend left=80, ->];
			 \tikzstyle{mupperedge} = [bend left=70, ->];
			 \tikzstyle{supperedge} = [bend right=40, ->];
			 \tikzstyle{loweredge} = [bend right=60, ->];

\pgfkeyssetvalue{/edge/con}{{\tiny $CON$}}
\pgfkeyssetvalue{/edge/pp}{{\tiny $P^+$}}
\pgfkeyssetvalue{/edge/q}{{\tiny $Q$}}

			\node[cir] (x) at (1, 0)  {{\scriptsize $x$}};
			\node[cir] (y) at (2, 0)  {{\scriptsize $y$}};
			\node[cir] (v) at (3, 0)  {{\scriptsize $v$}};

			\node at (0,0) {};
			\node at (4,0) {};

			\draw[supperedge] (x) to node [above] {\pgfkeysvalueof{/edge/pp}} (y);
			\draw[supperedge] (y) to node [above] {\pgfkeysvalueof{/edge/q}} (v);

			\draw[mupperedge] (x) to node [above] {\pgfkeysvalueof{/edge/con}} (v);
			
			\node at ( 0.5, 0) {};
			\node at ( 3.5, 0) {};
		\end{tikzpicture}
	}
	\vspace{1cm}
	\subfigure[$P^+$-protection: case 2]{
		\label{pic:thm-pref-prot-con-2}
		\begin{tikzpicture}[scale=1.1]
			 \tikzstyle{cir} = [draw=black,rounded corners,inner sep=2pt];
			 \tikzstyle{lupperedge} = [bend left=80, ->];
			 \tikzstyle{mupperedge} = [bend left=70, ->];
			 \tikzstyle{supperedge} = [bend right=40, ->];
			 \tikzstyle{loweredge} = [bend right=60, ->];

\pgfkeyssetvalue{/edge/con}{{\tiny $CON$}}
\pgfkeyssetvalue{/edge/pp}{{\tiny $P^+$}}
\pgfkeyssetvalue{/edge/q}{{\tiny $Q$}}

			\node[cir] (u) at (0, 0)  {{\scriptsize $u$}};
			\node[cir] (x) at (1, 0)  {{\scriptsize $x$}};
			\node[cir] (y) at (2, 0)  {{\scriptsize $y$}};
			\node[cir] (v) at (3, 0)  {{\scriptsize $v$}};

			\node at (-1,0) {};
			\node at (4,0) {};

			\draw[supperedge] (u) to node [above] {\pgfkeysvalueof{/edge/pp}}  (x);
			\draw[supperedge] (x) to node [above] {\pgfkeysvalueof{/edge/pp}} (y);
			\draw[supperedge] (y) to node [above] {\pgfkeysvalueof{/edge/q}} (v);

			\draw[mupperedge] (u) to node [below left] {\pgfkeysvalueof{/edge/pp}} (y);
			\draw[mupperedge] (x) to node [below right] {\pgfkeysvalueof{/edge/q}} (v);

			\draw[lupperedge] (u) to node [above] {\pgfkeysvalueof{/edge/con}} (v);
		\end{tikzpicture}
	}
\end{center}
	\vspace{-8mm}
	\caption{Proof of Theorem \ref{thm:pref-prot-con}}
	\label{pic:thm-pref-prot-con}
\end{figure}
\end{proof}

Note that the sets of the end nodes of $(CON \cup Q)$-edges and the end nodes of $CON$-edges coincide by the construction 
of $Q$. Therefore, $(CON \cup Q)$ is \divtolayers or \boundedlayer if and only if
$CON$ is \divtolayers or \boundedlayer, correspondingly. Hence, if $CON$ is a \boundedlayer relation with
respect to $\pref$, Algorithms \ref{alg:minContr} and \ref{alg:minContrFinite} can be used 
to compute a preference-protecting minimal \contraction{} of $\pref$ by $CON$.
If the relations $\pref$ and $CON$ are finite, then $Q$ can be constructed
in polynomial time in the size of $\pref$ and $CON$ by a relational algebra expression constructed from its definition. If the relations
are finitely representable, then $Q$ may be computed using the quantifier elimination operator $QE$. 

For Theorem \ref{thm:pref-prot-con} to apply, the relation $P^+$ has to be transitive. 
Non-transitivity of $P^+$ implies that there are two edges $xy, yz \in P^+$ which should
be protected while transitive edge $xz$ is not critical. However, a 
relation obtained as a result of preference-protecting contraction is a preference relation (i.e., SPO).
Hence, the edge $xz$ will also be protected in the resulting preference relation. 
This fact implies that protecting any relation is equivalent to protecting its 
minimal transitive extension: its transitive closure. Therefore, if $P^+$ is not 
transitive, one needs to compute its transitive closure to use Theorem \ref{thm:pref-prot-con}.
For finite relations, transitive closure can be computed in polynomial time \cite{cormen}. 
For finitely representable relations, \emph{Constraint Datalog} \cite{kanellakis95} can be used to 
compute transitive closure. 

  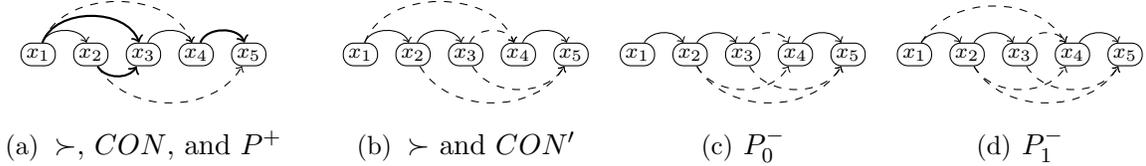
\begin{figure}[ht]
	      \begin{center}
		\subfigure[\small{$\pref$, $CON$, and $P^+$}]{
			\begin{tikzpicture}[scale=0.7]
			 \tikzstyle{cir} = [draw=black,rounded corners,inner sep=2pt];
			 \tikzstyle{upperedge} = [bend left=70, ->];
			 \tikzstyle{loweredge} = [bend right=60, ->];
			 \node[cir] (x1) at (1, 0) {{\scriptsize $x_1$}};
			 \node[cir] (x2) at (2, 0) {{\scriptsize $x_2$}};
			 \node[cir] (x3) at (3, 0) {{\scriptsize $x_3$}};
			 \node[cir] (x4) at (4, 0) {{\scriptsize $x_4$}};
			 \node[cir] (x5) at (5, 0) {{\scriptsize $x_5$}};

			\node at (0,0) {};
			\node at (6,0) {};
			  
			 \draw[upperedge] (x1) to (x2);
			 \draw[upperedge] (x3) to (x4);
			 \draw[upperedge, thick] (x4) to (x5);

			 \draw[upperedge, dashed] (x1) to (x4);
			 \draw[loweredge, dashed] (x2) to (x5);

			 \draw[upperedge, thick] (x1) to (x3);
			 \draw[loweredge, thick] (x2) to (x3);
			\end{tikzpicture}
		   \label{pic:min-pref-prot-con-1}
		}
		\subfigure[\small{$\pref$ and $CON'$}]{
			\begin{tikzpicture}[scale=0.7]
			 \tikzstyle{cir} = [draw=black,rounded corners,inner sep=2pt];
			 \tikzstyle{upperedge} = [bend left=70, ->];
			 \tikzstyle{loweredge} = [bend right=60, ->];
			 \node[cir] (x1) at (1, 0) {{\scriptsize $x_1$}};
			 \node[cir] (x2) at (2, 0) {{\scriptsize $x_2$}};
			 \node[cir] (x3) at (3, 0) {{\scriptsize $x_3$}};
			 \node[cir] (x4) at (4, 0) {{\scriptsize $x_4$}};
			 \node[cir] (x5) at (5, 0) {{\scriptsize $x_5$}};
			  
			 \draw[upperedge] (x1) to (x2);
			 \draw[upperedge] (x2) to (x3);
			 \draw[upperedge, dashed] (x3) to (x4);
			 \draw[upperedge] (x4) to (x5);

			 \draw[upperedge, dashed] (x1) to (x4);
			 \draw[loweredge, dashed] (x2) to (x5);

			 \draw[loweredge, dashed] (x3) to (x5);	
			\end{tikzpicture}
		   \label{pic:min-pref-prot-con-2}
		}
		\subfigure[\small{$P^-_0$}]{
			\begin{tikzpicture}[scale=0.7]
			 \tikzstyle{cir} = [draw=black,rounded corners,inner sep=2pt];
			 \tikzstyle{upperedge} = [bend left=60, ->];
			 \tikzstyle{loweredge} = [bend right=60, ->];
			 \node[cir] (x1) at (1, 0) {{\scriptsize $x_1$}};
			 \node[cir] (x2) at (2, 0) {{\scriptsize $x_2$}};
			 \node[cir] (x3) at (3, 0) {{\scriptsize $x_3$}};
			 \node[cir] (x4) at (4, 0) {{\scriptsize $x_4$}};
			 \node[cir] (x5) at (5, 0) {{\scriptsize $x_5$}};
			  
			 \draw[upperedge] (x1) to (x2);
			 \draw[upperedge, dashed] (x3) to (x4);
			 \draw[upperedge] (x4) to (x5);

			 \draw[upperedge] (x2) to (x3);
			 \draw[loweredge, dashed] (x2) to (x4);
			 \draw[loweredge, dashed] (x2) to (x5);
			 \draw[loweredge, dashed] (x3) to (x5);
			\end{tikzpicture}
		   \label{pic:min-pref-prot-con-3}
		}	        
		\subfigure[\small{$P^-_1$}]{
			\begin{tikzpicture}[scale=0.7]
			 \tikzstyle{cir} = [draw=black,rounded corners,inner sep=2pt];
			 \tikzstyle{upperedge} = [bend left=60, ->];
			 \tikzstyle{loweredge} = [bend right=60, ->];
			 \node[cir] (x1) at (1, 0) {{\scriptsize $x_1$}};
			 \node[cir] (x2) at (2, 0) {{\scriptsize $x_2$}};
			 \node[cir] (x3) at (3, 0) {{\scriptsize $x_3$}};
			 \node[cir] (x4) at (4, 0) {{\scriptsize $x_4$}};
			 \node[cir] (x5) at (5, 0) {{\scriptsize $x_5$}};
			  
			 \draw[upperedge] (x1) to (x2);
			 \draw[upperedge, dashed] (x3) to (x4);
			 \draw[upperedge] (x4) to (x5);

			 \draw[upperedge] (x2) to (x3);
			 \draw[loweredge, dashed] (x2) to (x4);
			 \draw[loweredge, dashed] (x2) to (x5);
			 \draw[loweredge, dashed] (x3) to (x5);
			 \draw[upperedge, dashed] (x1) to (x4);
			\end{tikzpicture}
		   \label{pic:min-pref-prot-con-4}
		}	        
	      \end{center}
                   \vspace{-0.2in}
     \label{pic:min-pref-prot-con}
     \caption{Using Theorem \ref{thm:pref-prot-con} to compute a preference-protecting minimal \contraction}
  \end{figure}

Another important observation here is that the $P^+$-protecting minimal \contraction of $\pref$ by $CON$
computed according to Theorem \ref{thm:pref-prot-con} is not necessary a \emph{prefix} \contraction
of $\pref$ by $CON$. This fact is illustrated in the following example.

\begin{example}\label{ex:min-pref-prot-con}
  Let a preference relation $\pref$ be a total order of $\{x_1, \ldots, x_5\}$ 
(Figure \ref{pic:min-pref-prot-con-1}, the transitive 
edges are omitted for clarity). Let a \contracting{} $CON$ be $\{x_1x_4, x_2x_5\}$, and 
$P^+ = \{x_1x_3, x_2x_3, x_4x_5\}$. 

The existence of a minimal $P^+$-protecting \contraction of $\pref$ by $CON$
follows from Theorem \ref{thm:min-con-cons-criterion}. 
We use Theorem \ref{thm:pref-prot-con} to construct it. 
The set $Q$ is equal to $\{x_3x_4, x_3x_5\}$ and $CON' = \{x_1x_4, x_2x_5, x_3x_4, x_3x_5\}$. 
We construct a prefix \contraction of $\pref$ by $CON'$.
The relation $CON'$ has two \layers: $L_0 = \{x_2x_5, x_3x_5\}$, $L_1 = \{x_1x_4, x_3x_4\}$. 
Then $E_0 = \{x_2x_5, x_3x_5, x_2x_4, x_3x_4\}$, $P^-_0 = E_0$, $E_1 = \{x_1x_4, x_3x_4\}$, 
$P^-_1 = E_0 \cup E_1$, and $P^- = P^-_1$. By Theorem  \ref{thm:pref-prot-con}, $P^-$ is a 
$P^+$-protecting minimal \contraction of $\pref$ by $CON$. However, $P^-$ is not a
prefix \contraction of $\pref$ by $CON$, because the edges $x_3x_4$, $x_3x_5$ do not start any $CON$-detour.
\end{example}

\section{Meet preference contraction}

In this section, we consider the operation of \emph{meet preference contraction}. In contrast to the preceding sections,
where the main focus was the minimality of preference relation change, the contraction operation considered here
changes a preference relation not necessarily in a minimal way. A \meetcontraction of a preference relation 
is semantically a \emph{union} of all minimal sets of reasons of discarding a given set preferences. 
When a certain set of preferences is required to be protected while contracting a preference relation, 
the operation of \emph{preference-protecting meet contraction} may be used. 

\begin{definition}\label{def:meet-con}
 Let $\pref$ be a preference relation, $CON$ a \contractingrel{} of $\pref$, and $P^+ \subseteq \pref$.
The relation $P^m$ is a {\emph{\meetcontraction} of $\pref$ by $CON$} iff
$$P^m = \bigcup_{P^- \in \mathcal{P}^m}P^-,$$
for the set $\mathcal{P}^m$ of all minimal \contractions{} of $\pref$ of $CON$.
The relation $P^m_{P^+}$ is a \emph{\meetprotcontraction{$P^+$} of $\pref$ by $CON$} iff
$$P^m_{P^+} = \bigcup_{P^- \in \mathcal{P}^m_{P^+}}P^-,$$
for the set $\mathcal{P}^m_{P^+}$ of all $P^+$-protecting minimal \contractions{} of $\pref$ of $CON$.
\end{definition}

Note that the relations $(\pref - \,P^m)$ and $(\pref - \,P^m_{P^+})$ can be represented as intersections 
of preference (i.e., SPO) relations and thus are also preference (i.e., SPO) relations. Let us first consider 
the problem of constructing \meetcontractions. 

By the definition above, an edge $xy$ is in the \meetcontraction of a preference 
relation $\pref$ by $CON$ if there is a minimal \contraction{} of $\pref$ by $CON$ which contains $xy$. 
Theorem \ref{thm:min-con-criterion} implies that if there is no $CON$-detour in $\pref$ containing $xy$, 
then $xy$ is not in the corresponding \meetcontraction. However, the fact that $xy$ belongs to 
a $CON$-detour is not a sufficient condition for $xy$ to be in the corresponding \meetcontraction. 

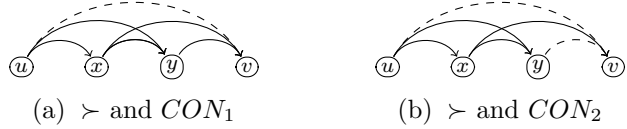
\begin{figure}
 \centering
	\subfigure[$\pref$ and $CON_1$]{
	\label{pic:ex-meet-con-intro-2}
 \begin{tikzpicture}
			 \tikzstyle{cir} = [draw=black,rounded corners,inner sep=2pt];
			 \tikzstyle{upperedge} = [bend left=60, ->];
			 \tikzstyle{loweredge} = [bend right=60, ->];

			\node[cir] (u) at (0, 0)  {{\scriptsize $u$}};
			\node[cir] (x) at (1, 0)  {{\scriptsize $x$}};
			\node[cir] (y) at (2, 0)  {{\scriptsize $y$}};
			\node[cir] (v) at (3, 0)  {{\scriptsize $v$}};

			\draw[upperedge] (u) to (x);
			\draw[upperedge] (x) to (y);
			\draw[upperedge] (x) to (y);
			\draw[upperedge] (y) to (v);

			\draw[upperedge] (u) to (y);
			\draw[upperedge] (x) to (v);
			\draw[upperedge, dashed] (u) to (v);  
 \end{tikzpicture}
}
\hspace{1cm}
	\subfigure[$\pref$ and $CON_2$]{
	\label{pic:ex-meet-con-intro-3}
 \begin{tikzpicture}
			 \tikzstyle{cir} = [draw=black,rounded corners,inner sep=2pt];
			 \tikzstyle{upperedge} = [bend left=60, ->];
			 \tikzstyle{loweredge} = [bend right=60, ->];

			\node[cir] (u) at (0, 0)  {{\scriptsize $u$}};
			\node[cir] (x) at (1, 0)  {{\scriptsize $x$}};
			\node[cir] (y) at (2, 0)  {{\scriptsize $y$}};
			\node[cir] (v) at (3, 0)  {{\scriptsize $v$}};

			\draw[upperedge] (u) to (x);
			\draw[upperedge] (x) to (y);
			\draw[upperedge, dashed] (y) to (v);

			\draw[upperedge] (u) to (y);
			\draw[upperedge] (x) to (v);
			\draw[upperedge, dashed] (u) to (v);  
 \end{tikzpicture}
}

	\caption{Example \ref{ex:meet-con-intro}}
	\label{pic:ex-meet-con-intro}
\end{figure}

\begin{example}\label{ex:meet-con-intro}
 Let a preference relation $\pref$ be a total order of $\{u, x, y, v\}$. Let 
 also $CON_1 = \{uv\}$ (Figure \ref{pic:ex-meet-con-intro-2}) and 
 $CON_2 = \{uv, yv\}$ (Figure \ref{pic:ex-meet-con-intro-3}). There is only one $CON_1$- and $CON_2$-detour containing $xy$: 
 $u \succ x \succ y \succ v$. 
 There is also a minimal \contraction{} of $\pref$ by $CON_1$ which contains $xy$: $P^-_1 = \{uy, xv, xy, uv\}$. 
 However, there is no minimal \contraction{} of $\pref$ by $CON_2$ which contains $xy$ because the edge $yv$ of 
  the $CON_2$-detour $u \succ x \succ y \succ v$ is in $CON_2$.
\end{example}

In Theorem \ref{thm:meet-contr}, we show how \meetcontractions can be constructed in the case of \emph{\boundedlayer 
\contractingrels}.  
According to that theorem, a $\pref$-edge $xy$ is in the \meetcontraction of $\pref$ by $CON$ if and only if
there is a \contraction $P^-$ of $\pref$ by $CON$ such that $xy$ is the only $P^-$-edge in some $CON$-detour. 
We use Theorem \ref{thm:min-con-cons-criterion} to show that there is a minimal \contraction of $\pref$ by 
$CON$ which contains $xy$ while the other edges of the detour are protected.

\begin{theorem}\label{thm:meet-contr}
 Let $CON$ be a \boundedlayer \contractingrel{} of a preference relation $\pref$. Then
 the \meetcontraction of $\pref$ by $CON$ is
\begin{align*}
 P^m = \{xy\ |\ \exists uv \in CON\ .\  & u \succeq x \succ y \succeq v 
	\wedge \\ 
				      & (ux \in (\pref - \,CON) \vee u = x) \wedge (yv \in (\pref - \,CON) \vee y = v) \}
\end{align*}
\end{theorem}

\begin{proof}
  By Corollary \ref{col:simple-cond}, an edge $xy$ is in a minimal \contraction{} $P^-$ of $\pref$ by $CON$, 
  if there is $CON$-detour of at most three edges in $\pref$ in which $xy$ is the only $P^-$-edge. Hence any 
  minimal \contraction{} is a subset of $P^m$. Now take every edge $xy$ of $P^m$ and show there is a
  minimal \contraction{} of $\pref$ by $CON$ which contains $xy$. Let $u \succeq x \succ y \succeq v$ 
  for $uv \in CON$. Let us construct a set $P'$ as follows:

$$ P' = \left\lbrace 
\begin{array}{ll}
 \{ux, yv\}  & \mbox{ if } u \succ x \wedge y \succ v\\
 \{ux\}      & \mbox{ if } u \succ x \wedge y = v\\
 \{yv\}      & \mbox{ if } u = x \wedge y \succ v\\
 \emptyset   & \mbox{ if } u = x \wedge y = v\\
\end{array}
   \right.
$$
$P'$ is transitive, $P' \cap CON = \emptyset$, and $P' \subseteq \pref$. 
Theorem \ref{thm:min-con-cons-criterion} implies that there is a $P'$-protecting minimal \contraction{} $P^-$ of 
$\pref$ by $CON$. Since $P^-$ protects $P'$, there is a $CON$-detour in $\pref$ from $u$ to $v$ in which 
$xy$ is the only $P^-$-edge. This implies that $xy \in P^-$. \qed
\end{proof}

Now consider the case of $P^+$-protecting \meetcontractions. A naive solution is 
to construct it as the difference of $P^m$ defined above and $P^+$. However, in the next example we show that
such solution does not work in general.

\begin{figure}
 \centering
 \begin{tikzpicture}
			 \tikzstyle{cir} = [draw=black,rounded corners,inner sep=2pt];
			 \tikzstyle{upperedge} = [bend left=60, ->];
			 \tikzstyle{loweredge} = [bend right=60, ->];

			\node[cir] (u) at (0, 0)  {{\scriptsize $u$}};
			\node[cir] (x) at (1, 0)  {{\scriptsize $x$}};
			\node[cir] (y) at (2, 0)  {{\scriptsize $y$}};
			\node[cir] (v) at (3, 0)  {{\scriptsize $v$}};

			\draw[upperedge, thick] (u) to (x);
			\draw[upperedge] (x) to (y);
			\draw[upperedge] (x) to (y);
			\draw[upperedge] (y) to (v);

			\draw[upperedge, dashed] (u) to (y);
			\draw[upperedge, dashed] (x) to (v);
			\draw[upperedge] (u) to (v);  
 \end{tikzpicture}
	\caption{$\pref$, $CON$, and $P^+$ from Example \ref{ex:prot-meet-con}}
	\label{pic:ex-meet-con-prot}
\end{figure}
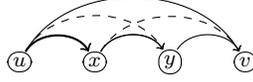

\begin{example}\label{ex:prot-meet-con}
 Let a preference relation $\pref$ be a total order of $\{u, x, y, v\}$ (Figure \ref{pic:ex-meet-con-prot}). Let 
 also $CON = \{uy, xv\}$ and $P^+ = \{ux\}$. Note that $yv \not \in P^+$, and by Theorem \ref{thm:meet-contr}, 
$yv \in P^m$. Hence, $yv \in (P^m - P^+)$. However, note that $ux \in P^+$ implies that $xy$ must be 
a member of every $P^+$-protecting \contraction{} in order to disconnect the path from $u$ to $y$. Hence, there is 
no $CON$-detour in which $yv$ is the only edge of the \contraction, and $yv$ is not a member of any $P^+$-protecting 
\contraction{}.
\end{example}

The next theorem shows how a $P^+$-protecting \contraction{} may be constructed. The idea is similar 
to Theorem \ref{thm:meet-contr}. However, to construct a \meetcontraction, we used the set $CON$
as a common part of all minimal \contractions{}. In the case of $P^+$-protecting \meetcontraction, 
a superset $C_{P^+}$ of $CON$ is contained in all of them. Such a set $C_{P^+}$ may be viewed
as a union of $CON$ and the set of all edges of $\succ$ that \emph{must be discarded due to the protection 
of $P^+$}.

\begin{theorem}\label{thm:meet-contr-prot}
 Let $CON$ be a \boundedlayer \contractingrel{} of a preference relation $\pref$, and $P^+$ a transitive relation 
 such that $P^+ \subseteq \pref$ and $P^+ \cap CON = \emptyset$. Then
 the $P^+$-protecting \meetcontraction of $\pref$ by $CON$ is
\begin{align*}
 P^m_{P^+} = \{xy\ |\ &xy \not \in P^+ \wedge \exists uv \in CON \ .\   u \succeq x \succ y \succeq v 
	\ \wedge \\ 
				      & (ux \in (\pref - \,C_{P^+}) \vee u = x) \wedge (yv \in 
					(\pref - \,C_{P^+}) \vee y = v) \},
\end{align*}
for
\begin{align*}
 C_{P^+} = \{xy\ |\ \exists uv \in CON\ .\  u \succeq x \succ y \succeq v 
	\wedge 
				      (ux \in P^+ \vee u = x) \wedge (yv \in P^+ \vee y = v) \}
\end{align*}
\end{theorem}

\begin{proof}
 First, it is easy to observe that $C_{P+}$ is a subset of any $P^+$-protecting \contraction{} of 
$\pref$ by $CON$. It is constructed from the edges $xy$ which participate in $CON$-detours of length 
at most three where all the other edges have to be protected. Since every $CON$-detour has to have
at least one edge in a \contraction{}, $xy$ has to be a member of every \contraction{}.

We show that every $P^+$-protecting minimal \contraction{} $P^-$ of $\pref$ by $CON$ is a
subset of $P^m_{P^+}$. If some $xy \in P^-$, then by Corollary \ref{col:simple-cond}
there is an edge $uv \in CON$ such that 
$u\succeq x \succ y \succeq v$
and 
$ux, yv \not \in P^-$. We show that $xy \in P^m_{P^+}$. That holds if $xy \not \in P^+$ (which holds for $P^-$ by 
definition) and
$$(ux \in (\pref - \,C_{P^+}) \vee u = x)
\wedge (yv \in (\pref - \,C_{P^+}) \vee y = v)$$
If both $u = x$ and $y = v$ hold then the expression above holds. Now assume $u \succ x$ (the case $y\succ v$
is similar). If $ux \in C_{P^+}$ then, as we showed above, $ux \in P^-$ which is a contradiction. Hence, 
$ux \in (\pref - \ C_{P^+})$ and $xy \in P^m_{P^+}$. Finally, $P^- \subseteq P^m_{P^+}$. 

Now we show that every $xy \in P^m_{P^+}$ is contained in every $P^+$-protecting minimal 
\contraction{} of $\pref$ by $CON$. The proof is similar to the proof of Theorem \ref{thm:meet-contr}. 
By definition of $P^m_{P^+}$, take $xy$ such that $u \succeq x \succ y \succeq v$. 
Construct the set $P'$ for $xy$ as in the proof of
Theorem \ref{thm:meet-contr}. We show that for the set $P'' = TC(P^+ \cup P')$ we have $P'' \cap CON = \emptyset$.
For the sake of contradiction, assume $P'' \cap CON \neq \emptyset$. 
This implies that there is a $CON$-detour consisting of $P^+$ and $P'$ edges. Having only
$P^+$-edges in the detour contradicts the initial assumption that $P^+ \cap CON = \emptyset$. Having a single edge
of $P'$ in the detour implies that the edge (either $ux$ or $yv$) is in $C_{P^+}$, which contradicts
the definition of $P^m_{P^+}$. Having both $ux$ and $yv$ in the detour implies that $xy \in P^+$
which also contradicts the definition of $P^m_{P^+}$. Hence, $P'' \cap CON = \emptyset$, 
and by Theorem \ref{thm:pref-prot-con}, there is a $P''$-protecting minimal \contraction $P^-$
of $\pref$ by $CON$ which is also a $P^+$-protecting minimal
\contraction. Since there is a $CON$-detour in which $xy$ is unprotected by $P^-$, $xy \in P^-$.\qed
\end{proof}

We note that given the expressions for the meet and $P^+$-protecting \meetcontractions in Theorems 
\ref{thm:meet-contr} and \ref{thm:meet-contr-prot}, one can easily
obtain such contractors for finite and finitely representable relations: by evaluation of 
a relational algebra query in the former case and by quantifier elimination in the latter case.

  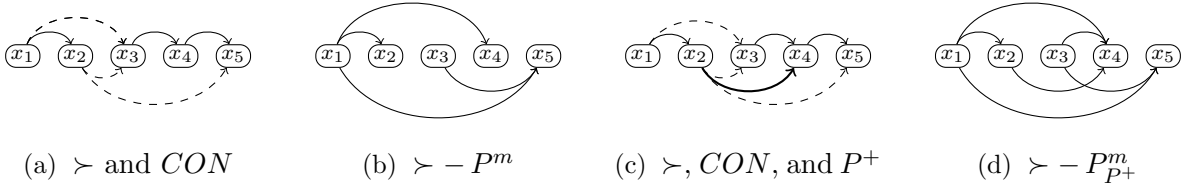
\begin{figure}[ht]
	      \begin{center}
		\subfigure[\small{$\pref$ and $CON$}]{
			\begin{tikzpicture}[scale=0.7]
			 \tikzstyle{cir} = [draw=black,rounded corners,inner sep=2pt];
			 \tikzstyle{upperedge} = [bend left=70, ->];
			 \tikzstyle{loweredge} = [bend right=60, ->];
			 \node[cir] (x1) at (1, 0) {{\scriptsize $x_1$}};
			 \node[cir] (x2) at (2, 0) {{\scriptsize $x_2$}};
			 \node[cir] (x3) at (3, 0) {{\scriptsize $x_3$}};
			 \node[cir] (x4) at (4, 0) {{\scriptsize $x_4$}};
			 \node[cir] (x5) at (5, 0) {{\scriptsize $x_5$}};
			  
			 \draw[upperedge] (x1) to (x2);
			 \draw[upperedge] (x3) to (x4);
			 \draw[upperedge] (x4) to (x5);

			 \draw[upperedge, dashed] (x1) to (x3);
			 \draw[loweredge, dashed] (x2) to (x5);

			 \draw[upperedge, dashed] (x1) to (x3);
			 \draw[loweredge, dashed] (x2) to (x3);
			 \node at (1, -1.3) {};
			\end{tikzpicture}
		   \label{pic:meet-con-1}
		}
		\hspace{3mm}
		\subfigure[\small{$\pref - \,P^m$}]{
			\begin{tikzpicture}[scale=0.7]
			 \tikzstyle{cir} = [draw=black,rounded corners,inner sep=2pt];
			 \tikzstyle{upperedge} = [bend left=70, ->];
			 \tikzstyle{loweredge} = [bend right=60, ->];
			 \node[cir] (x1) at (1, 0) {{\scriptsize $x_1$}};
			 \node[cir] (x2) at (2, 0) {{\scriptsize $x_2$}};
			 \node[cir] (x3) at (3, 0) {{\scriptsize $x_3$}};
			 \node[cir] (x4) at (4, 0) {{\scriptsize $x_4$}};
			 \node[cir] (x5) at (5, 0) {{\scriptsize $x_5$}};
			  
			 \draw[upperedge] (x1) to (x2);
			 \draw[loweredge] (x3) to (x5);
			 \draw[loweredge] (x1) to (x5);
			 \draw[upperedge] (x1) to (x4);
			 \node at (1, -1) {};
			\end{tikzpicture}
		   \label{pic:meet-con-2}
		}
		\hspace{3mm}
		\subfigure[\small{$\pref$, $CON$, and $P^+$}]{
			\begin{tikzpicture}[scale=0.7]
			 \tikzstyle{cir} = [draw=black,rounded corners,inner sep=2pt];
			 \tikzstyle{upperedge} = [bend left=60, ->];
			 \tikzstyle{loweredge} = [bend right=60, ->];
			 \node[cir] (x1) at (1, 0) {{\scriptsize $x_1$}};
			 \node[cir] (x2) at (2, 0) {{\scriptsize $x_2$}};
			 \node[cir] (x3) at (3, 0) {{\scriptsize $x_3$}};
			 \node[cir] (x4) at (4, 0) {{\scriptsize $x_4$}};
			 \node[cir] (x5) at (5, 0) {{\scriptsize $x_5$}};
			  
			 \draw[upperedge] (x1) to (x2);
			 \draw[loweredge, thick] (x2) to (x4);
			 \draw[upperedge] (x3) to (x4);
			 \draw[upperedge] (x4) to (x5);

			 \draw[upperedge, dashed] (x1) to (x3);
			 \draw[loweredge, dashed] (x2) to (x3);
			 \draw[loweredge, dashed] (x2) to (x5);

			 \node at (1, -1.3) {};
			\end{tikzpicture}
		   \label{pic:meet-con-3}
		}	        
		\hspace{3mm}
		\subfigure[\small{$\pref - \,P^m_{P^+}$}]{
			\begin{tikzpicture}[scale=0.7]
			 \tikzstyle{cir} = [draw=black,rounded corners,inner sep=2pt];
			 \tikzstyle{upperedge} = [bend left=70, ->];
			 \tikzstyle{loweredge} = [bend right=60, ->];
			 \node[cir] (x1) at (1, 0) {{\scriptsize $x_1$}};
			 \node[cir] (x2) at (2, 0) {{\scriptsize $x_2$}};
			 \node[cir] (x3) at (3, 0) {{\scriptsize $x_3$}};
			 \node[cir] (x4) at (4, 0) {{\scriptsize $x_4$}};
			 \node[cir] (x5) at (5, 0) {{\scriptsize $x_5$}};
			  
			 \draw[upperedge] (x1) to (x2);
			 \draw[loweredge] (x3) to (x5);
			 \draw[loweredge] (x1) to (x5);
			 \draw[loweredge] (x2) to (x4);
			 \draw[upperedge] (x3) to (x4);
			 \draw[upperedge] (x1) to (x4);
			 \node at (1, -1) {};
			\end{tikzpicture}
		   \label{pic:meet-con-4}
		}	        
	      \end{center}
                   \vspace{-0.2in}
     \caption{Computing \meetcontraction and $P^+$-protecting \meetcontraction}
     \label{pic:meet-con}
  \end{figure}

\begin{example}\label{ex:meet-con}
Let a preference relation $\pref$ be a total order of $\{x_1, \ldots, x_5\}$ 
(Figure \ref{pic:meet-con-1}, the transitive 
edges are omitted for clarity). Let a \contracting{} $CON$ be $\{x_1x_3, x_2x_3, x_2x_5\}$, and 
$P^+ = \{x_2x_4\}$.

A \meetcontraction $P^m$ of $\pref$ by $CON$ is $\{x_1x_3, x_2x_3, x_2x_5, x_2x_4, x_3x_4, x_4x_5\}$.
The resulting contracted preference relation is shown on Figure \ref{pic:meet-con-2}. 
A 
$P^+$-protecting \meetcontraction of $\pref$ by $CON$ is 
$\{ x_1x_3, x_2x_3, x_2x_5, x_4x_5\}$. The resulting contracted
preference relation is shown on Figure \ref{pic:meet-con-4}. Note that 
$C_{P^+}$ here is $CON \cup \{x_4x_5\}$.

\end{example}

\section{Querying with contracted preferences}

When dealing with preferences, the two most common tasks are 1) given two tuples, find
the more preferred one, and 2) find the most preferred tuples in a set. 
In this section, we assume that preference and \contractingrels{} are represented as
\emph{preference formulas}. 
Out of the two problems above, the first can be solved easily by the evaluation of
the corresponding preference formula for the given pair of tuples. 
To solve the latter problem, the operators of 
\emph{winnow} \cite{Chomicki2003} and \emph{BMO} \cite{kiesling2002} are proposed. 
The winnow picks from a given
set of tuples the undominated tuples according to a given 
preference relation. A special case of the winnow operator is called \emph{skyline}
\cite{Borz01theskyline}. It operates with preference relations representing 
\emph{Pareto improvement}. 
A number of
evaluation optimization methods for queries
involving winnow have been proposed \cite{Chomicki2007opt,Hafenrichter2005,DBLP:conf/icde/ChomickiGGL03,
DBLP:journals/vldb/GodfreySG07,DBLP:journals/tods/PapadiasTFS05}. 

\begin{definition}\label{def:winnow}
  Let $\U$ be a universe of tuples each having a set
  of attributes $\A$. Let $\,\succ$ be a preference relation over $\U$.
Then the \emph{winnow operator} is written as 
  $w_{\succ}(\U)$, and for every finite subset $r$ of $\ \U$:
$$
w_\succ(r) = \{t \in r\: |\: \neg \exists t' \in r . t' \succ t\}
$$
\end{definition}

In this section, we show some new techniques which can be used to optimize
the evaluation of the winnow operator under contracted preferences.

In user-guided preference modification frameworks 
\cite{Chomicki2007,Balke2006}, it is assumed that users alter their preferences after
examining sets of the most preferred tuples returned by winnow. 
Thus, if preference contraction is 
incorporated into such frameworks, there is a need to 
compute winnow under contracted preference relations. Here we
show how the evaluation of winnow can be optimized in such cases.

Let $\pref$ be a preference relation, $CON$ be a \contracting{}
of $\pref$, $P^-$ be a \contraction{} of $\,\succ$ by $CON$,
and the contracted preference relation $\pref' = (\,\succ -\,P^-)$.
Denote the set of the starting and the ending tuples of $R$-edges
for a binary relation $R$ as $S({R})$ and $E({R})$ correspondingly.
\begin{align*}
S({R}) = \{x\: |\: \exists y\ .\  xy \in R\}\\
E({R}) = \{y\: |\: \exists x\ .\  xy \in R\}
\end{align*}
Let us also define the set $M({CON})$ of the tuples which participate in $CON$-detours in $\,\succ$
\begin{align*}
M({CON}) = \{y\: |\: \exists x,z\ .\ x \succ y \wedge xz \in CON \wedge y \succeq z  \}
\end{align*}
Assume we also know quantifier-free formulas $\form{S(P^-)}$, $\form{E(P^-)}$,  $\form{M(CON)}$,
and $\form{S(CON)}$
representing these sets for $P^-$ and $CON$. Then the following holds.

\begin{proposition}\label{prop:winnow-opt}
Given a finite set of tuples $r$
  \begin{enumerate}
    \item $w_{\succ}(r) \subseteq w_{\succ'}(r)$
    \item If $\sigma_{\form{S(P^-)}}(w_{\succ}(r)) = \emptyset$, then $w_{\succ}(r) = w_{\succ'}(r)$.
    \item If $P^-$ is a prefix \contraction{}, then 
        \mbox{$\sigma_{\form{S(P^-)}}(r) = \sigma_{\form{S(CON)}}(r) $}

    \item
        $w_{\succ'}(r) = $  $w_{\succ'}(w_{\succ}(r) \cup \sigma_{\form{E(P^-)}}(r))$
  \end{enumerate}
\end{proposition}

\begin{proof}$ $

  \begin{enumerate}
    \item By definition, $w_{\succ}(r)$ contains the set of undominated tuples
      w.r.t the preference relation $\succ$. Thus, $\succ' \subset \succ$ implies
      that if a tuple $o$ was undominated w.r.t $\succ$, it will be undominated
      w.r.t $\succ'$, too. Hence, $w_{\succ}(r) \subseteq w_{\succ'}(r)$.
    \item the SPO of $\pref$ implies that for every tuple $o$ not in $w_{\succ}(r)$, there is a tuple
	$o' \in w_{\succ}(r)$ such that $o' \succ o$. Hence, if no edges going from $w_{\succ}(r)$ are contracted
	by $P^-$, every $o \not \in w_{\succ}(r)$ will still be dominated according to $\succ'$.

     \item Follows from the definition of the \emph{prefix} contraction. 

	\item From 1 we know that $w_{\succ}(r) \subseteq w_{\succ'}(r)$. 
	For every tuple $o \in w_{\succ'}(r) - w_{\succ}(r)$, at least one edge going
	to it in $\pref$ has been contracted by $P^-$. Thus, 
	$w_{\succ'}(r) \subseteq w_{\succ}(r) \cup \sigma_{F_{E(P^-)}}(r)$ and
	$w_{\succ'}(r) = w_{\succ'}(w_{\succ}(r) \cup \sigma_{F_{E(P^-)}}(r))$. \qed
  \end{enumerate}
\end{proof}

According to Proposition \ref{prop:winnow-opt},
the result of winnow under a contracted preference is always
a superset of the result of winnow under the original preference.
The second property shows when the contraction does not change
the result of winnow. Running the winnow query is generally expensive,
thus one can first evaluate $\sigma_{\form{S(P^-)}}$ 
(or $\sigma_{\form{S(CON)}}$, if $P^-$ is a prefix contraction) over the
computed result of the original winnow.
If the result is empty, then computing the winnow 
under the contracted preference relation is not needed.

The last statement of the proposition
is useful when the set $r$ is large and 
thus running $w_{\succ'}$ 
over the entire set $r$ is expensive. 
Instead, one can compute $\sigma_{\form{E(P^-)}}(r)$ 
and then evaluate $w_{\succ'}$ over 
$(w_{\succ}(r) \cup \sigma_{\form{E(P^-)}}(r))$
(assuming that $w_{\succ}(r)$ is already known). 

\begin{example}
 Let a preference relation $\pref$ be defined by $\form{\succ}(o, o') \equiv o.p < o'.p$, and 
 a \contracting{} $CON$ of $\pref$ be defined by $\form{CON}(o, o') \equiv o.p = 0 \wedge o'.p = 3$,
 where $p$ is an $\Q$-attribute. Take set of tuples $r = \{1, 2, 3, 4\}$ in which every tuple has 
 a single attribute $p$. Then $w_{\succ}(r) = \{1\}$.
 Take two minimal \contractions{} $P^-_1$ and $P^-_2$ of $\pref$ by $CON$ defined by the following formulas
\begin{align*}
 \form{P^-_1}(o, o') \equiv o.p = 0 \wedge 0 < o'.p \leq 3 \\
 \form{P^-_2}(o, o') \equiv 0 \leq o.p < 3 \wedge o'.p = 3
\end{align*}
The corresponding contracted preference relations $\pref_1$ and $\pref_2$ are defined by 
 $\form{\succ_1}(o, o') \equiv \form{\succ}(o, o') $ $\wedge \neg \form{P^-_1}(o, o')$ and
 $\form{\succ_2}(o, o') \equiv \form{\succ}(o, o') \wedge \neg \form{P^-_2}(o, o')$.
The \contraction $P^-_1$ is prefix, thus $\form{S(P^-_1)}(o) \equiv \form{S(CON)} \equiv o.p = 0$.
The \contraction $P^-_2$ is not prefix, and $\form{S(P^-_2)}(o) \equiv 0 \leq o.p < 3$. 

First, $\sigma_{\form{S(P^-_1)}}(w_\succ(r)) = \emptyset$ implies $w_{\succ_1}(r) = w_{\succ}(r)$. Second, 
$\sigma_{\form{S(P^-_2)}}(w_\succ(r))$ is not empty and equal to $\{1\}$. Note that
$\form{E(P^-_2)}(o) \equiv o.p = 3$. Hence, $\sigma_{\form{E(P^-_2)}}(r) = \{3\}$ and 
$w_{\succ_2}(r) = w_{\succ_2}(w_{\succ}(r) \cup \sigma_{\form{E(P^-_2)}}(r)) = \{1, 3\}$.

\end{example}

\section{Experimental evaluation}

\pgfcreateplotcyclelist{\myblackwhitelist}{%
    {black,mark=*},
    {black,mark=square},
    {black,mark=o},
    {black,mark=otimes},
    {black,mark=triangle}
}

In this section, we present the results of experimental evaluation 
of the preference contraction framework proposed here. We implemented the following operators of preference contraction: 
prefix contraction (denoted as PREFIX), 
preference-protecting minimal contraction ($P^+$-MIN), 
meet contraction (MEET), and preference-protecting meet contraction
($P^+$-MEET).
PREFIX was implemented using Algorithm \ref{alg:minContrFinite}, 
$P^+$-MIN according to Theorem \ref{thm:pref-prot-con}, MEET 
according to Theorem \ref{thm:meet-contr}, and $P^+$-MEET according to Theorem \ref{thm:meet-contr-prot}.
We used these operators to contract finite preference relations stored
in a database table $R(X,Y)$. The preference relations used in the experiments were finite
\emph{skyline preference relations} \cite{Borz01theskyline}. Such relations are often used in 
database applications. We note that such relations are generally not materialized (as database tables)
when querying databases with skylines. However, they may be materialized in scenarios of preference
elicitation \cite{DBLP:conf/dasfaa/BalkeGL07}. 
To generate such 
relations, we used the NHL 2008 Player Stats dataset \cite{nhl} of 852 tuples. 
Each tuple has 18 different attributes out of which we used 5. All algorithms used in the experiments were implemented in 
Java 6. We ran the experiments on Intel Core 2 Duo CPU 2.1 GHz with 2.0 GB RAM. All tables were
stored in a PostgreSQL 8.3 database.

\medskip

In the first experiment, we modeled the scenario in which a user manually selects preferences to contract. 
Here we used preference relations consisting of $2000$, $3000$, and $5000$ edges. 
The sizes of \contractingrels
used here range from 1 to 35 edges. We do not pick more than 35 edges assuming that in this scenario a user unlikely 
provides a large set of preferences to discard. For every \contractingrel size, we randomly generated 10 different 
\contractingrels and computed the average time spent to compute \contractions and 
the average size of them. 
The relations $P^+$ storing preferences to protect contained $25\%$ of edges of the corresponding
preference relation. 

Figure \ref{pic:contr-exp-time-vs-consize} shows how the running times of contraction operators
depend on the size a preference relation to contract and the size of a \contractingrel. 
As we can observe, PREFIX has the best performance among all operators, regardless of 
the size of the preference relation and the \contractingrel relation.
Note also that the running times of preference-protecting operators are significantly larger then 
the running times of their unconstrained counterparts. These running times 
predominantly depend on the time spent to compute the transitive closure of $P^+$. 

\begin{figure}[ht]
	\subfigure[$|\succ|=2000$]{
 \begin{tikzpicture}[scale=0.85]
	\pgfplotsset{every axis legend/.append style={
	    	at={(0.5,1.03)},
    		anchor=south,
		},
		every axis/.append style={
		font=\scriptsize
		}
	}

     	\begin{semilogyaxis}[
        	height=5cm,
          	width=5cm,
          	cycle list name=\myblackwhitelist,
		ylabel={running time (ms)},
		xlabel={$|CON|$},
		legend columns=2,
		minor x tick num=1
      	]

        \addplot coordinates {
		(1, 5.9052000000000000)
		(3, 7.0700000000000000)
		(5, 9.6006000000000000)
		(10, 11.3609000000000000)
		(15, 14.1016000000000000)
		(20, 16.9723000000000000)
		(25, 19.7204000000000000)
		(30, 24.2398000000000000)
		(35, 25.9463000000000000)
	};
        \addlegendentry{{\small PREFIX}};

        \addplot coordinates {
		(1, 234.9372000000000000)
		(3, 237.9211000000000000)
		(5, 230.0415000000000000)
		(10, 232.3039000000000000)
		(15, 236.5687000000000000)
		(20, 241.7534000000000000)
		(25, 239.7723000000000000)
		(30, 255.0569000000000000)
		(35, 257.2339000000000000)
	};
        \addlegendentry{{\small $P^+$-MIN}};

        \addplot coordinates {
(		1, 0.86180000000000000000)
(		3, 2.4081000000000000)
(		5, 4.1391000000000000)
(		10, 8.4309000000000000)
(		15, 18.9592000000000000)
(		20, 16.4025000000000000)
(		25, 22.2481000000000000)
(		30, 24.6686000000000000)
(		35, 28.1393000000000000)

	};
        \addlegendentry{{\small MEET}};

        \addplot coordinates {
(1, 255.1663000000000000)
(3, 281.0575000000000000)
(5, 342.4532000000000000)
(10, 411.7167000000000000)
(15, 478.2876000000000000)
(20, 559.5078000000000000)
(25, 604.4270000000000000)
(30, 772.5439000000000000)
(35, 722.4933000000000000)
	};
        \addlegendentry{{\small $P^+$-MEET}};

	\end{semilogyaxis}
 \end{tikzpicture}
	\label{pic:contr-exp-time-vs-consize-2000}
}
\subfigure[$|\succ| = 3000$]{
 \begin{tikzpicture}[scale=0.85]
	\pgfplotsset{every axis legend/.append style={
	    	at={(0.5,1.03)},
    		anchor=south,
		},
		every axis/.append style={
		font=\scriptsize
		}
	}

     	\begin{semilogyaxis}[
        	height=5cm,
          	width=5cm,
          	cycle list name=\myblackwhitelist,
		ylabel={running time (ms)},
		xlabel={$|CON|$},
		legend columns=2,
		minor x tick num=1
      	]

        \addplot coordinates {
		(1, 9.6533000000000000)
		(3, 13.9911000000000000)
		(5, 12.9200000000000000)
		(10, 16.9135000000000000)
		(15, 21.4553000000000000)
		(20, 26.4622000000000000)
		(25, 30.4900000000000000)
		(30, 37.6715000000000000)
		(35, 44.4290000000000000)
	};
        \addlegendentry{{\small PREFIX}};

        \addplot coordinates {
		(1, 481.5897000000000000)
		(3, 476.1029000000000000)
		(5, 478.7727000000000000)
		(10, 476.7377000000000000)
		(15, 485.7221000000000000)
		(20, 497.2794000000000000)
		(25, 503.0842000000000000)
		(30, 502.5633000000000000)
		(35, 513.1363000000000000)
	};
        \addlegendentry{{\small $P^+$-MIN}};

        \addplot coordinates {
		(1, 1.3492000000000000)
		(3, 3.8672000000000000)
		(5, 7.1176000000000000)
		(10, 13.2227000000000000)
		(15, 18.6338000000000000)
		(20, 24.3004000000000000)
		(25, 32.0524000000000000)
		(30, 37.7382000000000000)
		(35, 45.4036000000000000)
	};
        \addlegendentry{{\small MEET}};

        \addplot coordinates {
(1, 496.6230000000000000)
(3, 539.5989000000000000)
(5, 785.0971000000000000)
(10, 859.7701000000000000)
(15, 969.9022000000000000)
(20, 1066.7659000000000000)
(25, 1241.8696000000000000)
(30, 1367.0609000000000000)
(35, 1440.7846000000000000)
	};
        \addlegendentry{{\small $P^+$-MEET}};

	\end{semilogyaxis}
 \end{tikzpicture}
	\label{pic:contr-exp-time-vs-consize-3000}
}
\subfigure[$|\succ| = 5000$]{
 \begin{tikzpicture}[scale=0.85]
	\pgfplotsset{every axis legend/.append style={
	    	at={(0.5,1.03)},
    		anchor=south,
		},
		every axis/.append style={
		font=\scriptsize
		}
	}

     	\begin{semilogyaxis}[
        	height=5cm,
          	width=5cm,
          	cycle list name=\myblackwhitelist,
		ylabel={running time (ms)},
		xlabel={$|CON|$},
		legend columns=2,
		minor x tick num=1
      	]

        \addplot coordinates {
		(1, 16.9557000000000000)
		(3, 21.2175000000000000)
		(5, 23.4796000000000000)
		(10, 33.8474000000000000)
		(15, 36.5972000000000000)
		(20, 49.3023000000000000)
		(25, 49.6625000000000000)
		(30, 59.9019000000000000)
		(35, 74.6545000000000000)
	};
        \addlegendentry{{\small PREFIX}};

        \addplot coordinates {
(1, 1073.6141000000000000)
(3, 1057.5116000000000000)
(5, 1079.6400000000000000)
(10, 1097.4738000000000000)
(15, 1085.9495000000000000)
(20, 1096.3924000000000000)
(25, 1124.9653000000000000)
(30, 1084.4411000000000000)
(35, 1104.9442000000000000)
	};
        \addlegendentry{{\small $P^+$-MIN}};

        \addplot coordinates {
		(1, 2.1168000000000000)
		(3, 6.3472000000000000)
		(5, 10.7771000000000000)
		(10, 20.2219000000000000)
		(15, 32.0397000000000000)
		(20, 43.9596000000000000)
		(25, 53.3231000000000000)
		(30, 61.7763000000000000)
		(35, 74.3983000000000000)
	};
        \addlegendentry{{\small MEET}};

        \addplot coordinates {
		(1, 204.7027000000000000)
		(3, 403.8810000000000000)
		(5, 715.5919000000000000)
		(10, 921.9467000000000000)
		(15, 1634.6978000000000000)
		(20, 1977.9117000000000000)
		(25, 2171.9408000000000000)
		(30, 2482.0794000000000000)
		(35, 2444.0134000000000000)
	};
        \addlegendentry{{\small $P^+$-MEET}};

	\end{semilogyaxis}
 \end{tikzpicture}
	\label{pic:contr-exp-time-vs-consize-5000}
}
\caption{Contraction performance. Small \contractingrels} 
\label{pic:contr-exp-time-vs-consize}
\end{figure}
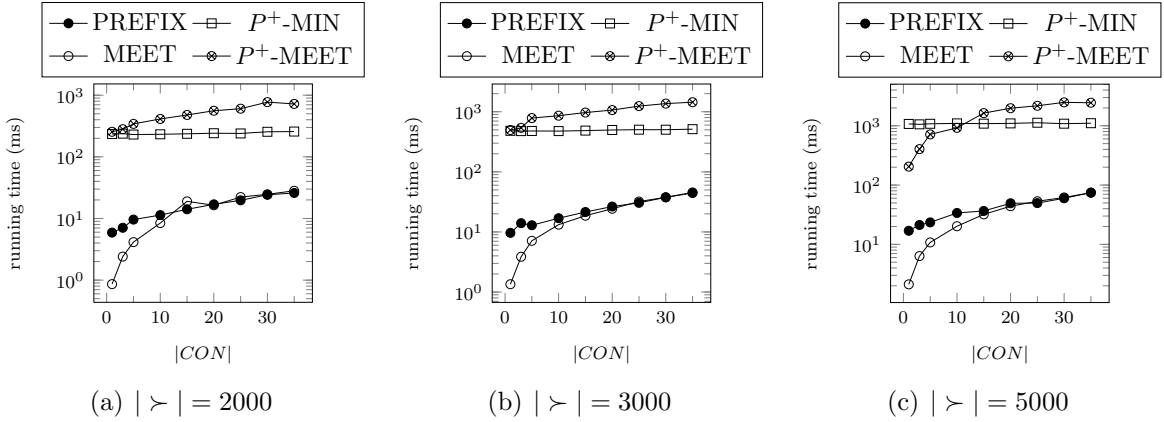

Figure \ref{pic:contr-exp-size-vs-consize} shows the dependence of the \contraction size on 
the size of preference relation and the size of \contracting. For every value of the \contractingrel size, the charts show the average size of the corresponding \contraction. As we can see, the sizes of minimal \contractions (PREFIX and $P^+$-MIN) are the least among all \contractions. This supports the intuition that a minimal set of reasons for preferences
not to hold is smaller than the set of all such reasons. Another important observation here is that
due to the comparatively large size of $P^+$, 
the size of a \meetprotcontraction{$P^+$} is generally half the size of the corresponding \meetcontraction. 

\begin{figure}[ht]
	\subfigure[$|\succ|=2000$]{
 \begin{tikzpicture}[scale=0.85]
	\pgfplotsset{every axis legend/.append style={
	    	at={(0.5,1.03)},
    		anchor=south,
		},
		every axis/.append style={
		font=\scriptsize
		}
	}

     	\begin{axis}[
        	height=5cm,
          	width=5cm,
          	cycle list name=\myblackwhitelist,
		ylabel={$|P^-|$},
		xlabel={$|CON|$},
		minor x tick num=1, 
		legend columns = 2
      	]

        \addplot coordinates {
		(1, 22.5000000000000000)
		(3, 36.1000000000000000)
		(5, 75.9000000000000000)
		(10, 127.4000000000000000)
		(15, 176.1000000000000000)
		(20, 238.5000000000000000)
		(25, 285.0000000000000000)
		(30, 379.2000000000000000)
		(35, 394.6000000000000000)
	};
        \addlegendentry{{\small PREFIX}};

        \addplot coordinates {
		(1, 24.2000000000000000)
		(3, 38.1000000000000000)
		(5, 73.9000000000000000)
		(10, 132.4000000000000000)
		(15, 172.1000000000000000)
		(20, 244.5000000000000000)
		(25, 292.0000000000000000)
		(30, 378.2000000000000000)
		(35, 398.6000000000000000)
	};
        \addlegendentry{{\small $P^+$-MIN}};

        \addplot coordinates {
(1, 263.1000000000000000)
(3, 275.3000000000000000)
(5, 607.5000000000000000)
(10, 793.5000000000000000)
(15, 893.0000000000000000)
(20, 1100.6000000000000000)
(25, 1236.4000000000000000)
(30, 1458.7000000000000000)
(35, 1448.5000000000000000)
	};
        \addlegendentry{{\small MEET}};

        \addplot coordinates {
		(1, 61.8000000000000000)
		(3, 110.0000000000000000)
		(5, 214.1000000000000000)
		(10, 360.9000000000000000)
		(15, 448.9000000000000000)
		(20, 594.4000000000000000)
		(25, 655.1000000000000000)
		(30, 897.4000000000000000)
		(35, 825.8000000000000000)
	};
        \addlegendentry{{\small $P^+$-MEET}};

	\end{axis}
 \end{tikzpicture}
	\label{pic:contr-exp-size-vs-consize-2000}
}
\subfigure[$|\succ| = 3000$]{
 \begin{tikzpicture}[scale=0.85]
	\pgfplotsset{every axis legend/.append style={
	    	at={(0.5,1.03)},
    		anchor=south,
		},
		every axis/.append style={
		font=\scriptsize
		}
	}

     	\begin{axis}[
        	height=5cm,
          	width=5cm,
          	cycle list name=\myblackwhitelist,
		ylabel={$|P^-|$},
		xlabel={$|CON|$},
		minor x tick num=1,
		minor y tick num=1,
		legend columns = 2
      	]

        \addplot coordinates {
(1, 15.8000000000000000)
(3, 39.1000000000000000)
(5, 91.9000000000000000)
(10, 168.9000000000000000)
(15, 211.3000000000000000)
(20, 273.7000000000000000)
(25, 344.7000000000000000)
(30, 451.7000000000000000)
(35, 522.0000000000000000)
	};
        \addlegendentry{{\small PREFIX}};

        \addplot coordinates {
(1, 18.8000000000000000)
(3, 49.1000000000000000)
(5, 100.9000000000000000)
(10, 175.9000000000000000)
(15, 228.3000000000000000)
(20, 289.7000000000000000)
(25, 367.7000000000000000)
(30, 465.7000000000000000)
(35, 524.0000000000000000)
	};
        \addlegendentry{{\small $P^+$-MIN}};

        \addplot coordinates {
(1, 120.8000000000000000)
(3, 273.5000000000000000)
(5, 913.8000000000000000)
(10, 1127.6000000000000000)
(15, 1314.7000000000000000)
(20, 1415.7000000000000000)
(25, 1639.9000000000000000)
(30, 1958.6000000000000000)
(35, 2114.5000000000000000)
	};
        \addlegendentry{{\small MEET}};

        \addplot coordinates {
(1, 43.0000000000000000)
(3, 124.5000000000000000)
(5, 361.7000000000000000)
(10, 490.3000000000000000)
(15, 580.8000000000000000)
(20, 694.7000000000000000)
(25, 870.7000000000000000)
(30, 1136.4000000000000000)
(35, 1263.2000000000000000)
	};
        \addlegendentry{{\small $P^+$-MEET}};

	\end{axis}
 \end{tikzpicture}
	\label{pic:contr-exp-size-vs-consize-3000}
}
\subfigure[$|\succ| = 5000$]{
 \begin{tikzpicture}[scale=0.85]
	\pgfplotsset{every axis legend/.append style={
	    	at={(0.5,1.03)},
    		anchor=south,
		},
		every axis/.append style={
		font=\scriptsize
		}
	}

     	\begin{axis}[
        	height=5cm,
          	width=5cm,
          	cycle list name=\myblackwhitelist,
		ylabel={$|P^-|$},
		xlabel={$|CON|$},
		minor x tick num=1,
		minor y tick num=1,
		legend columns = 2
      	]

        \addplot coordinates {
(1, 24.1000000000000000)
(3, 57.4000000000000000)
(5, 116.7000000000000000)
(10, 183.9000000000000000)
(15, 289.2000000000000000)
(20, 408.5000000000000000)
(25, 466.0000000000000000)
(30, 532.1000000000000000)
(35, 613.6000000000000000)
	};
        \addlegendentry{{\small PREFIX}};

        \addplot coordinates {
(1, 32.1000000000000000)
(3, 68.4000000000000000)
(5, 128.7000000000000000)
(10, 198.9000000000000000)
(15, 310.2000000000000000)
(20, 428.5000000000000000)
(25, 482.0000000000000000)
(30, 546.1000000000000000)
(35, 629.6000000000000000)
	};
        \addlegendentry{{\small $P^+$-MIN}};

        \addplot coordinates {
(1, 354.5000000000000000)
(3, 674.1000000000000000)
(5, 1263.4000000000000000)
(10, 1523.1000000000000000)
(15, 2259.5000000000000000)
(20, 2609.6000000000000000)
(25, 2553.2000000000000000)
(30, 2849.1000000000000000)
(35, 2901.0000000000000000)
	};
        \addlegendentry{{\small MEET}};

        \addplot coordinates {
(1, 137.4000000000000000)
(3, 283.0000000000000000)
(5, 493.1000000000000000)
(10, 654.2000000000000000)
(15, 1046.7000000000000000)
(20, 1266.4000000000000000)
(25, 1404.0000000000000000)
(30, 1590.6000000000000000)
(35, 1536.7000000000000000)
	};
        \addlegendentry{{\small $P^+$-MEET}};

	\end{axis}
 \end{tikzpicture}
	\label{pic:contr-exp-size-vs-consize-5000}
}
\caption{Full contractor size. Small \contractingrels} 
\label{pic:contr-exp-size-vs-consize}
\end{figure}
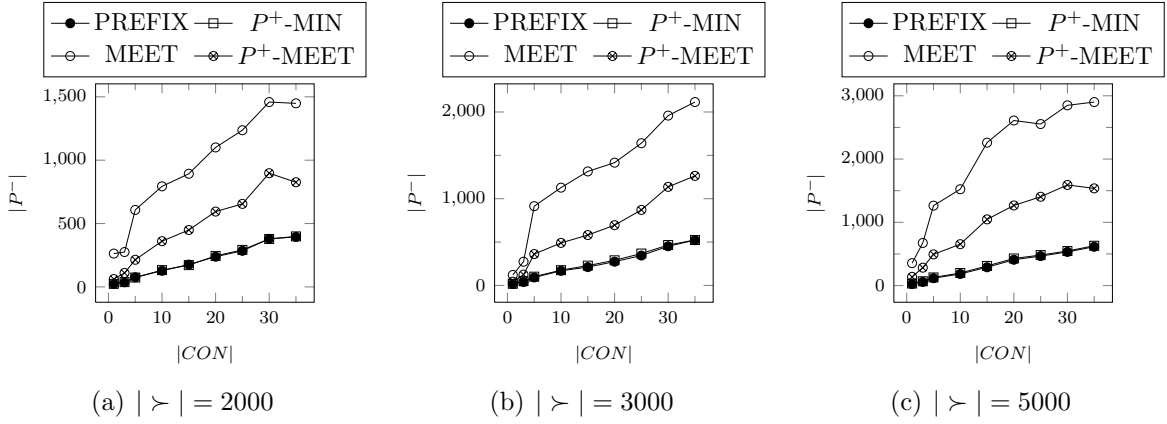

\medskip

In the next experiment, we assume that \contractingrels are elicited automatically based on 
indirect user feedback. Hence, they may be of large size. We construct such relations 
here from \emph{similar edges}. Two edges $xy$ and $x'y'$ are considered similar if 
the tuples $x$, $x'$ and $y$, $'y$ are similar. We use the cosine similarity measure to compute 
similarity of tuples. Here we fixed the size of the preference relation to $5000$. The sizes of 
\contractingrels range from $10\%$ to $50\%$ of preference relation size. The size of every $P^+$ 
is $25\%$ of the corresponding preference relation size. Similarly to the previous
experiment, we computed the performance of the contraction operators and the sizes of generated \contractions.
The results are shown in Figure \ref{pic:contr-exp-large}.

\begin{figure}
	\centering
	\subfigure[Contraction performance]{
 	\begin{tikzpicture}
	\pgfplotsset{every axis legend/.append style={
	    	at={(0.5,1.03)},
    		anchor=south,
		},
		every axis/.append style={
		font=\scriptsize
		}
	}

     	\begin{axis}[
        	height=5cm,
          	width=5cm,
          	cycle list name=\myblackwhitelist,
		ylabel={running time (ms)},
		xlabel={$\frac{|CON|}{|\succ|}$},
		legend columns=2,
		minor x tick num=1,
		minor y tick num=1
      	]

 	 	\addplot coordinates{
(0.1, 643.3242000000000000)
(0.2, 1261.9972000000000000)
(0.3, 2035.1748000000000000)
(0.4, 2581.2116000000000000)
(0.5, 3382.7923000000000000)
		};
		\addlegendentry{{\small PREFIX}};

 	 	\addplot coordinates{
(0.1, 1607.4499000000000000)
(0.2, 2214.9417000000000000)
(0.3, 3038.6724000000000000)
(0.4, 3631.4255000000000000)
(0.5, 4306.2362000000000000)
		};
		\addlegendentry{{\small $P^+$-MIN}};

 	 	\addplot coordinates{
(0.1, 839.2135000000000000)
(0.2, 1635.7405000000000000)
(0.3, 2769.4070000000000000)
(0.4, 3657.0518000000000000)
(0.5, 4416.0570000000000000)
		};
		\addlegendentry{{\small MEET}};

 	 	\addplot coordinates{
(0.1, 2946.2443000000000000)
(0.2, 3081.1540000000000000)
(0.3, 3984.9963000000000000)
(0.4, 4782.8790000000000000)
(0.5, 5004.0301000000000000)
		};
		\addlegendentry{{\small $P^+$-MEET}};
		\end{axis}
 	\end{tikzpicture}
	\label{pic:contr-exp-time-vs-consize-large}
	}
	\subfigure[Full contractor size]{
 	\begin{tikzpicture}
	\pgfplotsset{every axis legend/.append style={
	    	at={(0.5,1.03)},
    		anchor=south,
		},
		every axis/.append style={
		font=\scriptsize
		}
	}

     	\begin{axis}[
        	height=5cm,
          	width=5cm,
          	cycle list name=\myblackwhitelist,
		ylabel={$\frac{|P^-|}{|\succ|}$},
		xlabel={$\frac{|CON|}{|\succ|}$},
		ymin=0,
		legend columns = 2,
		minor x tick num=1,
		minor y tick num=1
      	]

 	 	\addplot coordinates{
(0.1, 0.21)
(0.2, 0.28)
(0.3, 0.43)
(0.4, 0.49)
(0.5, 0.57)
		};
		\addlegendentry{{\small PREFIX}};
	 	\addplot coordinates{
(0.1, 0.22)
(0.2, 0.30)
(0.3, 0.44)
(0.4, 0.49)
(0.5, 0.59)
		};
		\addlegendentry{{\small $P^+$-MIN}};
 	 	\addplot coordinates{
(0.1, 0.36)
(0.2, 0.45)
(0.3, 0.6)
(0.4, 0.63)
(0.5, 0.67)
		};
		\addlegendentry{{\small MEET}};

 	 	\addplot coordinates{
(0.1, 0.30)
(0.2, 0.44)
(0.3, 0.55)
(0.4, 0.61)
(0.5, 0.67)
		};
		\addlegendentry{{\small $P^+$-MEET}};
		\end{axis}
 	\end{tikzpicture}
	\label{pic:contr-exp-size-vs-consize-large}
}
\caption{Large \contractingrels}
\label{pic:contr-exp-large}
\end{figure}

First, we note that here the difference between running times of the contraction algorithms 
is not as large as in the previous experiment. Next, consider the value of the function $aux(CON, P^-) = \frac{|P^-|}{|CON|}-1$
in this and the previous experiment. $aux(CON, P^-)$ is equal to the average number of edges
contracted to contract one edge of $CON$. 
Due to the similarity of edges in 
$P^-$, the value of $aux(CON, P^-)$ is significantly smaller in this experiment 
than in the previous one. For instance, $aux(CON, P^-)$ according to 
Figure \ref{pic:contr-exp-size-vs-consize-large} ranges from to $1.1$ to $0.1$ for PREFIX. 
For the same algorithm in Figure \ref{pic:contr-exp-size-vs-consize-5000}, $aux(CON, P^-)$ ranges
from 16 to 23. 

\medskip

Note that in all experiments, the time spent to compute any \contraction did not go beyond 5 seconds. 
If the \contractingrel is small and preference protection is not used, then 
these times are even less than 100ms. Hence we conclude that the algorithms we proposed to contract
finite relations are efficient and may be used in real-life database applications. 

\section{Related work}

\subsection{Relationships with other operators of preference relation change}

A number of operators of preference relation change have been proposed so far. 
An operator of preference revision is defined in \cite{Chomicki2007}. 
A preference relation there is \emph{revised} by another preference relation called
a \emph{revising relation}. The result of revision is still another preference relation. 
\cite{Chomicki2007} defines three semantics of preference revision -- 
union, prioritized, and Pareto -- which are different in the way an original and 
a revising preference relations are composed. For all these
semantics, \cite{Chomicki2007} identifies cases (called \emph{$0$-, $1$-, and $2$-conflicts}) when the revision fails, i.e.,
when there is no SPO preference relation satisfying the operator semantics. 
This work consideres revising preference relations only by preference relations. 
Although it does not address the problem of discarding subsets of preference relations
explicitly, revising a preference relation using Pareto and prioritized revision operators may 
result in discarding a subset of the original preference relation. It has been shown here
that the revised relation is an SPO for limited classes of the composed relations.

Another operator of preference relation change is defined in \cite{Balke2006}. 
This work deals with a special class of preference relations called \emph{skyline}
\cite{Borz01theskyline}. Preference relations in \cite{Balke2006} are changed by \emph{equivalence relations}. In particular, 
a modified preference relation is an extension of the original relation in which 
specified tuples are \emph{equivalent}. This change operator is defined for only
those tuples which are \emph{incomparable} or already \emph{equivalent} according
to the original preference relation. This preference change operator only adds new edges
to the original preference relation, and thus, preference relation contraction cannot
be expressed using this operator.

In \cite{DBLP:conf/aaai/MindolinC08}, we introduced the operation of minimal preference 
contraction for preference relations. We studied properties of this operation and proposed 
algorithms for computing \contractions and preference-protecting \contractions for \boundedlayer \contractingrels. 
In the current paper, we generalize this approach and 
we develop a method of checking the \boundedlayerprop for finitely representable \contractingrels. 
We 
introduce the operations of meet and meet preference-protecting contraction, and
propose methods for computing them. We also provide experimental evaluation of the 
framework and a comprehensive discussion of related work.

\subsection{Relationships with the belief revision theory}

Preferences can be considered as a special form of human \emph{beliefs}, and thus
their change may be modeled in the context of the belief change theory. The approach here is to
represent beliefs as truth-functional logical sentences. A \emph{belief set} is a set of 
the sentences that are believed by an agent. A common assumption is that
belief sets are closed under logical consequence. The most common operators of belief
set change are \emph{revision} and \emph{contraction} \cite{agm1985}. A number of versions of those operators have
been proposed \cite{hansson-book-chapter} to capture various real life scenarios.

This approach is quite different from the preference relation approach. First, the language
of truth functional sentences is rich and allows for rather complex statements
about preferences: conditional preferences ($a > b \rightarrow c > d$), 
ambiguous preferences ($a > b \vee c > d$) etc. In contrast to that, preferences in the preference
relation framework used in this paper are certain: given a preference relation $\succ$, it is only possible to
check if a tuple is preferred or not to another tuple. Another important difference of these
two frameworks is that the belief revision theory exploits the open-world assumption, while
the preference relation framework uses the closed-world assumption. In addition to that, 
belief revision is generally applicable in the context of finite domains. However, 
the algorithms we have proposed here can be applied to finite and infinite preference relations.

\subsection{Relationships with the preference state framework}

Another preference representation and change framework close to the belief revision theory
is the \emph{preference state} framework \cite{Hansson1995}. As in belief revision,
a preference state is a logically closed sets of sentences describing preferences of an agent.
However, every preference state has an underlying set of preference \emph{relations}. The connection
between states and relations is as follows. A preference relation (which is an order of tuples)
is an unambiguous description of an agent preference. A preference relation induces a set of logical 
sentences which describe the relations. However, it is not always the case
that people's preferences are unambiguous. Hence, every preference \emph{state} is associated with a \emph{set}
of possible preference relations.

Here we show an adaptation of the preference state framework to the preference
relation framework. As a result, we obtain a framework that encompasses preference contraction and 
restricted preference revision.

\begin{definition}\label{def:alt-lang}
 An \emph{alternative} is an element of $\U$. Nonempty subsets of $\U$ are called 
 \emph{sets of alternatives}. The tuple language $\lang_{\U}$ is defined as
\begin{itemize}
 \item if $X, Y \in \U$ then $X > Y \in L_{\U}$
 \item if $X > Y \in L_{\U}$ then $\neg (X > Y) \in L_{\U}$.
\end{itemize}
\end{definition}

A subset of $\lang_{\U}$ is called a \emph{restricted preference set}. The language defined
above is a very restricted version of the language in \cite{Hansson1995} since the only Boolean operator
allowed is negation. Throughout the discussion, we assume that the set of alternatives is fixed to 
a subset $\U_r$ of $\U$.

\begin{definition}
 Let $R$ be a subset of $\U_r \times \U_r$. The set $[R]$ of sentences is defined as follows:
\begin{itemize}
 \item $x > y \in [R]$ iff $xy \in R$
 \item $\neg (x > y) \in [R]$ iff $x, y \in \U_r$ and $x > y \not \in [R]$
\end{itemize}
\end{definition}

\begin{definition}
 A binary relation $R \subset \U_r \times \U_r$ is 
	a \emph{restricted preference model} iff it is a strict partial order. Given a restricted preference model 
	$R$, the corresponding $[R]$ is called a \emph{restricted preference state}.
\end{definition}

In contrast to the definition above, the preference model in \cite{Hansson1995} is defined as a \emph{set} of SPO relations,
and a preference state is an intersection of $[R]$ for all members $R$ of the corresponding preference model.

We define two operators of change of restricted preference states: revision and contraction. 
Restricted states here are changed by sets of statements. 
In \cite{Hansson1995}, a change of a preference state by a set of sentences is defined
as the corresponding change by the conjunction of the corresponding statements. Moreover, 
change by any set of sentences is allowed. 
In the adaptation
of that framework we define here, conjunctions of statements are not a part of the language. Moreover, 
preference revision \cite{Chomicki2007} only allows for adding new preferences, 
and preference relation contraction we have proposed in this paper allows only discarding
existing preferences. Here we aim to define the operator of restricted preference set
revision which captures the semantics of those two operators.

\begin{definition}
 A restricted preference set $S$ is called \emph{positive} iff all sentences it contains are
 in form
 $A > B$ for some $A, B \in \U_r$. Analogously, $S$ is \emph{negative} iff it only contains sentences
 in form $\neg (A > B)$ for some $A, B \in \U_r$. 

 A restricted preference set is a \emph{complement} of $S$ (denoted as $\overline{S}$) if
 for all $A, B \in \U_r$, $A > B \in S$ iff $\neg(A > B) \in S$ and $\neg(A > B) \in S$ iff $A > B \in S$.

 A relation $R_S$ is a \emph{minimal representation} of a restricted preference state $S$
 iff $R_S$ is a minimal relation such that $S \subseteq [R_S]$.
\end{definition}

Positive and negative restricted preference sets are used to change restricted preference states. 
Intuitively, a positive preference set represents the existence of preferences while a negative
set represents a lack of preferences. 

\begin{definition}
 Let $R$ be a restricted preference model. 
Then the operator $*$ on $R$ is a \emph{restricted preference revision on $R$} if and only if for all
positive/negative restricted preference sets $S$, $R \setrev S = \cap \{R'\}$ for all $R'$ such that
\begin{enumerate}
 \item $S \subseteq [R']$
 \item $R'$ is an SPO
 \item there is no SPO $R''$ with $S \subseteq [R'']$ such that
	$R \subseteq R'' \subset R'$ (if $S$ is positive) or 
	$R' \subset R'' \subseteq R$ (if $S$ is negative).
\end{enumerate}
\end{definition}

The last condition in the definition above expresses the minimality of restricted preference state change. 
This condition is different for positive and negative sets: when we add positive statements, we do not
want to discard any existing positive sentences, and when negative statements are added, no new positive
sentences should be added. 
The restricted preference revision operator defined above is different from preference state revision 
in \cite{Hansson1995}. First, preference state revision allows for revision by (finite) sets of arbitrary
sentences, not only positive and negative sentences, as here. Second, the minimality condition here is defined using set containment
while in \cite{Hansson1995} it is defined as a function of symmetric set difference of the original preference 
relations and $R'$. As a result, revising by preference state by a positive/negative sentence may result in 
losing an existing positive/negative sentence. The last difference is based on preference state representation: 
the result of preference revision in \cite{Hansson1995} is a union of relations $R''$ while in our case it is 
an intersection. 

Below we define the operator of contraction for restricted preference states which is similar 
to the contraction of preference states. 

\begin{definition}
 Let $R$ be a restricted preference model. 
Then the operator $\setcon$ on $R$ is \emph{restricted preference contraction on $R$} if and only if for all
positive/negative restricted preference sets $S$, $R \setcon S = R \setrev \overline{S}$.
\end{definition}

Given the operators on restricted preference states we have defined here, their relationships 
with the preference change framework are straightforward.

\begin{proposition}\label{prop:set-rev-ops}
 Let $R$ be a restricted preference model, $S$ be a positive or negative restricted preference set, 
and $R_S$ be a minimal representation of $S$.
 Then $R \setrev S$ is
 \begin{enumerate}
  \item $\emptyset$, if $S$ is a positive restricted preference set and $R \cup R_S$ has cyclic path,
  \item $TC(R \cup R_S)$, if $S$ is a positive restricted preference set and $R \cup R_S$ has no
	cyclic paths,
  \item $\cap \{ R - P^-\ |\ P^-\mbox{ is a minimal \contraction{} of $R$ by $\overline{R_S}$}\}$, 
	if $S$ is a negative restricted preference set,
 \end{enumerate}
 where $TC$ is the transitive closure operator.
\end{proposition}

\begin{proof}

 When a restricted preference model is revised by a positive preference set, the resulting relation 
 $R \setrev S$ is the intersection of all minimal SPO extensions $R'$ of $R$ and $R_S$
 (i.e., $R'$ has to contain an edge from $A$ to $B$ if $A > B \in S$). Such an extension $R'$ does not
 exist if there is an cyclic
 path in $R \cup R_S$. However, if no cyclic paths exist, then there is only one such a minimal extension 
 $R'$ which is equal to the transitive closure of $R \cup R_S$. Hence, $R \setrev S = TC(R \cup R_S)$.
 We note that this result is equivalent to the 
	result of the \emph{union preference revision} \cite{Chomicki2007}.
 
 When a restricted preference model is revised by a negative preference set, the resulting relation $R \setrev S$
 has to be a subset of $R$. Moreover, for all $\neg(A > B) \in S$, there should be no edge from $A$ to $B$
 in $R \setrev S$. Hence, $R \setrev S$ is an intersection of minimally contracted $R$ by $R_{\overline{S}}$,
 which is a result of the full meet contraction of $R$ by $R_{\overline{S}}$. \qed
\end{proof}

Below we list some properties of the revision and contraction operators of restricted preference states. 

\begin{proposition}\label{prop:rev-set-props}
 Let $R$ be a restricted preference model and $S$ be a positive/negative restricted preference set. Then
 \begin{enumerate}
  \item $R \setrev S$ is an SPO (closure)
  \item $S \subseteq [R \setrev S]$ unless $S$ is positive and $R_S \cup R$ has a cyclic path (limited success)
  \item If $S \subseteq [R]$, then $R = R \setrev S$ (vacuity)
 \end{enumerate}
\end{proposition}

\begin{proof}
 All the properties here follow from Proposition \ref{prop:set-rev-ops}. Namely, property 1 follows from
 the fact that the result of $R \setrev S$ is an SPO in every case of Proposition \ref{prop:set-rev-ops}.
 Property 2 follows from Proposition \ref{prop:set-rev-ops} and the definition of $[R \setrev S]$. 
 Property 3 follows from Proposition \ref{prop:set-rev-ops} and 1) $S \subseteq [R]$ implies
 $R_S \subseteq R$ (if $S$ is positive), and 2) a minimally contracted preference relation is equal to itself 
 if contracted by non-existent edges (if $S$ is negative). \qed
\end{proof}

\begin{proposition}\label{prop:con-set-props}
 Let $R$ be a restricted preference model and $S$ be a restricted positive/negative preference set. Then
 \begin{enumerate}
  \item $R \setcon S$ is an SPO (closure)
  \item $S \subseteq [R \setcon S]$ unless $S$ is negative and $R_{\overline{S}} \cup R$ has a cyclic path (limited success)
  \item If $S \cap [R] = \emptyset$, then $R = R \setcon S$ (vacuity)
  \item $R \setrev S = (R \setcon \overline{S}) \setrev S$ unless $S$ is positive and $R_{S} \cup R$ has a cyclic path (limited Levi identity)
  \item $R \setcon S = R \setrev \overline{S}$ (Harper identity, by definition)
 \end{enumerate} 
\end{proposition}

\begin{proof}
 Properties 1, 2, and 3 follow from Proposition \ref{prop:rev-set-props}. Property 4 follows from the fact
that $R \setcon \overline{S} = R \setrev S$ by definition, and Proposition \ref{prop:rev-set-props} implies $R \setrev S = (R\setrev S) \setrev S$ 
when either $S$ is negative or $S$ is positive but $R_{S} \cup R$ has no cyclic path. \qed
\end{proof}

An important difference between the restricted preference-set change operators 
and the corresponding change operators from \cite{Hansson1995}
is that the restricted versions are not always successful (property 2 in Proposition \ref{prop:set-rev-ops}),
and Levi identity holds for a certain class of restricted preference sets.
In addition to that, the operator of preference set contraction in \cite{Hansson1995}
has the property of inclusion ($R \subseteq R \setcon S$) and recovery (if $S \subseteq [R]$, then 
$R = (R \setcon S) \setrev S$). As for the restricted framework defined here, inclusion does not hold 
due to the representation of a preference model as a single SPO relation. Recovery does not hold 
here due to the restrictions to the language (namely, not allowing disjunction of sentences).

\smallskip

We note that one of the main targets of our current work was development of an efficient and practical approach 
of contracting preference relations in the binary relation framework, in the finite and the finitely
representable cases. In addition to the defining semantics
of preference contraction operators, we have also developed a set of algorithms which 
can be used to compute contractions. We have tested them on real-life data and demostrated their efficiency. 
In contrast, \cite{Hansson1995} focuses more on semantical 
aspects of preference change and does not address computational issues of preference
change operators. In particular, finite representability is not addressed. 

\subsection{Other related frameworks}

An approach of preference change is proposed in \cite{DBLP:conf/aaai/ChenP06}. Preferences here
are changed via \emph{interactive example critiques}. This paper identified three types of common 
critique models: similarity based, quality based, and quantity based. However, no formal framework
is provided here. \cite{Freund2004} describes revision of \emph{rational} preference relations over 
propositional formulas. The revision operator proposed here satisfies the postulates of success and minimal change. 
The author shows that the proposed techniques work in case of revision by a single statement and can be 
extended to allow revisions by multiple statements. 

\cite{Dong99} proposes algorithms of incremental maintenance of 
 the transitive closure of graphs using relational algebra. 
The graph modification operations
are edge insertion and deletion. 
Transitive graphs in \cite{Dong99} consist of two kinds of edges:
the edges of the original graph and the edges
induced by its transitive closure. 
When an edge $xy$ of the original graph is contracted,
the algorithm also deletes
all the transitive edges $uv$ such that
all the paths from $u$ to $v$ in the original
graph go through $xy$.
As a result, such contraction is not minimal according to
our definition of minimality. Moreover, \cite{Dong99}
considers only finite graphs, whereas our 
algorithms can work with infinite relations.

\section{Conclusions and future work}

In this paper, we have presented an approach to contracting preference relations. We have considered 
several operators of preference contraction: minimal preference contraction, 
minimal preference-preserving contraction, and (preference protecting) meet contraction
inspired by different scenarios of cautious preference change. We have proposed algorithms and techniques
of computing contracted preference relations for a class of finite and finitely representable relations. 
We have introduced some techniques 
of optimizing preference queries in the presence of contraction.
We have also evaluated the proposed algorithms experimentally and showed that they can be
used in real-life database applications.

We have shown how preference contraction can be evaluated for a special class
of \boundedlayer \contractingrels{}. One of the areas of our future work is to relax that
property and consider more general \contractingrels{}.

An interesting direction of future work is to design an operator of generalized preference relation 
change that allows to change preference relations by discarding existing as well as adding new 
preferences at the same time. The current approaches of preference relation change 
are restricted to only one type of change.

As we showed in the discussion of related work, the existing preference revision
approach \cite{Chomicki2007} fails to work in the presence of conflicts (cycles). A promising direction here
is to use the preference contraction operators presented here to resolve such conflicts. 

In this paper, we assume that the relations defining the preferences to discard are explicitly
formulated by the user. However, such an assumption hardly works in practical scenarios of 
preference change: formulating such a relation requires a full knowledge of his or her preferences,
which may not be the case. Hence, a promising direction is to perform interactive preference
contraction or change.

\section*{Appendix A}

{\bf Theorem \ref{thm:disjuncts}. {\bf (Checking \boundedlayerprop).\ }}{\it
 Let $\form{R}$ be an \ero-formula in DNF, representing an SPO relation $R$, of the following form
 $$F_R(o, o') = F_{R_1}(o, o') \vee \ldots \vee F_{R_l}(o, o'),$$
 where $\form{R_i}$ is a conjunction of atomic formulas. 
 Then checking if there is a constant $k$ such that the length of all $R$-paths is at most $k$ 
 can be done by a single evaluation of $QE$ over a formula of size linear in $|F_R|$. 
}

\medskip
Let $R_i$ be a binary relation represented by the formula $F_{R_i}$ for all $i \in [1, l]$. 
We split the proof of Theorem \ref{thm:disjuncts} into several lemmas. In Lemma \ref{lemma:unbound-disjunct}, we show that
the length of all $R$-paths is bounded by a constant if and only if the length of all $R_i$-paths is bounded by a
constant for every disjunct $F_{R_i}$ of $F_R$. Lemma \ref{lemma:length-conj} shows that the length of all $R_i$-paths is bounded
by a constant if and only if there is a bound on the length of all paths induced by a relation represented by at least one 
conjunct of $F_{R_i}$.
In Lemma \ref{lemma:conj-part-bound}, we show how to check if the length of all paths induced by a conjunct of $F_{R_i}$ is bounded.

\medskip
To prove the first lemma, we use the following idea. Let a sequence $S = (o_{1}, \ldots, o_{n})$ 
of $n \geq 2$ tuples be an $R$-sequence, i.e.,
\begin{equation}\label{eq:disj-0}
(o_{1}, o_{2}), \ldots, (o_{{n-1}}, o_{n}) \in R
\end{equation}

The transitivity of $R$ implies that there is an $R$-edge from $o_{1}$ to 
all other tuples in $S$, i.e.,
\begin{equation}\label{eq:disj-1}
(o_{1}, o_{2}), \ldots, (o_{1}, o_{n}) \in R
\end{equation}
Note that \eqref{eq:disj-1} contains only edges started by $o_1$. 
Since $R = \cup_{i = 1}^l R_i$, for every $R$-edge in \eqref{eq:disj-1}, 
there is $i \in [1, l]$ such that 
it is also an $R_i$-edge. Let $R_{j}$ for some ${j} \in [1, l]$ be such that the number of $R_{j}$-edges in 
\eqref{eq:disj-1} is maximum. Such $R_{j}$ is called \emph{a major component of $S$}. 
Let the sequence $S'$ consist of the end nodes of all these $R_{j}$-edges in the order they appear in $S$.
Such $S'$ is called \emph{a major subsequence of $S$}.

\begin{observation}\label{obs:disj-1}
	Let $S$ be an $R$-sequence, $R_{i^*}$ a major component of $S$,
	and $S'$ be the corresponding major subsequence of $S$. Then
	\begin{enumerate}
	 \item $S'$ is an $R$-sequence
	 \item if the length of $S$ is $n$, then the length of $S'$ is at least $\lceil\frac{n-1}{l}\rceil$
	\end{enumerate}
\end{observation}

The first fact of Observation \ref{obs:disj-1} follows from transitivity of $R$, and the second fact 
follows from the definition of major subsequence. Note that a major subsequence
is an $R$-sequence too. Hence, if it has at least two tuples, we can construct its major subsequence.

\begin{observation}\label{obs:disj-2}
	Let $S_0, \ldots, S_t$ be $R$-sequences such that for all $i \in [1, t]$, $S_i$ is a major subsequence
	of $S_{i-1}$ with the corresponding major components $R_{j_i}$. Let $o, o'$ be the first tuples of $S_1$ and
	$S_t$ correspondingly. Then $R_{j_1}(o, o')$.
\end{observation}

Observation \ref{obs:disj-2} follows from the definition of major subsequence. 

\begin{example}
 Let $S_0 = (x_1, x_2, x_3, x_4, x_5, x_6, x_7, x_8, x_9, x_{10}, x_{11}, x_{12})$ be an $R$-sequences. Figure \ref{pic:major-subsequences}
 illustrates possible construction of a major subsequence $S_1$ of $S_0$, a major subsequence $S_2$ of $S_1$, and
 a major subsequence $S_3$ of $S_2$. The edges on Figure \ref{pic:major-subsequences} correspond to 
 the major-component edges. In every sequence, a node is dark if it is in the major subsequence of 
 the sequence. Note that $S_3$ does not have a major subsequence because a subsequence has to have at least
 two nodes.
\end{example}

\begin{figure}[ht]
\begin{center}
 \begin{tikzpicture}[yscale=0.5, xscale=0.8]
		\tikzstyle{cir} = [draw=black,rounded corners,inner sep=3pt];
		\tikzstyle{seq} = [fill=black!60, text=white];
		\tikzstyle{upperedge} = [bend left=35, ->];
		\tikzstyle{loweredge} = [bend right=60, ->];

		\node          at (-2, 0){{\small $S_0: $}};
		\node[cir] (a) at (0, 0) {{\scriptsize $x_1$}};
		\node[cir] (b) at +(1, 0) {{\scriptsize $x_2$}};
		\node[cir] (c) at +(2, 0) {{\scriptsize $x_3$}};
		\node[cir, seq] (d) at +(3, 0) {{\scriptsize $x_4$}};
		\node[cir] (e) at +(4, 0) {{\scriptsize $x_5$}};
		\node[cir, seq] (f) at +(5, 0) {{\scriptsize $x_6$}};
		\node[cir, seq] (g) at +(6, 0) {{\scriptsize $x_{7}$}};
		\node[cir] (h) at +(7, 0) {{\scriptsize $x_8$}};
		\node[cir, seq] (i) at +(8, 0) {{\scriptsize $x_{9}$}};
		\node[cir, seq] (j) at +(9, 0) {{\scriptsize $x_{10}$}};
		\node[cir] (k) at +(10, 0) {{\scriptsize $x_{11}$}};
		\node[cir, seq] (l) at +(11, 0) {{\scriptsize $x_{12}$}};
		
		\path (a) edge[upperedge] (d);
		\path (a) edge[upperedge] (f);
		\path (a) edge[upperedge] (g);
		\path (a) edge[upperedge] (i);
		\path (a) edge[upperedge] (j);
		\path (a) edge[upperedge] (l);
		
		\node          at (-2, -2){{\small $S_1: $}};
		\node[cir] (d) at (3, -2) {{\scriptsize $x_4$}};
		\node[cir] (f) at +(5, -2) {{\scriptsize $x_6$}};
		\node[cir, seq] (g) at +(6, -2) {{\scriptsize $x_{7}$}};
		\node[cir, seq] (i) at +(8, -2) {{\scriptsize $x_{9}$}};
		\node[cir, seq] (j) at +(9, -2) {{\scriptsize $x_{10}$}};
		\node[cir, seq] (l) at +(11, -2) {{\scriptsize $x_{12}$}};
		
		\path (d) edge[upperedge] (g);
		\path (d) edge[upperedge] (i);
		\path (d) edge[upperedge] (j);
		\path (d) edge[upperedge] (l);

		\node          at (-2, -4){{\small $S_2: $}};
		\node[cir] (g) at +(6, -4) {{\scriptsize $x_{7}$}};
		\node[cir, seq] (i) at +(8, -4) {{\scriptsize $x_{9}$}};
		\node[cir] (j) at +(9, -4) {{\scriptsize $x_{10}$}};
		\node[cir, seq] (l) at +(11, -4) {{\scriptsize $x_{12}$}};
		
		\path (g) edge[upperedge] (i);
		\path (g) edge[upperedge] (l);

		\node          at (-2, -6){{\small $S_3: $}};
		\node[cir] (i) at +(8, -6) {{\scriptsize $x_{9}$}};
		\node[cir] (l) at +(11, -6) {{\scriptsize $x_{12}$}};
		
 		\path (i) edge[upperedge] (l);
 \end{tikzpicture}
\end{center}
	\caption{Major subsequences}
	\label{pic:major-subsequences}
\end{figure}
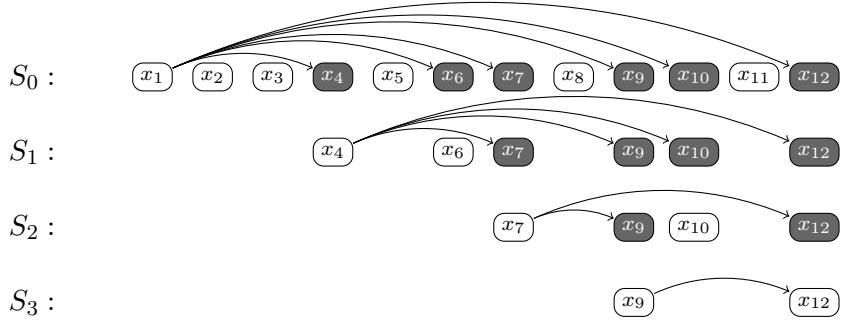

\begin{lemma}\label{lemma:unbound-disjunct}
There is a constant bounding the length of all $R$-paths if and only if for all $i \in [1, l]$, there is a
constant bounding the length of all $R_i$-paths.
\end{lemma}

\begin{proof}$ $
In the case when $l = 1$, the lemma trivially holds. Further we assume $l > 1$.

\rightsideproof If for some $i \in [1, l]$, the length of $R_i$-paths cannot be bounded, neither can the 
length of $R$-paths.

\leftsideproof Assume that for all $i \in [1, l]$, all $R_i$-paths are of length at most $k$. 
Show that
the length of all $R$-paths is not more than $\sum_{i=1}^{(k+2)l+1}l^i-2$. For the sake of contradiction, 
let there be an $R$-path 
of length $\sum_{i=1}^{(k+2)l+1}l^i-1$. Let $S_0$ be the corresponding $R$-sequence.
The size of $S_0$ is $\sum_{i=0}^{(k+2)l+1}l^i$. Let $S_1$ be a major subsequence of $S_0$. 
By Observation \ref{obs:disj-1}, $S_1$ is also an $R$-sequence, 
and its length is at least $\sum_{i=0}^{(k+2)l}l^i$. Following that logic, let $S_{t}$ be a major subsequence
of $S_{t-1}$ with the corresponding major component $R_{j_{t-1}}$. The size of $S_{t}$ is 
at least $\sum_{i=0}^{(k+2)l-t+1}l^i$. Such computation may continue while the size of $S_t$ is greater than one, 
i.e., while $t \leq (k+2)l$. Let the major components of $S_1, \ldots, S_{(k+2)l}$ be 
$R_{j_1}, \ldots, R_{j_{(k+2)l}}$ correspondingly. Note that there are at most $l$ possible different major components.
Thus, at least $k+2$ major components in $R_{j_1}, \ldots, R_{j_{(k+2)l}}$ are the same. Let us denote 
the first $k+2$ of them as $R_{t_{1}}, \ldots, R_{t_{k+2}}$ and the tuples which start the corresponding
major sequences as $o_{t_1}, \ldots, o_{t_{k+2}}$. By Observation \ref{obs:disj-2}, 
$$R_{t_1}(o_{t_1}, o_{t_2}) \wedge R_{t_2}(o_{t_2}, o_{t_3}) \wedge \ldots \wedge R_{t_{k+1}}(o_{t_{k+1}}, o_{t_{k+2}})$$
Since all $R_{t_{1}}, \ldots, R_{t_{k+2}}$ are the same, the expression above implies 
that there is an $R_i$-path of length $k+1$ for some $i \in [1, l]$ which is a contradiction. \qed
\end{proof}

In Lemma \ref{lemma:unbound-disjunct}, we showed that the problem of checking the bounded-length property of all $R$-paths
can be reduced to the problem of testing the same property for $R_i$-paths.
Note that $R_i$ is represented
by a formula $F_{R_i}$ which is a conjunction of atomic formulas. 
Let the set of all attributes which are present in the formula
$F_{R_i}$ be defined as $\A_{F_{R_i}}$. Then $F_{R_i}$ can be represented as
$$F_{R_i}(o, o') = \bigwedge_{A \in \A_{F_{R_i}}} \lambda_{A}(o, o'),$$
where $\lambda_A(o, o')$ is a conjunction of all atomic formulas in which the attribute $A$ is used.
Note that the structure of the preference formula language implies that \emph{every atomic formula 
belongs to exactly one $\lambda_{A}$}. 

Denote the relation represented by $\lambda_{A}$ as $\Lambda_A$. 
In the next lemma, we show that the problem of checking the \boundedlayerprop of all $R_i$-paths can be reduced
to the same problem for $\Lambda_{A}$-paths.

\begin{lemma}\label{lemma:length-conj}
  There is a constant bounding the length of all $R_i$-paths if and only if 
   for some $A \in \A_{F_{R_i}}$, there is a constant bounding the length of all $\Lambda_A$-paths.
\end{lemma}

\begin{proof}$ $
 
 \leftsideproof Let for every $k$, there be an $R_i$-path of length at least $k$
	$$R_i(o_1, o_2) \wedge R_i(o_2, o_3) \wedge \ldots \wedge R_i(o_k, o_{k+1})$$
 Then for all $A \in \A_{F_{R_i}}$, we have a $\Lambda_A$-path of length at least $k$
	$$\Lambda_A(o_1, o_2) \wedge \Lambda_A(o_2, o_3) \wedge \ldots \wedge \Lambda_A(o_k, o_{k+1})$$ 

 \rightsideproof Let for every $k$ and $A \in \A_{F_{R_i}}$, there be an $\Lambda_A$-path of length at least $k$
	$$\Lambda_{A}(o_1^A, o_2^A) \wedge \Lambda_A(o_2^A, o_3^A) \wedge \ldots \wedge \Lambda_A(o_k^A, o_{k+1}^A)$$
  Construct a sequence of tuples $(o_1, o_2, o_3, \ldots)$ as follows. Let $o_j.A = o_j^A.A$ if
	$A \in \A_{F_{R_i}}$. Otherwise, let $o_j.A$ be any value from the domain $\Dom_{A}$ of $A$. Clearly, the following $R_i$-path
 	is of length at least $k$
	$$R_i(o_1, o_2) \wedge R_i(o_2, o_3) \wedge \ldots \wedge R_i(o_k, o_{k+1})$$
	\qed
\end{proof}

\begin{lemma}\label{lemma:conj-part-bound}
 There is a constant bounding the length of all $\Lambda_A$-paths if and only if 
 there is no $\Lambda_A$-path of length three, i.e., 
	$$\neg \exists o_1, o_2, o_3, o_4 \in \U \ .\ \Lambda_A(o_1, o_2) \wedge \Lambda_A(o_2, o_3) \wedge \Lambda_A(o_3, o_4)$$
\end{lemma}

\begin{proof}$ $

	\leftsideproof If for every constant $k$, there is a $\Lambda_A$-path of length at least $k$, 
	then there is a $\Lambda_A$-path of length three.

	\rightsideproof If $\Lambda_A$ is unsatisfiable, then there are no $\Lambda_A$-paths. Thus, we assume that
	$\Lambda_A$ is satisfiable. Based on the preference formula language, the formula $\lambda_A(o, o')$ can be split
	into at most three conjunctive formulas: 
	\begin{enumerate}
	 \item $\phi_L$: a conjunction of all atomic formulas $o.A \theta c$, 
	 \item $\phi_R$: a conjunction of all atomic formulas $o'.A \theta c$, 
	 \item $\phi_M$: a conjunction of all atomic formulas $o.A \theta o'.A$	 
	\end{enumerate}
	for $\theta \in \{=, \neq, <, >\}$ and a $\D$- or $\Q$-constant $c$. Any of these three formulas may be missing
	because $\lambda_A$ may not containt atomic formulas of the specified type. 
	$\phi_L$ and $\phi_R$ capture the range of the left and the right argument in $\lambda_A$, 
	correspondingly, and $\phi_M$ constrains their relationship. 

	Here we assume that $A$ is a
	$\Q$-attribute, and the case of $\D$-attributes is similar. Note that if $\phi_L$ is defined, 
	then the range $r_{L}$ of $\phi_L$ is 1) an open rational number interval with a finite number 
	of holes (due to possible atomic formulas $o.A \neq c$), or 2) a single rational value
	(due to the formula $o.A = c$).
	 If $\phi_L$ is undefined, then 
	$r_{L}$ is the entire set of rational numbers. Thus, the the number of distinct elements $|r_L|$ in $r_L$ is either $\infty$ or 1. The same holds for the number of distinct elements $|r_R|$ in $r_R$. 
	Hence for our class of formulas, $|r_L \cap r_R| \in \{1, \infty \}$.
	
	Now consider the structure of $\phi_M$. If $\phi_M$ is undefined, then 
	$|r_L \cap r_R| > 0$ implies that there are $\Lambda_A$-paths of length at least $k$ for every $k$, consisting of tuples
	whose $A$-values are arbitrary elements of $r_L \cap r_R$. 
	If $\mbox{``$o.A = o'.A$``} \in \phi_M$, then 
	no other atomic formula is in $\phi_M$ (otherwise, $\Lambda_A$ is unsatisfiable). Since 
	$|r_L \cap r_R| > 0$, $\Lambda_A$-paths of length at least $k$ for every $k$  can be constructed of tuples with the 
	value of $A$ all equal to any member of $r_L \cap r_R$. If $\mbox{``$o.A > o'.A$``} \in \phi_M$, then
	$\mbox{``$o.A = o'.A$``}, \mbox{``$o.A < o'.A$``} \not \in \phi_M$ (otherwise $\lambda_A$ is unsatisfiable).
	However, $\mbox{``$o.A \neq o'.A$``}$ may be in $\phi_M$ and is implied by $\mbox{``$o.A > o'.A$``} \in \phi_M$ 
	so can be dropped.
	The existence of a $\Lambda_A$-path of length three implies that $|r_L \cap r_R| > 1$ and thus
	$|r_L \cap r_R|  = \infty$. Hence there are $\Lambda_A$-paths of length at least $k$ for every $k$. 
	The case of $\mbox{``$o.A < o'.A$``} \in \phi_M$
	is analogous. The last case is when only $\mbox{``$o.A \neq o'.A$``}$. The existence of a $\Lambda_A$-path
	of the length three implies that there are two different values $c_1, c_2 \in r_L \cap r_R$. 
	Hence, $\Lambda_A$-paths of length at least $k$ for every $k$ can be constructed by taking any sequence 
	of tuples in which the value of $A$ of every even tuple is $c_1$ and of every odd tuple is $c_2$.
	\qed
\end{proof}

\begin{proofofthm-disjuncts}$ $
 Here we show how to construct a formula which is true iff there is a constant $k$
such that the length of all ${R}$-paths is bounded by $k$. By Lemma \ref{lemma:unbound-disjunct},
such a formula can be written as a conjunction of $l$ formulas each of which represents the fact
that the length of all $R_i$-paths is bounded. By Lemma \ref{lemma:length-conj},
such a formula can be written as a disjunction of formulas each of which represents the fact that
the length of all $\Lambda_A$-paths is bounded. By Lemma \ref{lemma:conj-part-bound}, such formulas
are of size linear in the size of $\Lambda_A$. Hence, the resulting formula is linear in the size
of $F_{R}$. Due to the construction in Lemma \ref{lemma:conj-part-bound}, the formula has quantifiers. They can be eliminated
using $QE$. \qed
\end{proofofthm-disjuncts}

\bibliographystyle{alpha}
\bibliography{preference}

\end{document}